\documentclass{article}
\usepackage{iclr2027_conference,times}
\usepackage{iftex}
\ifPDFTeX\else
  \usepackage{fontspec}
\fi

\usepackage{graphicx}
\usepackage{adjustbox}  
\usepackage{amsmath}
\usepackage{amssymb}
\usepackage{amsthm}
\usepackage{booktabs}
\usepackage{multirow}
\usepackage{fancyhdr}
\usepackage{float}      
\usepackage{colortbl}   
\usepackage{xspace}
\usepackage{enumitem}
\usepackage{url}
\usepackage{titletoc}

\definecolor{iclrblue}{rgb}{0.21,0.49,0.74}

\usepackage[breaklinks,colorlinks,allcolors=iclrblue]{hyperref}
\usepackage[capitalise]{cleveref}   
\crefname{assumption}{Assumption}{Assumptions}
\Crefname{assumption}{Assumption}{Assumptions}

\makeatletter

\let\iclr@includegraphics\includegraphics
\renewcommand{\includegraphics}[2][]{%
  \IfFileExists{#2}%
    {\iclr@includegraphics[#1]{#2}}%
    {\fbox{\begin{minipage}[c][0.22\linewidth][c]{0.95\linewidth}%
       \centering\ttfamily\footnotesize missing artwork\\\detokenize{#2}%
     \end{minipage}}}}
\newcommand{\rawincludegraphics}[2][]{\iclr@includegraphics[#1]{#2}}
\makeatother

\definecolor{bestgreen}{RGB}{201,234,199}
\definecolor{gaingreen}{RGB}{0,110,60}
\definecolor{oursrow}{gray}{0.94}
\newcommand{\best}[1]{\cellcolor{bestgreen}#1}
\newcommand{\ours}{\rowcolor{oursrow}}
\newcommand{\ci}[1]{\ensuremath{{}_{\pm #1}}}
\newcommand{\gain}[1]{{\color{gaingreen}#1}}
\newcommand{\loss}[1]{{\color{red}#1}}
\newcommand{\blk}[2]{\multicolumn{#1}{@{}l}{\textbf{#2}}}
\newcommand{\sub}[2]{\multicolumn{#1}{@{}l}{\emph{#2}}}
\newcommand{\na}{---}

\theoremstyle{plain}
\newtheorem{proposition}{Proposition}
\newtheorem{lemma}[proposition]{Lemma}
\newtheorem{corollary}[proposition]{Corollary}
\theoremstyle{definition}
\newtheorem{definition}[proposition]{Definition}
\newtheorem{assumption}[proposition]{Assumption}
\theoremstyle{remark}
\newtheorem{remark}[proposition]{Remark}

\newcommand{\method}{MARC}
\newcommand{\purity}{\operatorname{Purity}}
\newcommand{\antihub}{\operatorname{AntiHub}}
\newcommand{\Kstar}{K^{\star}}

\title{Retrieval Geometry Shapes\\Cache-Based CLIP Adaptation}

\author{%
  {\normalfont \textbf{Mahir Shahriar Tamim}\textsuperscript{1,*}\quad
  \textbf{Md.\ Samiul Alim}\textsuperscript{1,*}\quad
  \textbf{Azmine Toushik Wasi}\textsuperscript{2}}\\
  \textbf{Shahriyar Zaman Ridoy}\textsuperscript{1}\quad
  \textbf{Meharun Nesa}\textsuperscript{1}\quad
  \textbf{Mohammad Abu Yousuf}\textsuperscript{3}\\
  \textbf{Alex Lamb}\textsuperscript{4}\quad
  \textbf{Mohammad Ali Moni}\textsuperscript{5}\\[5pt]
  {\normalfont\small\textsuperscript{1}Department of Electrical \& Computer Engineering, North South University}\\
  {\normalfont\small\textsuperscript{2}Computational Intelligence and Operations Laboratory}\\
  {\normalfont\small\textsuperscript{3}Institute of Information Technology, Jahangirnagar University}\\
  {\normalfont\small\textsuperscript{4}College of AI, Tsinghua University}\\
  {\normalfont\small\textsuperscript{5}Charles Sturt University}\\[3pt]
  {\normalfont\small\textsuperscript{*}\emph{Equal contribution.}}\\[4pt]
  {\normalfont\small\texttt{\{mahir.tamim,samiul.alim01,shahriyar.ridoy,meharun.nesa\}@northsouth.edu}}\\
  {\normalfont\small\texttt{azminetoushik.wasi@gmail.com}\quad
  \texttt{yousuf@juniv.edu}}\\
  {\normalfont\small\texttt{lambalex@tsinghua.edu.cn}}\\
  {\normalfont\small\texttt{mmoni@csu.edu.au}}%
}

\iclrfinalcopy

\begin{document}
\maketitle
\lhead{Under review as a conference paper at ICLR 2027}

\begin{abstract}
Cache-based test-time adaptation improves CLIP predictions by storing and retrieving examples from the target stream while keeping the model frozen. However, existing methods largely treat the feature space used for image--image retrieval as fixed, leaving open how much adaptation depends on the retrieval space itself. We study this question by fixing the memory and changing only the retrieval encoder, finding that the same memory can yield very different gains: across sixteen retrieval spaces, ImageNet-A cache gain ranges from at most $+0.44$ points for CLIP and MAE to $+19.7\pm0.4$ for DINOv2-L, while label-free retrieval-space selection retains $98\%$ of oracle gain on ImageNet-V2. These results show that memory quality depends not only on which examples are stored, but also on how they are retrieved. Motivated by this finding, we propose \method{} (Memory-Augmented Retrieval for CLIP), a training-free system that uses frozen CLIP for prediction and DINOv2-B for retrieval with a single fusion weight. A single-view cache repairs $1{,}074\pm21$ baseline errors, compared with $878\pm4$ for a 64-view ensemble, at roughly one seventh of the cost. Across four ImageNet distribution shifts, \method{} reaches a 67.91\% OOD average and, at matched DINOv2-B scale and eight views, achieves $64.17\pm0.31\%$ versus $62.75\pm0.15\%$ for a graph-based cache system while running 2.6 times faster. 
Overall, our results establish retrieval space as a first-order design choice for robust cache-based adaptation in remote sensing, scientific imaging, and changing visual environments.
\end{abstract}


\setlength{\textfloatsep}{6pt}
\setlength{\floatsep}{6pt}
\setlength{\abovecaptionskip}{-2pt}
\setlength{\belowcaptionskip}{-2pt}

\section{Introduction}
\label{sec:intro}
\textbf{R}eliable recognition under distribution shift is essential for deploying open-vocabulary vision models. \textbf{C}ontrastive \textbf{L}anguage-\textbf{I}mage \textbf{P}re-training (CLIP)~\citep{radford2021clip} can recognize a wide range of categories without task-specific training, but its performance often drops under distribution shifts~\citep{hendrycks2021imageneta,hendrycks2021imagenetr,wang2019imagenetsk}. Cache-based test-time adaptation (TTA) provides a training-free remedy: it stores features and pseudo-labels from the test stream, retrieves similar examples for each new image, and uses them to refine the zero-shot prediction~\citep{karmanov2024tda,zhang2024boostadapter,han2024dota,zhang2025scap,huang2025cosmic}. Recent work has mainly improved how the cache is built and used, with techniques such as regional bootstrapping, distributional estimation, and graph-based retrieval.

However, cache retrieval depends on another design choice that has received much less attention: \textbf{the feature space used to measure image-image similarity}. Most cache-based methods retrieve neighbors using CLIP's image representation~\citep{karmanov2024tda,zhang2024boostadapter}. This implicitly asks one representation to serve two different purposes. Zero-shot classification compares an image with $C$ text prototypes, whereas cache retrieval compares it with other images in the memory. Under distribution shift, a representation that works well for image-text matching may not produce the most useful image-image neighborhoods (\cref{fig:mismatch}). This motivates a simple question: \emph{how much of cache-based adaptation comes from the memory mechanism, and how much comes from the retrieval similarities induced by the encoder?}

We study this question through a controlled retrieval-encoder intervention. Self-supervised encoders already provide strong $k$-NN representations~\citep{caron2021dino,oquab2023dinov2}; our contribution is to isolate their effect on cache retrieval. Frozen CLIP determines the predictions, pseudo-labels, admissions, and retained indices, while the stream order is shared across spaces and only the retrieval encoder changes the query--key similarities. This keeps the cache trajectory fixed and makes retrieval space the only changing variable. TDA~\citep{karmanov2024tda} performs dynamic caching in CLIP space, while COSMIC~\citep{huang2025cosmic} combines DINOv2 with dual semantic graphs and hyper-class querying. Our intervention instead directly measures the effect of changing retrieval space alone, with implications for adaptation across diverse application domains.

\begin{figure}[t]
\centering
\includegraphics[width=\linewidth]{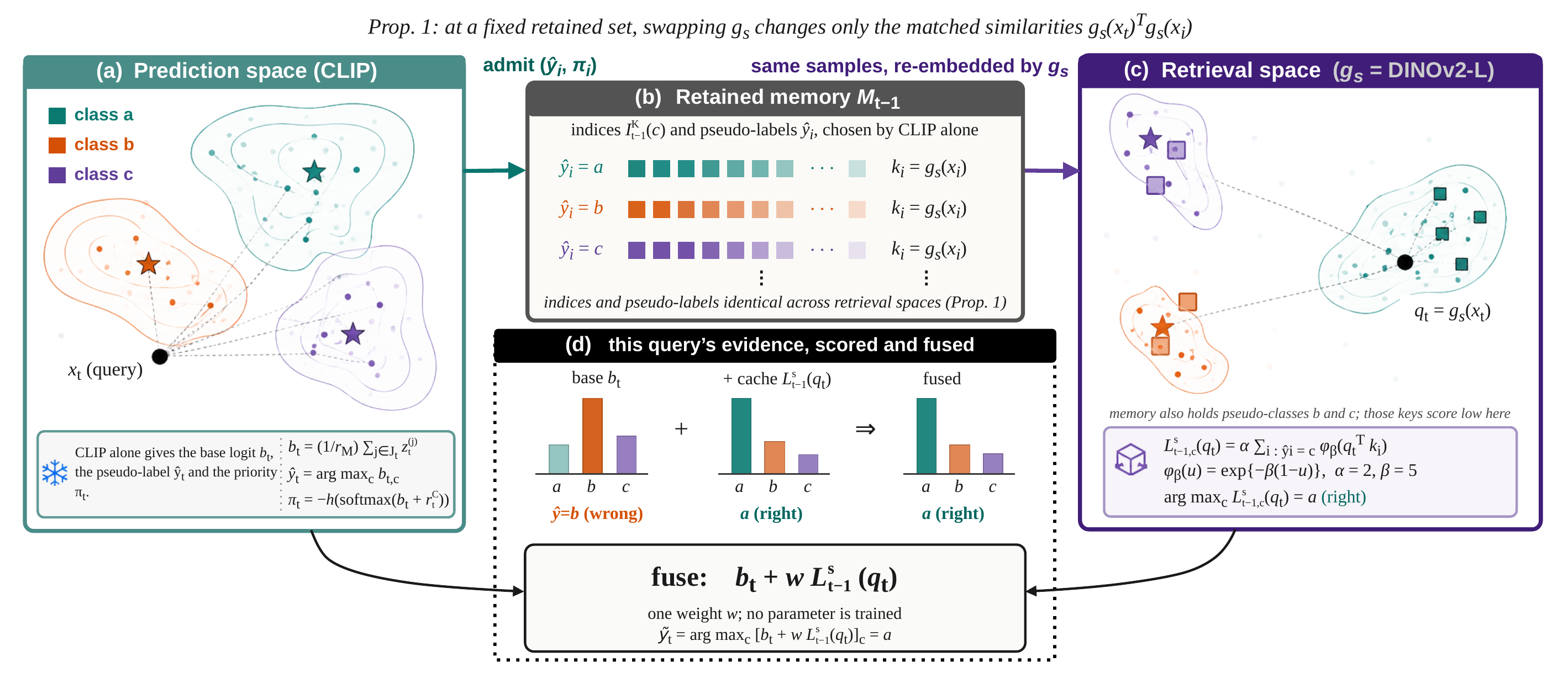}
\caption{\textbf{Controlled retrieval-encoder intervention.}
\textbf{(a)} Frozen CLIP computes the base logits $b_t$, pseudo-label $\hat{y}_t$, and admission priority $\pi_t$.
\textbf{(b)} All retrieval spaces share the same retained samples and stored pseudo-labels, which are re-embedded as keys $k_i=g_s(x_i)$.
\textbf{(c)} The retrieval encoder embeds the query as $q_t=g_s(x_t)$, so changing $g_s$ modifies only the query--key similarities while keeping the cache trajectory fixed (\cref{prop:space-intervention}).
\textbf{(d)} DINOv2-L retrieves evidence for the correct class $a$, correcting CLIP's initial class-$b$ prediction after fusion with $b_t+wL^{s}_{t-1}(q_t)$.
The main \method{} benchmarks use DINOv2-B (on 6 datasets: Caltech101, DTD, EuroSAT, Flowers102, Pets, and UCF101), while DINOv2-L is used for the controlled intervention and scale analyses.}
\label{fig:mismatch}
\end{figure}

This intervention reveals that retrieval geometry affects not only neighbor quality, but also the gain and useful capacity of the cache. Existing methods use 1 to $64$ augmented views~\citep{karmanov2024tda,huang2025cosmic,shu2022tpt,feng2023difftpt,zhang2024boostadapter,zero2024} and often tune the per-class capacity $K$ by grid search. Across sixteen retrieval spaces and five streams, four different encoders achieve the highest cache gain. On ImageNet-A, changing only the retrieval encoder moves cache gain from at most $+0.44$ points (CLIP; MAE $+0.03$) to $+19.74\pm0.40$ for DINOv2-L, and from $+0.00$ to $+19.16$ once the fusion weight is selected on held-out ImageNet-V2. Neighborhood purity predicts the best space when candidate values are well separated, while pseudo-purity enables label-free selection from unlabeled target images.

Motivated by these findings, we introduce \method{} (\textbf{M}emory-\textbf{A}ugmented \textbf{R}etrieval for \textbf{C}LIP), a training-free system with frozen CLIP and DINOv2-B encoders. \method{} uses CLIP to build and admit samples into the memory, retrieves them with DINOv2-B, and fuses the retrieved evidence with the zero-shot prediction. This simple design decouples prediction from retrieval, allowing the retrieval space to be changed without altering the predictor or the cache trajectory, and provides a clean way to study their individual effects.
Our core contributions are:
\begin{enumerate}[
    leftmargin=1.2em,
    itemsep=0.3pt,
    topsep=0.3pt,
    parsep=0pt,
    partopsep=0pt
]

\item \textbf{We isolate retrieval similarity as a design variable.} Holding the predictor, prompts, admissions, retained indices, and test stream fixed, we vary only the retrieval encoder and its query--key similarities, directly isolating their effect on cache-based adaptation.

\item \textbf{We establish retrieval geometry as a key determinant of cache effectiveness.} Across retrieval spaces and streams, cache gain varies from nearly zero to about $20$ points, with different encoders emerging as strongest across shifts. Neighborhood purity predicts gain, anti-hubness captures useful capacity, and stronger retrieval reduces the marginal benefit of additional views.

\item \textbf{We enable label-free retrieval-space selection.} Pseudo-purity closely tracks labeled rankings and selects effective retrieval spaces without target labels, providing a practical alternative to oracle selection under online constraints.
\end{enumerate}

Our findings have substantial practical implications. Pseudo-purity selection retains $90\%$ and $98\%$ of oracle gain on ImageNet-A and ImageNet-V2, respectively. Across four ImageNet shifts, \method{} reaches a $67.91\%$ OOD average and, at matched DINOv2-B scale and eight views, achieves $64.17\pm0.31\%$ versus $62.75\pm0.15\%$ for a graph-based cache while running $2.6\times$ faster. These properties are particularly relevant to deployment settings with evolving visual distributions, from remote sensing and scientific imaging to continuously changing visual environments. Together, our results establish retrieval space as a first-order design choice and position \method{} as a simple reference system for studying its interaction with prediction, memory, and fusion.

\vspace{-3mm}
\section{Related Work}
\label{sec:related}
\vspace{-2mm}
Cache-based TTA keeps the predictor frozen and adapts through external memory: Tip-Adapter~\citep{zhang2022tipadapter} introduced key--value caching for CLIP, while TDA~\citep{karmanov2024tda} extended it online using entropy-prioritized pseudo-labels; subsequent methods modify cache construction, statistics, or fusion~\citep{zhang2024boostadapter,han2024dota,zhang2025scap,hu2024bafta,zhang2026tata}. COSMIC~\citep{huang2025cosmic} combines CLIP and DINOv2 caches with semantic graphs and hyper-class queries, whereas we isolate retrieval geometry by varying only the retrieval encoder while fixing cache construction and admission. Self-supervised encoders such as DINO and DINOv2~\citep{caron2021dino,oquab2023dinov2} provide strong image--image neighborhoods, unlike reconstruction-oriented MAE~\citep{he2022mae}, making them useful retrieval spaces for studying cache gain and capacity. Prior work has examined neighborhood structure through $k$-NN accuracy, purity, hubness, and affinity graphs~\citep{radovanovic2010hubness,hu2024bafta,boudiaf2022lame}, but has not isolated the feature space defining those neighborhoods as a design variable; we test whether purity and anti-hubness predict gain and useful capacity. Finally, while augmentation-based methods such as MEMO, TPT, DiffTPT, and ZERO~\citep{zhang2022memo,shu2022tpt,feng2023difftpt,zero2024} use multiple views for test-time adaptation, we study how their value changes with retrieval effectiveness. Additional related work is discussed in detail in Appendix~\ref{sec:apx-related}.

\vspace{-3mm}
\section{Method}
\label{sec:method}
\vspace{-2mm}

\begingroup
\setlength{\abovedisplayskip}{4pt plus 1pt minus 1pt}
\setlength{\belowdisplayskip}{4pt plus 1pt minus 1pt}
\setlength{\abovedisplayshortskip}{2pt plus 1pt}
\setlength{\belowdisplayshortskip}{2pt plus 1pt minus 1pt}
\setlength{\jot}{2pt}

\subsection{Preliminaries: Problem Formulation}
\label{sec:setup}
\vspace{-1.5mm}
We consider online test-time adaptation with a frozen CLIP classifier. Unlabeled images $\{x_t\}_{t=1}^{T}$ arrive sequentially, with each $x_t$ predicted using the frozen model and memory from prior samples before possible insertion into the cache. Ground-truth labels are used only for evaluation; the predictor, prompts, and class prototypes remain fixed, making adaptation causal and entirely memory-based.
\\
Let $f$ denote the frozen CLIP image encoder and $\mathbf{t}_c$ the unit-normalized text prototype for class $c$. We use $\bar f(x)$ to denote the unit-normalized image feature and $\tau$ the fixed CLIP logit scale. Each test image $x_t$ is represented by $M$ views. Let $J_t$ index the $r_M=\max(1,\lfloor0.1M\rfloor)$ views with the lowest prediction entropy. For view $j$, the raw logit for class $c$ and the averaged base logit are
\begin{equation} \small
z^{(j)}_{t,c}=\tau\,\big\langle \bar f(x_t^{(j)}),\mathbf{t}_c\big\rangle,
\qquad
b_t=\frac{1}{r_M}\sum_{j\in J_t}z_t^{(j)},
\label{eq:view-logit}
\end{equation}
The base pseudo-label is $\hat y_t=\arg\max_c b_{t,c}$. It is produced by the frozen CLIP predictor and provides the class assignment used by both online memories in \cref{sec:method}.

\vspace{-2mm}
\subsection{Memory-Augmented Retrieval}
\label{sec:cache-pipeline}
\vspace{-2mm}
Our key design choice is the retrieval space. Let $g_s$ be a frozen, normalized retrieval encoder defining space $s$. For each test image $x_t$, its unaugmented feature $q_t=g_s(x_t)$ is used as the retrieval key if the image is admitted to memory.

\noindent\textbf{Cache Construction.}
We also maintain a CLIP-indexed memory $\mathcal C_t$ with a fixed per-class capacity of $K_C{=}16$. Before updating the memory, we query it with the CLIP feature $k_t^C=\bar f(x_t)$ to obtain $r_t^C=L^{\mathcal C_{t-1}}(k_t^C)$ using the kernel in \cref{eq:kernel}. For $p_t^{(j)}=\operatorname{softmax}(z_t^{(j)})$, we define
\begin{equation} \small
e_t=h\!\left(\frac{1}{M}\sum_{j=1}^{M}\operatorname{softmax}(p_t^{(j)})\right),\,
\pi_t=-h\!\left(\operatorname{softmax}(b_t+r_t^C)\right).
\label{eq:priority-memory}
\end{equation}
where $h(p)=-\sum_c p_c\log p_c$ is Shannon entropy with natural logarithms. Following the evaluated TDA implementation, $e_t$ applies the additional softmax shown above; normalized entropy quantities are marked explicitly. Replacing it with entropy of the view posterior gives $60.10\pm0.41\%$ versus $60.14\pm0.38\%$ across three paired ImageNet-A seeds; further admission controls are reported in App.~\ref{sec:s_components}.
The CLIP memory is updated with $(k_t^C,\hat y_t,-e_t)$, while $\pi_t$ determines which samples enter the retrieval memory. The CLIP memory has no output weight, so it affects neither the final logit nor the stored pseudo-label. The strict one-memory setting uses $K_C{=}0$.

\noindent\textbf{Retrieval and Prediction.}
For retrieval space $s$ and per-class capacity $K$, the pre-query memory is
$\mathcal M^{s,K}_{t-1}=\{(\mathbf k_i,\hat y_i,\pi_i)\}$, with at most $K$ entries per pseudo-class. It produces class evidence
\begin{equation} \small
L^{\mathcal M}_c(q_t)
=\alpha\!\sum_{i\in\mathcal M_{t-1}:\,\hat y_i=c}
\phi_\beta(q_t^\top\mathbf k_i),
\qquad
\phi_\beta(u)=\exp\{-\beta(1-u)\},
\label{eq:kernel}
\end{equation}
where $\alpha{=}2$ and $\beta{=}5$, following TDA~\citep{karmanov2024tda}. Both memories are read before they are updated. The prediction and memory updates are
\begin{equation} \small
\small
\widetilde y_t=\arg\max_c\big[b_t+wL^{\mathcal M_{t-1}}(q_t)\big]_c,\,
\mathcal C_t=\operatorname{Upd}_{K_C}(\mathcal C_{t-1};k_t^C,\hat y_t,-e_t),\,
\mathcal M_t^{s,K}=\operatorname{Upd}_K(\mathcal M_{t-1}^{s,K};q_t,\hat y_t,\pi_t).
\label{eq:online-loop}
\end{equation}
Each pseudo-class keeps its top-$K$ arrivals by priority. Since stored labels and priorities do not depend on the fused prediction, both memory trajectories are independent of $w$ and can be re-scored offline (\cref{prop:traj}).
We vary the number of views $M$, memory capacity $K$, and retrieval space $s$, while keeping the predictor, prompts, kernel, admission rule, and stream fixed. The DINOv2 retrieval configuration in \cref{eq:online-loop} is our \method{} (Memory-Augmented Retrieval for CLIP). It removes output fusion from the CLIP cache, negative memory, adaptive fusion, and pseudo-label refinement; all variants are detailed in App.~\ref{sec:s_components}.

\begin{figure}[t]
\centering
\begingroup
\setlength{\fboxsep}{0pt}
\setlength{\unitlength}{\linewidth}
\begin{picture}(1,0.22744)
  \put(0,0){\rawincludegraphics[width=\linewidth]{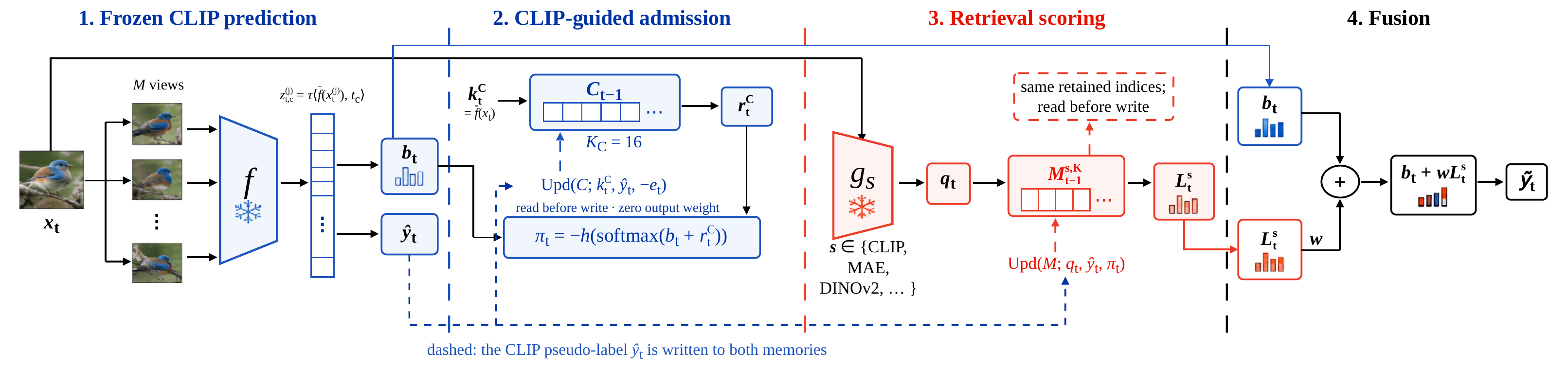}}
  \put(.748,.087){\colorbox{white}{\makebox[.040\linewidth][c]{\rule[-1.5pt]{0pt}{8pt}\tiny $L_t^s$}}}
\end{picture}
\endgroup
\caption{\textbf{The four-stage \method{} workflow.} (1) Frozen CLIP aggregates $M$ views into the base
logit $b_t$ and pseudo-label $\hat y_t$. (2) The blue, zero-output-weight CLIP memory is read
before update to compute $\pi_t$ and guide admission. (3) A frozen retrieval encoder $g_s$
embeds the query and the same admitted samples; the red retrieval memory is read before update
to produce $L_t^s$. Dashed arrows show the CLIP pseudo-label written to both memories.
(4) Fusion returns $b_t+wL_t^s$. All encoder parameters remain frozen throughout.}
\label{fig:pipeline}
\end{figure}

\vspace{-3mm}
\subsection{What the Decoupling Identifies}
\label{sec:space-intervention}
\vspace{-2mm}
To state the identification target, first consider a generic cache whose retained population
$J_{t-1}^s=(J_{t-1}^s(c))_{c=1}^C$ may itself depend on encoder $s$. For fixed predictor,
kernel, and fusion machinery, write $F_t(s,J)=b_t+wL_t^{s,J}$. Adding and subtracting the
matched-population predictor gives the exact decomposition
\begin{equation} \small
F_t(s,J_{t-1}^s)-F_t(s',J_{t-1}^{s'})=\underbrace{F_t(s,J_{t-1}^s)-F_t(s',J_{t-1}^s)}_{\text{encoder similarities on matched keys}}+\underbrace{F_t(s',J_{t-1}^s)-F_t(s',J_{t-1}^{s'})}_{\text{retained-population change}}.
\label{eq:geometry-population-decomp}
\end{equation}
A conventional composite-system comparison can contain both terms and may additionally change
the aggregation machinery. MARC removes the second term by construction: $\hat y_t$ and
$\pi_t$ are CLIP functions, so admission is independent of $s$. Let $I_{t-1}^K(c)$ denote
this common retained index set and write $u_{ti}^s=g_s(x_t)^\top g_s(x_i)$.

\begin{proposition}[Exact space isolation and prediction stability]
\label{prop:space-intervention}
Fix the stream, its augmentation draws, $K_C$, $K$, $w\ge0$, and all tie rules. For any two retrieval
spaces $s$ and $s'$:
\textbf{(i)} the retained index sets $I_{t-1}^K(c)$ are identical for every $t,c$;
\textbf{(ii)} their cache logits satisfy
\begin{equation} \small
L_{t,c}^{s}=\alpha e^{-\beta}\sum_{i\in I_{t-1}^K(c)}e^{\beta u_{ti}^{s}},\quad
\left|L_{t,c}^{s}-L_{t,c}^{s'}\right|\le D_{t,c}^{s,s'}:=\alpha\beta\sum_{i\in I_{t-1}^K(c)}\left|u_{ti}^{s}-u_{ti}^{s'}\right|.
\label{eq:space-perturbation}
\end{equation} \small
Let $F_t^s=b_t+wL_t^s$, $a_t^s=\arg\max_c F_{t,c}^s$, and
$\mathrm{gap}_t^s=F_{t,a_t^s}^s-\max_{c\ne a_t^s}F_{t,c}^s$. Then
\begin{equation}
\mathrm{gap}_t^s>w\!\left(D_{t,a_t^s}^{s,s'}+
                    \max_{c\ne a_t^s}D_{t,c}^{s,s'}\right)
\quad\Longrightarrow\quad
a_t^{s'}=a_t^s.
\label{eq:space-stability}
\end{equation}
\end{proposition}

Thus $J_{t-1}^s=J_{t-1}^{s'}=I_{t-1}^K$ and the population term in
\cref{eq:geometry-population-decomp} vanishes exactly. The intervention replaces only the
encoder-induced cosine similarities of matched query--key pairs; it cannot improve by retaining easier samples or
receiving different pseudo-labels. Spaces inducing the same similarities on these pairs
therefore produce identical logits and predictions, irrespective of model name.
\Cref{eq:space-stability} also localizes where an encoder swap can matter: only samples whose
fused decision margin is smaller than the induced kernel perturbation can change prediction.
This result does not order spaces or guarantee gain; it establishes that the controlled sweep
estimates the effect of replacing the encoder-induced similarity matrix at a fixed cache population. The proof is in
App.~\ref{sec:s_protocol}.
\\
\noindent \textbf{Scope of the theoretical results.}
The decomposition and \cref{prop:space-intervention} are exact identities induced by the
update rule and require no distributional assumptions. By contrast, the margin results in
\cref{eq:margins} provide query-level sufficient conditions under explicit separation
assumptions, while retrieval-space selection based on purity or anti-hubness additionally
relies on \cref{asm:identify}. These results characterize conditions under which retrieval
geometry affects predictions; they do not constitute unconditional performance guarantees.

\vspace{-3mm}
\subsection{Retrieval-Geometry Measurements.}
\label{sec:geometry}
\vspace{-2mm}
Before adaptation, let $S$ be a common calibration sample of $n$ target-domain images and let
$N_\kappa^s(i)$ be the $\kappa$ nearest neighbors of $x_i$ in space $s$, excluding $i$.  We
measure
\begin{equation} \small
P_s(\kappa)=\frac{1}{n\kappa}\sum_{i\in S}\sum_{j\in N_\kappa^s(i)}
\mathbf 1[y_j=y_i],\qquad
A_s(\kappa)=\frac{1}{n}\sum_{j\in S}\mathbf 1[d_j^s(\kappa)=0],
\label{eq:geometry-main}
\end{equation}
where $d_j^s(\kappa)=\sum_{i\in S}\mathbf 1[j\in N_\kappa^s(i)]$ is $k$-NN in-degree.
$P_s$ is neighborhood purity and requires labels; $A_s$ is the anti-hub fraction and is
label-free.  The probe radius $\kappa$ is not the per-class cache capacity $K$; our tables use
$\kappa{=}10$.  Purity measures contaminating \emph{edges}, whereas anti-hubness measures the
fraction of keys with no incoming edge.  They are diagnostics of different failure modes, not
surrogate objectives optimized by the method.  App.~\ref{sec:s_capacity} gives their exact
accounting identities and explains the calibration assumptions needed to connect them to an
online, priority-filtered cache.
\\
For the reported benchmark diagnostics, $S$ is the complete target test split---for example,
$n{=}7{,}500$ on ImageNet-A and $n{=}10{,}000$ on ImageNet-V2---and is the same split used for
evaluation rather than a held-out set. Accordingly, labeled purity is an offline benchmarker's
diagnostic. Pseudo-purity replaces $y_i$ in \cref{eq:geometry-main} with frozen CLIP predictions
and can instead be computed on an unlabeled target calibration batch.
\\
\noindent\textbf{Analysis Quantities.}
For a query of true class $y$, the base and cache margins are
\begin{equation} \small
\mu_t=b_{t,y}-\max_{c\ne y}b_{t,c},\qquad
\delta_t^{s,K}=L_y^{\mathcal M}(q_t)-\max_{c\ne y}L_c^{\mathcal M}(q_t).
\label{eq:margins}
\end{equation}
For $w\ge0$, fusion is correct whenever $\mu_t+w\delta_t^{s,K}>0$; retrieval therefore helps
only when its margin aligns with base-model errors. Under a separated-neighborhood condition,
the cache-margin floor increases with true local purity and decreases with pseudo-label error
and tail mass. Increasing capacity can instead expose a wrong-class key that is more similar
than the newly exposed true-class key, lowering the pointwise margin. App.~\ref{sec:s_capacity}
states these conditional results, their calibration assumptions, and their proofs. We use
purity to diagnose attainable signal and anti-hubness to diagnose useful capacity, not as
unqualified guarantees.
\\
\noindent\textbf{Controlled Evaluation of Retrieval Spaces.}
Because cache-logit scale varies across spaces, controlled comparisons use the finite-grid
envelope
\begin{equation} \small
\mathrm{Acc}^{\star}(s,K) \;=\; \max_{w \in \mathcal{W}} \;
\mathrm{Acc}\big(\{b_t+wL^{\mathcal{M}^{s,K}_{t-1}}(q_t)\}_{t=1}^{T}\big),
\label{eq:envelope}
\end{equation}
over $\mathcal W=\{0,0.1,0.3,1,2,3,5,8,10,15,20,30,50,75,100\}$ and
$\mathcal K=\{1,2,3,4,8,16,32,64\}$. Since $0\in\mathcal W$, zero gain is an oracle floor,
not evidence that every weight is harmless. For the deployable check, we select $w$ on
ImageNet-V2 and apply it to ImageNet-A; this costs at most $0.95$ points and preserves the
purity ranking. App.~\ref{sec:s_transfer} gives the complete transfer protocol.

\endgroup

\vspace{-3mm}
\section{Experiments}
\vspace{-2mm}
\label{sec:protocol}
\noindent\textbf{$\checkmark$ Evaluation Protocol.}
We follow TDA~\citep{karmanov2024tda} on four ImageNet shifts (A, V2, R, and Sketch) and six cross-domain datasets: Caltech101~\citep{feifei2004caltech}, DTD~\citep{cimpoi2014dtd}, EuroSAT~\citep{helber2019eurosat}, Flowers102~\citep{nilsback2008flowers}, Pets~\citep{parkhi2012pets}, and UCF101~\citep{soomro2012ucf101}. Controlled runs use one frozen CLIP ViT-B/16 predictor and fixed prompts. The initial sweep compares CLIP, MAE~\citep{he2022mae}, DINOv2-S/B/L~\citep{oquab2023dinov2}, DINO v1~\citep{caron2021dino}, BEiT~\citep{bao2022beit}, and a supervised ViT~\citep{dosovitskiy2021vit}. We add eight retrieval spaces to cover a wider range of training objectives and architectures: OpenCLIP~\citep{cherti2023openclip}, SigLIP~\citep{zhai2023siglip}, EVA-02 CLIP~\citep{sun2023evaclip}, CLIP ViT-L/14, DINOv3-B~\citep{simeoni2025dinov3}, DeiT-III~\citep{touvron2022deit3}, ConvNeXt-B~\citep{liu2022convnext}, and a supervised ResNet-50~\citep{he2016resnet}. App.~\ref{sec:s_spaces16} describes these comparisons. All encoders are frozen. We use three seeds for retrieval-space comparisons and all reported \method{} benchmark rows. Runtimes are measured on one RTX 3060 after 100 warm-up samples; App.~\ref{sec:s_protocol} gives the full protocol.
\\
\noindent \textbf{$\blacklozenge$ Implementation Details.}
All \method{} configurations are training-free, with both prediction and retrieval encoders frozen throughout inference. Auxiliary-encoder groupings report inference resources separately from this designation. For every pair of retrieval spaces on ImageNet-A and ImageNet-V2, stream order, pseudo-labels, base logits, and base predictions are identical. Only the cache logit changes. Thus, the experiments vary only the encoder-induced query--key similarities, as specified in
\cref{prop:space-intervention}.
\\
\noindent\textbf{$\checkmark$ Final Design.}
After fixing DINOv2 as the retrieval space, we test three additional components: pseudo-label refinement, a CLIP-indexed negative memory, and entropy-adaptive fusion, cumulatively. On ImageNet-A at $M{=}64$, accuracy drops from $65.25\pm0.02$ to $64.12\pm0.26$, $64.14\pm0.25$, and $62.12\pm0.17$, respectively, with the reduced configuration performing best in every seed. Across four OOD splits at $M{=}1$, the three additions change accuracy by $-1.94$, $+0.00$, and $-0.09$ points on average. We therefore retain the simpler read-before-write design: frozen CLIP prediction, a zero-output-weight CLIP admission memory, a frozen DINOv2 retrieval cache, and a single fusion weight. App.~\ref{sec:s_components} details the tested components and reports all splits.

\vspace{-4mm}
\section{Experimental Findings}
\label{sec:experimental-findings}
\vspace{-3mm}
We organize our experimental findings into three tiers. \textit{(i) To show why retrieval space matters,} controlled same-trajectory interventions isolate the effect of the retrieval encoder by keeping predictions, admissions, and retained indices identical (\cref{tab:spaces16,tab:swap}). \textit{(ii) To show how retrieval quality translates into practical gains,} we evaluate retrieval-space selection, transfer, robustness, and accuracy--efficiency trade-offs under controlled budgets (\cref{tab:pareto}). \textit{(iii) To place our results in a broader benchmark context,} we report comparisons under the protocols used by prior work.
\vspace{-3mm}
\subsection{Tier 1: Effect of Retrieval Geometry on Cache Gain}
\label{sec:swap}
\vspace{-2mm}
\noindent\textbf{$\checkmark$ Retrieval Space Drives Cache Gain.}
Our central experiment asks how much cache gain changes when we vary only the feature space used for retrieval. We use $M{=}1$ so that the cache is the only additional evidence. In the held-out comparison, all spaces use $K{=}8$, with $w$ selected on ImageNet-V2 and evaluated on ImageNet-A. The gains range from $+0.00$ for MAE and CLIP to $+0.45$ for DINOv2-S, $+9.73$ for DINOv2-B, and $+19.16$ for DINOv2-L. Thus, the same cache construction and retained samples can produce very different gains solely because of the retrieval space. Because all sixteen spaces share one stream and one retained population, these differences are paired per image: at the diagnostic weight the DINOv2-L cache gains $+18.19$ $[+17.21,+19.17]$ points over the same no-cache predictor while the CLIP cache loses $-10.28$ $[-11.38,-9.18]$ (App.~\ref{sec:s_ci}).
\\
\noindent\textbf{$\checkmark$ Cache Gain Persists Across Capacities.}
\cref{tab:swap} reports three cache-gain measures: gain at a fixed weight, the best gain over the weight grid, and gain after selecting $w$ on a held-out stream; App.~\ref{sec:s_transfer} defines these measures and their use in each column. It also reports pre-adaptation retrieval diagnostics and, at each space's best capacity $\Kstar$, both a fixed-$w$ diagnostic and the oracle envelope selected on ImageNet-A. The oracle sweep preserves the same ordering, with attainable gain ranging from $+0.09$ to $+20.06$, while a CLIP cache reaches only $+0.62$ (App.~\ref{sec:s_swap}). DINOv2-S is smaller than CLIP ViT-B/16 yet has more than twice its gain ceiling, showing that model size alone does not explain the effect. Matched-population margin distributions are given in App.~\ref{sec:s_drawn}.
\\
\noindent\textbf{$\checkmark$ Alternative Retrieval Spaces Remain Weaker.}
Modifying CLIP's retrieval geometry closes only a small part of the gap: mean subtraction and whitening reduce hubness and raise CLIP's best gain from $+0.44$ to $+1.07$, still far below DINOv2-L's $+20.06$ (App.~\ref{sec:s_swap}). The pattern also extends beyond a single encoder family: OpenCLIP and SigLIP perform similarly to CLIP, CLIP ViT-L/14 reaches $+5.52$, and ConvNeXt and ResNet show that the effect is not specific to transformer architectures. Overall, cross-modal encoders rank low unless scaled up (\cref{tab:spaces16}).
\\
\noindent\textbf{$\checkmark$ The Ranking Is Robust Across Seeds and Predictors.}
Across three coupled seeds, Spearman $\rho$ is $0.959$, $0.961$, and $0.967$; DINOv2-L leads every seed with $G^\star=19.74\pm0.40$ (Table cells show seed~0) and zero selection regret. The result is consistent with \cref{tab:swap} ($20.06\pm0.34$). With a CLIP ResNet-50 predictor, the same intervention gives $\rho{=}0.968$, again selecting DINOv2-L with zero regret and spanning $0.00$--$19.04$ points of gain. Thus, the purity--gain ordering is robust across seeds and predictors.

\begin{table}[t]
\centering\footnotesize
\setlength{\tabcolsep}{4pt}
\caption{\textbf{Results of controlled same-trajectory intervention, isolating the effect of retrieval space on cache gain (\textit{Tier 1}).} Sixteen retrieval spaces compared under the \textbf{same retained samples} (K=8, M=1), with each gain reported as the evaluation-stream oracle envelope $G^\star$ maximized over the weight grid. \textit{Purity} and \textit{anti-hubness} are measured before adaptation. Paired $95\%$ intervals, McNemar tests, and calibration-subsample intervals, are given in App.~\ref{sec:s_ci}.}
\label{tab:spaces16}
\adjustbox{max width=\linewidth}{%
\begin{tabular}{lcccrr}
\toprule
Cache space & $\purity@10$\,$\uparrow$ (A) & $\purity@10$\,$\uparrow$ (V2) & $\antihub$\,$\downarrow$ (A) & IN-A $G^\star$\,$\uparrow$ & IN-V2 $G^\star$\,$\uparrow$ \\
\midrule
\blk{6}{Self-supervised (self-distillation)} \\
DINOv2-L & 67.3 & 48.7 & 1.67\% & \best{$+19.36$} & $+3.19$ \\
DINOv2-B & 48.9 & 44.4 & 1.20\% & $+10.27$ & $+2.12$ \\
DINOv3-B & 44.6 & 35.6 & 0.79\% & $+6.49$ & $+0.78$ \\
DINOv2-S & 26.1 & 33.1 & 0.69\% & $+1.24$ & $+0.35$ \\
DINO v1 & 16.8 & 29.1 & 1.95\% & $+0.05$ & $+0.13$ \\
\midrule
\blk{6}{Label-supervised} \\
DeiT-III ViT-B/16 & 46.1 & 55.3 & 2.21\% & $+8.01$ & \best{$+5.06$} \\
Supervised ViT-B/16 & 35.0 & 50.8 & 1.37\% & $+4.43$ & $+3.48$ \\
ConvNeXt-B & 31.0 & 42.1 & 1.16\% & $+2.85$ & $+1.63$ \\
ResNet-50 (sup) & 12.3 & 35.8 & 2.73\% & $+0.04$ & $+1.54$ \\
\midrule
\blk{6}{Cross-modal (language-supervised)} \\
CLIP ViT-L/14 & 42.8 & 28.3 & 5.59\% & $+5.52$ & $+0.06$ \\
EVA-02 CLIP & 38.9 & 33.9 & 2.28\% & $+4.15$ & $+0.10$ \\
SigLIP ViT-B/16 & 29.5 & 29.3 & 5.43\% & $+0.99$ & $+0.00$ \\
OpenCLIP ViT-B/16 & 24.3 & 24.4 & 3.35\% & $+0.76$ & $+0.00$ \\
CLIP ViT-B/16 & 29.7 & 22.9 & 6.27\% & $+0.44$ & $+0.00$ \\
\midrule
\blk{6}{Masked image modelling} \\
MAE ViT-B/16 & 3.2 & 1.5 & 3.44\% & $+0.03$ & $+0.00$ \\
BEiT-B/16 & 19.1 & 29.0 & 11.00\% & $+0.00$ & $+0.00$ \\
\midrule
\emph{Spearman $\rho$(purity, gain)} & & & & \gain{$+0.959$} & \gain{$+0.943$} \\
\emph{Spearman $\rho$($-\antihub$, gain)} & & & & $+0.541$ & $+0.756$ \\
\emph{Selection regret of purity} & & & & \best{$0.00$} & \best{$0.00$} \\
\bottomrule
\end{tabular}%
}
\end{table}

\begin{figure}[t]
\centering
\includegraphics[width=\linewidth]{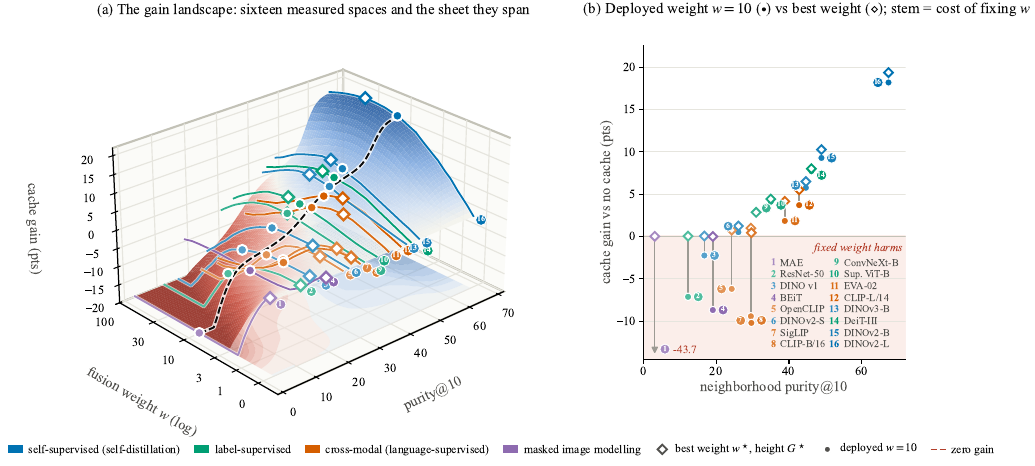}
\caption{\textbf{Cache gain as a function of neighborhood purity and fusion weight on ImageNet-A (\textit{Tier 2}).}
High-purity spaces benefit from a larger cache contribution, whereas low-purity spaces degrade
when the cache is weighted too strongly. Panel (b) compares the fixed choice $w{=}10$ with the
best weight for each space. The surface interpolates between discrete encoder points and does
not assert a continuous functional dependence.}
\label{fig:s_landscape}
\end{figure}

\vspace{-3mm}
\subsection{Tier 2: Retrieval Quality, Transfer, and Efficiency}
\label{sec:capacity}
\vspace{-2mm}
\noindent\textbf{$\checkmark$ Purity predicts gain, while anti-hubness predicts capacity.}
Useful capacity varies by more than an order of magnitude across retrieval spaces (\cref{tab:swap}). Purity strongly ranks attainable gain ($\rho{=}0.90$) but is less predictive of $\Kstar$ ($\rho{=}0.36$), whereas negative anti-hubness ranks $\Kstar$ at $\rho{=}0.87$ but gain at only $\rho{=}0.50$. On held-out ImageNet-V2, the corresponding correlations are $0.98$ and $0.97$. Calibration-subsample intervals separate both arms of this dissociation on ImageNet-A rather than leaving it as an observed ordering: purity ranks gain at $0.900$ $[0.900,0.900]$ against $0.547$ $[0.500,0.718]$ for anti-hubness, while anti-hubness ranks $\Kstar$ at $0.817$ $[0.718,0.872]$ against $0.359$ $[0.359,0.359]$ for purity (App.~\ref{sec:s_ci}). \Cref{fig:s_landscape} shows the joint effect of purity and fusion weight, while \cref{eq:margins} explains the process.
\\
\noindent\textbf{$\checkmark$ Purity predicts cache gain and enables label-free selection.}
Across eight spaces, purity ranks cache gain at $\rho{=}0.93$ on ImageNet-A and $0.98$ on ImageNet-V2, with $\rho\ge0.80$ on five of six cross-domain streams; EuroSAT is the main exception due to similar candidate purity. Purity-based selection incurs only $0.35$ points of mean regret, while pseudo-purity closely matches labeled rankings ($\rho{=}0.98$/$1.00$ on ImageNet-A/V2; $0.93$ across eight streams), selecting the best encoder on $5/8$ streams with $1.48$ points of mean regret. Jointly selecting $s$ by pseudo-purity and transferring $(K,w)$ from a disjoint labeled stream retains about $90\%$ and $98\%$ of oracle gain on ImageNet-A/V2.
\\
\noindent\textbf{$\checkmark$ The trend generalizes across representations, streams, and admission policies.}
Projecting DINOv2-B to $32$--$512$ dimensions preserves a $\rho{=}0.94$ purity--gain correlation. Across sixteen spaces, the correlation is $0.959$ on ImageNet-A and $0.943$ on ImageNet-V2; purity selects different winners across streams, confirming that the best retrieval space is stream-dependent. Among harder ImageNet-A choices spanning $4.4$--$10.0$ points of gain, purity still gives $\rho{=}0.943$ with zero regret. Excluding zero-gain ties, the correlation rises to $0.971$/$0.973$ on ImageNet-A/V2. Finally, across five admission policies, DINOv2-L maintains an approximately $9.6$-point advantage over DINOv2-B, with gains of at least $18.3$ and $8.7$ points, respectively, under random/FIFO admission. Thus, the retrieval-space effect is robust to both representation scale and cache-admission policy; anti-hubness remains useful primarily for diagnosing capacity.
\\
\noindent\textbf{$\checkmark$ Diminishing Cache Gains with More Views.}
We sweep $M\in{1,2,4,8,16,32,64}$ with no cache, a CLIP cache, or a DINOv2 cache and compute the exact difference-in-differences from \cref{eq:substitution}. Across three ImageNet-A seeds, DINOv2 cache gain drops from $+9.30\pm0.38$ points at one view to $+6.56\pm0.10$ at $64$ views, giving an interaction contrast of $-2.74\pm0.33$, with the same direction in every seed; the one-to-four-view contrast is likewise negative at $-2.02\pm0.28$ across all seeds. The repair sets show the same pattern: a single-view cache repairs $1{,}074\pm21$ baseline errors versus $878\pm4$ for $64$ views, at roughly one seventh of the cost, with $426\pm14$ shared repairs ($48.5\pm1.4\%$ overlap). The corresponding regression sets contain $377\pm12$ and $289\pm3$ images, with $80\pm2$ shared, while repair--regression cross-sets have cardinality zero by the disjoint base-error/base-correct partition. App.~\ref{sec:s_substitution} provides the exact contrast, full sweeps, warm-up, and more information.
\\
\noindent\textbf{$\checkmark$ Accuracy--Efficiency Trade-off.}
A single-view DINOv2-B cache exceeds the $64$-view no-cache ensemble ($60.14\pm0.38$ vs.\ $58.69\pm0.09$, both over three seeds) at $6.9\times$ lower measured cost. At matched $M{=}8$, auxiliary-encoder scale, and hardware, \method{} reaches $64.17\pm0.31\%$ at $0.070$~s/image versus $62.75\pm0.15\%$ at $0.181$~s/image for COSMIC. This gives a $1.42$-point accuracy advantage at $2.6\times$ $[2.52,2.64]$ lower latency. A matched evaluation of both systems on three shared stream orders yields a per-image paired advantage of $+1.12$ $[+0.66,+1.58]$ points (App.~\ref{sec:s_ci}); App.~\ref{sec:s_timing} reports the complete Pareto table and three timing sessions (\cref{tab:pareto,tab:s_timing}).

\begin{table}[t]
\centering\footnotesize
\setlength{\tabcolsep}{5.5pt}
\renewcommand{\arraystretch}{1.08}
\caption{\textbf{Comparison with SOTA frameworks under their reported protocols.} Top-1 accuracy (\%) on ImageNet and its OOD variants, evaluating the overall effectiveness of \method{} (\textit{Tier 3}). Prompt sets, view budgets, and implementations follow the cited methods unless marked $\dagger$. \textbf{Bold} denotes the best training-free result overall; underlining denotes the best training-free result using CLIP as the sole visual encoder. \colorbox{bestgreen}{Green} marks the best result among methods using frozen DINOv2-B. Blocks indicate the visual-encoder resources used by each method. Shaded rows denote \method{}, $\dagger$ denotes our runs, and subscripts report standard deviations over three seeds.}
\label{tab:ood}
\adjustbox{max width=\linewidth}{%
\begin{tabular}{lccccccc}
\toprule
Method & ImageNet\,$\uparrow$ & IN-A\,$\uparrow$ & IN-V2\,$\uparrow$ & IN-R\,$\uparrow$ & IN-Sketch\,$\uparrow$ & Average\,$\uparrow$ & OOD Average\,$\uparrow$ \\
\midrule
\blk{8}{CLIP ResNet-50 backbone} \\
CLIP-RN50                                   & 58.16 & 21.83 & 51.41 & 56.15 & 33.37 & 44.18 & 40.69 \\
\midrule
\sub{8}{Methods with parameter adaptation} \\
CoOp~\citep{zhou2022coop}                   & 63.33 & 23.06 & 55.40 & 56.60 & 34.67 & 46.61 & 42.43 \\
CoCoOp~\citep{zhou2022cocoop}               & 62.81 & 23.32 & 55.72 & 57.74 & 34.48 & 46.81 & 42.82 \\
TPT~\citep{shu2022tpt}                      & 60.74 & 26.67 & 54.70 & 59.11 & 35.09 & 47.26 & 43.89 \\
DiffTPT~\citep{feng2023difftpt}             & 60.80 & 31.06 & 55.80 & 58.80 & 37.10 & 48.71 & 45.69 \\
\midrule
\sub{8}{Training-free methods: CLIP visual encoder} \\
TDA~\citep{karmanov2024tda}  & 61.35 & 30.29 & 55.54 & 62.58 & 38.12 & 49.58 & 46.63 \\
ETTA~\citep{etta2025}                        & \underline{61.80} & 31.17 & 55.87 & \underline{62.67} & 38.33 & \underline{49.97} & 47.01 \\
BoostAdapter~\citep{zhang2024boostadapter}   & \na   & \underline{35.12} & \underline{56.14} & 62.66 & \underline{38.87} & \na   & \underline{48.20} \\
\midrule
\sub{8}{Training-free methods: CLIP + frozen DINOv2-B} \\
\ours \textbf{\method{}(Ours)} (DINOv2-B), $M{=}1$$^\dagger$  & 71.48\ci{0.18} & 35.52\ci{0.48} & 58.43\ci{0.17} & 70.16\ci{0.04} & 47.88\ci{0.15} & 56.69\ci{0.03} & 53.00\ci{0.08} \\
\ours \textbf{\method{}(Ours)} (DINOv2-B), $M{=}64$$^\dagger$ & \best{\textbf{71.92\ci{0.18}}} & \best{\textbf{39.07\ci{0.44}}} & \best{\textbf{58.98\ci{0.21}}} & \best{\textbf{70.45\ci{0.17}}} & \best{\textbf{48.65\ci{0.21}}} & \best{\textbf{57.82\ci{0.16}}} & \best{\textbf{54.29\ci{0.23}}} \\
\midrule
\midrule
\blk{8}{CLIP ViT-B/16 backbone} \\
CLIP-ViT-B/16                               & 66.73 & 47.87 & 60.86 & 73.98 & 46.09 & 59.11 & 57.20 \\
\midrule
\sub{8}{Methods with parameter adaptation} \\
CoOp~\citep{zhou2022coop}                   & 71.51 & 49.71 & 64.20 & 75.21 & 47.99 & 61.72 & 59.28 \\
CoCoOp~\citep{zhou2022cocoop}               & 71.02 & 50.63 & 64.07 & 76.18 & 48.75 & 62.13 & 59.91 \\
TPT~\citep{shu2022tpt}                      & 68.98 & 54.77 & 63.45 & 77.06 & 47.94 & 62.44 & 60.81 \\
DiffTPT~\citep{feng2023difftpt}             & 70.30 & 55.68 & 65.10 & 75.00 & 46.80 & 62.58 & 60.65 \\
\midrule
\sub{8}{Training-free methods: CLIP visual encoder} \\
TDA~\citep{karmanov2024tda}  & 69.51 & 60.11 & 64.67 & 80.24 & 50.54 & 65.01 & 63.89 \\
ZERO+Ens~\citep{zero2024} & \underline{71.17} & 62.75 & 65.23 & 80.75 & 50.59 & 66.10 & 64.83 \\

BoostAdapter~\citep{zhang2024boostadapter}   & \na   & \underline{64.53} & \underline{65.51} & 80.95 & 51.28 & \na   & 65.57 \\
SCAP~\citep{zhang2025scap}                   & \na   & 64.52 & 64.65 & 81.68 & 51.65 & \na   & \underline{65.63} \\
TaTa~\citep{zhang2026tata}                   & 70.63 & 61.87 & 65.37 & \underline{81.78} & \underline{52.39} & \underline{66.41} & 65.35 \\
\midrule
\sub{8}{Training-free methods: CLIP + frozen DINOv2-B} \\
COSMIC~\citep{huang2025cosmic} (DINOv2-B), $M{=}8$$^\dagger$ & \best{\textbf{76.39}} & 62.75\ci{0.15} & \best{\textbf{67.54}} & 82.39 & \best{\textbf{57.95}} & \best{\textbf{69.40}} & 67.66 \\
\ours \textbf{\method{}(Ours)} (DINOv2-B), $M{=}4$$^\dagger$  & 74.53\ci{0.04} & 63.00\ci{0.42} & 65.73\ci{0.08} & 82.18\ci{0.09} & 56.26\ci{0.04} & 68.34\ci{0.08} & 66.79\ci{0.10} \\
\ours \textbf{\method{}(Ours)} (DINOv2-B), $M{=}64$$^\dagger$ & 74.84\ci{0.07} & \best{\textbf{65.25\ci{0.02}}} & 66.50\ci{0.09} & \best{\textbf{83.15\ci{0.02}}} & 56.76\ci{0.15} & 69.30\ci{0.04} & \best{\textbf{67.91\ci{0.06}}} \\
\bottomrule
\end{tabular}%
}
\end{table}

\vspace{-3mm}
\subsection{Tier 3: Benchmark Comparison}
\vspace{-2.5mm}
\label{sec:benchmark}
\noindent\textbf{$\checkmark$ \method{} Improves OOD Accuracy and Transfers Across Domains.}
\Cref{tab:ood} compares \method{} with CoOp, CoCoOp, TPT, DiffTPT, and other training-free methods under their reported protocols. Under reported protocols, \method{} reaches a $53.00\pm0.08$ OOD average with a ResNet-50 predictor and one view, versus $48.20$ for the strongest listed training-free baseline; with ViT-B/16, four views reach $66.79$. Across six cross-domain datasets, it averages $74.96$ at four views and $75.45$ at $64$ views, with cache gains ranging from $-0.48$ to $+11.20$ points (App.~\ref{sec:s_benchmark}, \cref{tab:crossdomain}); on Caltech101 and DTD the paired interval includes zero at both view budgets, so the cache has no detectable effect there (App.~\ref{sec:s_ci}). On two additional streams, the average cache gain falls to $+0.84$ versus $+3.32$ on the original six (\cref{tab:cd_extra}).
\\
\noindent\textbf{$\checkmark$ Gains Depend on View Budget and Evaluation Protocol.}
For ResNet-50, increasing from one to $64$ views adds only $1.29$ points, while four views reach $71.48$ on ImageNet-val. The dense $50$-per-class stream favors cache methods with $K{=}16$, so we treat this in-distribution result as protocol-dependent rather than evidence of robustness. Likewise, SOTA comparisons follow each method's reported protocol; notably, the $2.28$-point comparison with SCAP uses auxiliary DINOv2-B retrieval for \method{} versus a CLIP-only baseline, with an auxiliary-matched comparison reported separately.

\vspace{-4mm}
\section{Discussion}
\label{sec:discussion}\label{sec:limitations}
\vspace{-3mm}
Our findings suggest a broader view of cache-based adaptation: the value of a memory depends not only on what it stores, but also on the geometry through which its contents become accessible. This separation is practically important because retrieval can adapt the evidence available to a frozen foundation model without parameter updates, making it attractive for changing visual environments where repeated retraining is costly or unavailable. Across domains, cache contributions range from $-0.48$ to $+11.20$ points, indicating greater utility when the target shift leaves correction headroom and incoming data provide complementary neighborhood structure. This setting is relevant to remote sensing, scientific imaging, robotics, and other streaming applications where recent observations can provide useful local evidence. Retrieval also changes compute allocation: its gain decreases from $+9.30$ points with one view to $+6.56$ with $64$ views, suggesting partial overlap between retrieval and augmentation and favoring retrieval under constrained inference budgets.\\
Purity and anti-hubness provide complementary diagnostics of gain and useful capacity, while closely matched spaces motivate richer label-free selection criteria. Our evaluation spans two predictor backbones, diverse retrieval spaces, four ImageNet shifts, and multiple cross-domain datasets; broader model families, modalities, and non-stationary streams offer natural extensions. These findings suggest establishing retrieval quality before allocating additional compute to large view ensembles or complex adaptation machinery. Ultimately, our results position retrieval geometry as a practical interface between a fixed foundation model and a changing world, providing a path toward adaptation that is lightweight, modular, and responsive to the structure of incoming data.

\vspace{-4mm}
\section{Conclusion}
\label{sec:conclusion}
\vspace{-3mm}
Cache-based CLIP adaptation typically uses the same representation for prediction and retrieval, despite their different geometric requirements. We show that changing only the retrieval space can shift cache gain from nearly zero to about $20$ points, with consistent effects across seeds, architectures, scales, and admission policies. Stronger retrieval also reduces the value of additional views, suggesting that better memory geometry can outperform added inference complexity. Future work should study label-free retrieval selection across shifts, modalities, and applications such as remote sensing and scientific imaging. More broadly, retrieval space should be treated as a first-order design dimension for memory-based adaptation, particularly in changing visual environments where retraining is impractical, opening a new direction for building memory-based adaptation systems.

\vspace{-2mm}
\section*{LLM Usage Disclosure}
\vspace{-2mm}
The authors used large language model (LLM)-assisted tools for non-experimental support tasks during research and manuscript preparation. Specifically, these tools were used for grammar checking, language editing, code debugging, figure ideation, and visualization assistance. In addition, agentic review tools, including the ICLR-provided PAT framework and Stanford Agentic AI tools, were used to support manuscript review and revision by identifying potential issues, suggesting improvements, and providing additional critical feedback on the presentation and clarity of the work. All AI-assisted suggestions, reviews, and outputs were manually evaluated by the authors and, where relevant, independently tested or verified before incorporation. The authors retain full responsibility for the integrity and correctness of the work and for the interpretation of all reported results and analyses.

\section*{Acknowledgements}

We sincerely thank \textbf{\href{https://scholar.google.com/citations?user=4YhNJBEAAAAJ&hl=en}{Prof. Pietro Liò, University of Cambridge}} for his early supervision, insightful discussions about this research, and valuable support and feedback. His guidance helped shape the direction of the research and improve the quality of the manuscript.

We also thank \textbf{Sharjil Khan} and \textbf{Mahiyat Nawar Mantaqa} for their contributions during the initial phase and exploration of the ideas that led to this work.

\bibliography{marc}

\begin{thebibliography}{45}
\providecommand{\natexlab}[1]{#1}
\providecommand{\url}[1]{\texttt{#1}}
\expandafter\ifx\csname urlstyle\endcsname\relax
  \providecommand{\doi}[1]{doi: #1}\else
  \providecommand{\doi}{doi: \begingroup \urlstyle{rm}\Url}\fi

\bibitem[Bao et~al.(2022)Bao, Dong, Piao, and Wei]{bao2022beit}
Hangbo Bao, Li~Dong, Songhao Piao, and Furu Wei.
\newblock {BEiT}: {BERT} pre-training of image transformers.
\newblock In \emph{International Conference on Learning Representations
  (ICLR)}, 2022.

\bibitem[Boudiaf et~al.(2022)Boudiaf, Mueller, Ayed, and
  Bertinetto]{boudiaf2022lame}
Malik Boudiaf, Romain Mueller, Ismail~Ben Ayed, and Luca Bertinetto.
\newblock Parameter-free online test-time adaptation.
\newblock In \emph{2022 IEEE/CVF Conference on Computer Vision and Pattern
  Recognition (CVPR)}, pp.\  8334--8343, 2022.
\newblock \doi{10.1109/CVPR52688.2022.00816}.

\bibitem[Caron et~al.(2021)Caron, Touvron, Misra, J{\'e}gou, Mairal,
  Bojanowski, and Joulin]{caron2021dino}
Mathilde Caron, Hugo Touvron, Ishan Misra, Herv{\'e} J{\'e}gou, Julien Mairal,
  Piotr Bojanowski, and Armand Joulin.
\newblock Emerging properties in self-supervised vision transformers.
\newblock In \emph{Proceedings of the IEEE/CVF international conference on
  computer vision}, pp.\  9650--9660, 2021.

\bibitem[Cherti et~al.(2023)Cherti, Beaumont, Wightman, Wortsman, Ilharco,
  Gordon, Schuhmann, Schmidt, and Jitsev]{cherti2023openclip}
Mehdi Cherti, Romain Beaumont, Ross Wightman, Mitchell Wortsman, Gabriel
  Ilharco, Cade Gordon, Christoph Schuhmann, Ludwig Schmidt, and Jenia Jitsev.
\newblock Reproducible scaling laws for contrastive language-image learning.
\newblock In \emph{2023 IEEE/CVF Conference on Computer Vision and Pattern
  Recognition (CVPR)}, pp.\  2818--2829, 2023.
\newblock \doi{10.1109/CVPR52729.2023.00276}.

\bibitem[Cimpoi et~al.(2014)Cimpoi, Maji, Kokkinos, Mohamed, and
  Vedaldi]{cimpoi2014dtd}
Mircea Cimpoi, Subhransu Maji, Iasonas Kokkinos, Sammy Mohamed, and Andrea
  Vedaldi.
\newblock Describing textures in the wild.
\newblock In \emph{2014 IEEE Conference on Computer Vision and Pattern
  Recognition}, pp.\  3606--3613, 2014.
\newblock \doi{10.1109/CVPR.2014.461}.

\bibitem[Dastmalchi et~al.(2025)Dastmalchi, An, and Cheraghian]{etta2025}
Hamidreza Dastmalchi, Aijun An, and Ali Cheraghian.
\newblock Etta: Efficient test-time adaptation for vision-language models
  through dynamic embedding updates.
\newblock In \emph{36th British Machine Vision Conference 2025, {BMVC} 2025,
  Sheffield, UK, November 24-27, 2025}. BMVA, 2025.

\bibitem[Deng et~al.(2009)Deng, Dong, Socher, Li, Li, and
  Fei-Fei]{deng2009imagenet}
Jia Deng, Wei Dong, Richard Socher, Li-Jia Li, Kai Li, and Li~Fei-Fei.
\newblock Imagenet: A large-scale hierarchical image database.
\newblock In \emph{2009 IEEE Conference on Computer Vision and Pattern
  Recognition}, pp.\  248--255, 2009.
\newblock \doi{10.1109/CVPR.2009.5206848}.

\bibitem[Farina et~al.(2024)Farina, Franchi, Iacca, Mancini, and
  Ricci]{zero2024}
Matteo Farina, Gianni Franchi, Giovanni Iacca, Massimiliano Mancini, and Elisa
  Ricci.
\newblock Frustratingly easy test-time adaptation of vision-language models.
\newblock In \emph{Proceedings of the 38th International Conference on Neural
  Information Processing Systems}, NIPS '24, Red Hook, NY, USA, 2024. Curran
  Associates Inc.
\newblock ISBN 9798331314385.

\bibitem[Fei-Fei et~al.(2004)Fei-Fei, Fergus, and Perona]{feifei2004caltech}
Li~Fei-Fei, R.~Fergus, and P.~Perona.
\newblock Learning generative visual models from few training examples: An
  incremental bayesian approach tested on 101 object categories.
\newblock In \emph{2004 Conference on Computer Vision and Pattern Recognition
  Workshop}, pp.\  178--178, 2004.
\newblock \doi{10.1109/CVPR.2004.383}.

\bibitem[Feng et~al.(2023)Feng, Yu, Liu, Khan, and Zuo]{feng2023difftpt}
Chun-Mei Feng, Kai Yu, Yong Liu, Salman Khan, and Wangmeng Zuo.
\newblock Diverse data augmentation with diffusions for effective test-time
  prompt tuning.
\newblock In \emph{Proceedings of the IEEE/CVF International Conference on
  Computer Vision}, pp.\  2704--2714, 2023.

\bibitem[Han et~al.(2025)Han, Yang, Wang, Li, Xu, Shou, and Zhang]{han2024dota}
Zongbo Han, Jialong Yang, Guangyu Wang, Junfan Li, Qianli Xu, Mike~Zheng Shou,
  and Changqing Zhang.
\newblock Dota: Distributional test-time adaptation of vision-language models.
\newblock In D.~Belgrave, C.~Zhang, H.~Lin, R.~Pascanu, P.~Koniusz,
  M.~Ghassemi, and N.~Chen (eds.), \emph{Advances in Neural Information
  Processing Systems}, volume~38, pp.\  142772--142798. Curran Associates,
  Inc., 2025.

\bibitem[He et~al.(2016)He, Zhang, Ren, and Sun]{he2016resnet}
Kaiming He, Xiangyu Zhang, Shaoqing Ren, and Jian Sun.
\newblock Deep residual learning for image recognition.
\newblock In \emph{2016 IEEE Conference on Computer Vision and Pattern
  Recognition (CVPR)}, pp.\  770--778, 2016.
\newblock \doi{10.1109/CVPR.2016.90}.

\bibitem[He et~al.(2022)He, Chen, Xie, Li, Doll{\'a}r, and Girshick]{he2022mae}
Kaiming He, Xinlei Chen, Saining Xie, Yanghao Li, Piotr Doll{\'a}r, and Ross
  Girshick.
\newblock Masked autoencoders are scalable vision learners.
\newblock In \emph{Proceedings of the IEEE/CVF conference on computer vision
  and pattern recognition}, pp.\  16000--16009, 2022.

\bibitem[Helber et~al.(2019)Helber, Bischke, Dengel, and
  Borth]{helber2019eurosat}
Patrick Helber, Benjamin Bischke, Andreas Dengel, and Damian Borth.
\newblock Eurosat: A novel dataset and deep learning benchmark for land use and
  land cover classification.
\newblock \emph{IEEE Journal of Selected Topics in Applied Earth Observations
  and Remote Sensing}, 12\penalty0 (7):\penalty0 2217--2226, 2019.
\newblock \doi{10.1109/JSTARS.2019.2918242}.

\bibitem[Hendrycks et~al.(2020)Hendrycks, Mu, Cubuk, Zoph, Gilmer, and
  Lakshminarayanan]{hendrycks2020augmix}
Dan Hendrycks, Norman Mu, Ekin~D. Cubuk, Barret Zoph, Justin Gilmer, and Balaji
  Lakshminarayanan.
\newblock {AugMix}: A simple data processing method to improve robustness and
  uncertainty.
\newblock \emph{Proceedings of the International Conference on Learning
  Representations (ICLR)}, 2020.

\bibitem[Hendrycks et~al.(2021{\natexlab{a}})Hendrycks, Basart, Mu, Kadavath,
  Wang, Dorundo, Desai, Zhu, Parajuli, Guo, et~al.]{hendrycks2021imagenetr}
Dan Hendrycks, Steven Basart, Norman Mu, Saurav Kadavath, Frank Wang, Evan
  Dorundo, Rahul Desai, Tyler Zhu, Samyak Parajuli, Mike Guo, et~al.
\newblock The many faces of robustness: A critical analysis of
  out-of-distribution generalization.
\newblock In \emph{Proceedings of the IEEE/CVF international conference on
  computer vision}, pp.\  8340--8349, 2021{\natexlab{a}}.

\bibitem[Hendrycks et~al.(2021{\natexlab{b}})Hendrycks, Zhao, Basart,
  Steinhardt, and Song]{hendrycks2021imageneta}
Dan Hendrycks, Kevin Zhao, Steven Basart, Jacob Steinhardt, and Dawn Song.
\newblock Natural adversarial examples.
\newblock In \emph{Proceedings of the IEEE/CVF conference on computer vision
  and pattern recognition}, pp.\  15262--15271, 2021{\natexlab{b}}.

\bibitem[Hu et~al.(2024)Hu, Zhang, Sun, Chen, Kuo, and Nevatia]{hu2024bafta}
Xuefeng Hu, Ke~Zhang, Min Sun, Albert Chen, Cheng-Hao Kuo, and Ram Nevatia.
\newblock Bafta: Backprop-free test-time adaptation for zero-shot
  vision-language models.
\newblock \emph{arXiv preprint arXiv:2406.11309}, 2024.

\bibitem[Huang et~al.(2025)Huang, Jiang, Jiang, Li, Khan, and
  Wang]{huang2025cosmic}
Fanding Huang, Jingyan Jiang, Qinting Jiang, Hebei Li, Faisal~Nadeem Khan, and
  Zhi Wang.
\newblock Cosmic: Clique-oriented semantic multi-space integration for robust
  clip test-time adaptation.
\newblock In \emph{Proceedings of the Computer Vision and Pattern Recognition
  Conference}, pp.\  9772--9781, 2025.

\bibitem[Karmanov et~al.(2024)Karmanov, Guan, Lu, El~Saddik, and
  Xing]{karmanov2024tda}
Adilbek Karmanov, Dayan Guan, Shijian Lu, Abdulmotaleb El~Saddik, and Eric
  Xing.
\newblock { Efficient Test-Time Adaptation of Vision-Language Models }.
\newblock In \emph{2024 IEEE/CVF Conference on Computer Vision and Pattern
  Recognition (CVPR)}, pp.\  14162--14171. IEEE Computer Society, 2024.
\newblock \doi{10.1109/CVPR52733.2024.01343}.

\bibitem[Kolesnikov et~al.(2021)Kolesnikov, Dosovitskiy, Weissenborn, Heigold,
  Uszkoreit, Beyer, Minderer, Dehghani, Houlsby, Gelly, Unterthiner, and
  Zhai]{dosovitskiy2021vit}
Alexander Kolesnikov, Alexey Dosovitskiy, Dirk Weissenborn, Georg Heigold,
  Jakob Uszkoreit, Lucas Beyer, Matthias Minderer, Mostafa Dehghani, Neil
  Houlsby, Sylvain Gelly, Thomas Unterthiner, and Xiaohua Zhai.
\newblock An image is worth 16x16 words: Transformers for image recognition at
  scale.
\newblock In \emph{International Conference on Learning Representations}, 2021.

\bibitem[Liang et~al.(2020)Liang, Hu, and Feng]{liang2020shot}
Jian Liang, Dapeng Hu, and Jiashi Feng.
\newblock Do we really need to access the source data? source hypothesis
  transfer for unsupervised domain adaptation.
\newblock In \emph{International Conference on Machine Learning (ICML)}, pp.\
  6028--6039, 2020.

\bibitem[Liu et~al.(2022)Liu, Mao, Wu, Feichtenhofer, Darrell, and
  Xie]{liu2022convnext}
Zhuang Liu, Hanzi Mao, Chao-Yuan Wu, Christoph Feichtenhofer, Trevor Darrell,
  and Saining Xie.
\newblock A convnet for the 2020s.
\newblock In \emph{Proceedings of the IEEE/CVF conference on computer vision
  and pattern recognition}, pp.\  11976--11986, 2022.

\bibitem[Nilsback \& Zisserman(2008)Nilsback and
  Zisserman]{nilsback2008flowers}
Maria-Elena Nilsback and Andrew Zisserman.
\newblock Automated flower classification over a large number of classes.
\newblock In \emph{2008 Sixth Indian conference on computer vision, graphics \&
  image processing}, pp.\  722--729. IEEE, 2008.

\bibitem[Oquab et~al.(2023)Oquab, Darcet, Moutakanni, Vo, Szafraniec, Khalidov,
  Fernandez, Haziza, Massa, El-Nouby, et~al.]{oquab2023dinov2}
Maxime Oquab, Timoth{\'e}e Darcet, Th{\'e}o Moutakanni, Huy Vo, Marc
  Szafraniec, Vasil Khalidov, Pierre Fernandez, Daniel Haziza, Francisco Massa,
  Alaaeldin El-Nouby, et~al.
\newblock Dinov2: Learning robust visual features without supervision.
\newblock \emph{arXiv preprint arXiv:2304.07193}, 2023.

\bibitem[Parkhi et~al.(2012)Parkhi, Vedaldi, Zisserman, and
  Jawahar]{parkhi2012pets}
Omkar~M Parkhi, Andrea Vedaldi, Andrew Zisserman, and CV~Jawahar.
\newblock Cats and dogs.
\newblock In \emph{2012 IEEE conference on computer vision and pattern
  recognition}, pp.\  3498--3505. IEEE, 2012.

\bibitem[Radford et~al.(2021)Radford, Kim, Hallacy, Ramesh, Goh, Agarwal,
  Sastry, Askell, Mishkin, Clark, Krueger, and Sutskever]{radford2021clip}
Alec Radford, Jong~Wook Kim, Chris Hallacy, Aditya Ramesh, Gabriel Goh,
  Sandhini Agarwal, Girish Sastry, Amanda Askell, Pamela Mishkin, Jack Clark,
  Gretchen Krueger, and Ilya Sutskever.
\newblock Learning transferable visual models from natural language
  supervision.
\newblock In Marina Meila and Tong Zhang (eds.), \emph{Proceedings of the 38th
  International Conference on Machine Learning}, volume 139 of
  \emph{Proceedings of Machine Learning Research}, pp.\  8748--8763. PMLR,
  18--24 Jul 2021.

\bibitem[Radovanovic et~al.(2010)Radovanovic, Nanopoulos, and
  Ivanovic]{radovanovic2010hubness}
Milos Radovanovic, Alexandros Nanopoulos, and Mirjana Ivanovic.
\newblock Hubs in space: Popular nearest neighbors in high-dimensional data.
\newblock \emph{Journal of machine learning research}, 11\penalty0
  (86):\penalty0 2487--2531, 2010.

\bibitem[Recht et~al.(2019)Recht, Roelofs, Schmidt, and
  Shankar]{recht2019imagenetv2}
Benjamin Recht, Rebecca Roelofs, Ludwig Schmidt, and Vaishaal Shankar.
\newblock Do imagenet classifiers generalize to imagenet?
\newblock In \emph{International conference on machine learning}, pp.\
  5389--5400. PMLR, 2019.

\bibitem[Schuhmann et~al.(2022)Schuhmann, Beaumont, Vencu, Gordon, Wightman,
  Cherti, Coombes, Katta, Mullis, Wortsman, Schramowski, Kundurthy, Crowson,
  Schmidt, Kaczmarczyk, and Jitsev]{schuhmann2022laion5b}
Christoph Schuhmann, Romain Beaumont, Richard Vencu, Cade Gordon, Ross
  Wightman, Mehdi Cherti, Theo Coombes, Aarush Katta, Clayton Mullis, Mitchell
  Wortsman, Patrick Schramowski, Srivatsa Kundurthy, Katherine Crowson, Ludwig
  Schmidt, Robert Kaczmarczyk, and Jenia Jitsev.
\newblock Laion-5b: an open large-scale dataset for training next generation
  image-text models.
\newblock In \emph{Proceedings of the 36th International Conference on Neural
  Information Processing Systems}, NIPS '22, 2022.
\newblock ISBN 9781713871088.

\bibitem[Shu et~al.(2022)Shu, Nie, Huang, Yu, Goldstein, Anandkumar, and
  Xiao]{shu2022tpt}
Manli Shu, Weili Nie, De-An Huang, Zhiding Yu, Tom Goldstein, Anima Anandkumar,
  and Chaowei Xiao.
\newblock Test-time prompt tuning for zero-shot generalization in
  vision-language models.
\newblock NIPS '22, Red Hook, NY, USA, 2022. Curran Associates Inc.
\newblock ISBN 9781713871088.

\bibitem[Sim{\'e}oni et~al.(2025)Sim{\'e}oni, Vo, Khalidov, Seitzer,
  Baldassarre, Oquab, Jose, Szafraniec, Yi, Ramamonjisoa,
  et~al.]{simeoni2025dinov3}
Oriane Sim{\'e}oni, Huy~V Vo, Vasil Khalidov, Maximilian Seitzer, Federico
  Baldassarre, Maxime Oquab, Cijo Jose, Marc Szafraniec, Seungeun Yi,
  Micha{\"e}l Ramamonjisoa, et~al.
\newblock Dinov3.
\newblock \emph{arXiv preprint arXiv:2508.10104}, 2025.

\bibitem[Soomro et~al.(2012)Soomro, Zamir, and Shah]{soomro2012ucf101}
Khurram Soomro, Amir~Roshan Zamir, and Mubarak Shah.
\newblock Ucf101: A dataset of 101 human actions classes from videos in the
  wild.
\newblock \emph{arXiv preprint arXiv:1212.0402}, 2012.

\bibitem[Sun et~al.(2023)Sun, Fang, Wu, Wang, and Cao]{sun2023evaclip}
Quan Sun, Yuxin Fang, Ledell Wu, Xinlong Wang, and Yue Cao.
\newblock Eva-clip: Improved training techniques for clip at scale.
\newblock \emph{arXiv preprint arXiv:2303.15389}, 2023.

\bibitem[Touvron et~al.(2022)Touvron, Cord, and J{\'e}gou]{touvron2022deit3}
Hugo Touvron, Matthieu Cord, and Herv{\'e} J{\'e}gou.
\newblock Deit iii: Revenge of the vit.
\newblock In \emph{European Conference on Computer Vision}, pp.\  516--533.
  Springer, 2022.

\bibitem[Wang et~al.(2021)Wang, Shelhamer, Liu, Olshausen, and
  Darrell]{wang2021tent}
Dequan Wang, Evan Shelhamer, Shaoteng Liu, Bruno Olshausen, and Trevor Darrell.
\newblock Tent: Fully test-time adaptation by entropy minimization.
\newblock In \emph{International Conference on Learning Representations}, 2021.

\bibitem[Wang et~al.(2019)Wang, Ge, Lipton, and Xing]{wang2019imagenetsk}
Haohan Wang, Songwei Ge, Zachary Lipton, and Eric~P Xing.
\newblock Learning robust global representations by penalizing local predictive
  power.
\newblock \emph{Advances in neural information processing systems}, 32, 2019.

\bibitem[Zhai et~al.(2023)Zhai, Mustafa, Kolesnikov, and Beyer]{zhai2023siglip}
Xiaohua Zhai, Basil Mustafa, Alexander Kolesnikov, and Lucas Beyer.
\newblock Sigmoid loss for language image pre-training.
\newblock In \emph{2023 IEEE/CVF International Conference on Computer Vision
  (ICCV)}, pp.\  11941--11952, 2023.
\newblock \doi{10.1109/ICCV51070.2023.01100}.

\bibitem[Zhang et~al.(2025)Zhang, Xu, Liu, Peng, and Zhou]{zhang2025scap}
Chenyu Zhang, Kunlun Xu, Zichen Liu, Yuxin Peng, and Jiahuan Zhou.
\newblock Scap: Transductive test-time adaptation via supportive clique-based
  attribute prompting.
\newblock In \emph{Proceedings of the Computer Vision and Pattern Recognition
  Conference}, pp.\  30032--30041, 2025.

\bibitem[Zhang et~al.(2022{\natexlab{a}})Zhang, Levine, and
  Finn]{zhang2022memo}
Marvin Zhang, Sergey Levine, and Chelsea Finn.
\newblock Memo: test time robustness via adaptation and augmentation.
\newblock In \emph{Proceedings of the 36th International Conference on Neural
  Information Processing Systems}, 2022{\natexlab{a}}.

\bibitem[Zhang et~al.(2022{\natexlab{b}})Zhang, Zhang, Fang, Gao, Li, Dai,
  Qiao, and Li]{zhang2022tipadapter}
Renrui Zhang, Wei Zhang, Rongyao Fang, Peng Gao, Kunchang Li, Jifeng Dai,
  Yu~Qiao, and Hongsheng Li.
\newblock Tip-adapter: Training-free adaption of clip for few-shot
  classification.
\newblock In \emph{Computer Vision -- ECCV 2022}, pp.\  493--510. Springer
  Nature Switzerland, 2022{\natexlab{b}}.

\bibitem[Zhang et~al.(2024)Zhang, Wang, Guo, Dai, Chen, and
  Xia]{zhang2024boostadapter}
Taolin Zhang, Jinpeng Wang, Hang Guo, Tao Dai, Bin Chen, and Shu-Tao Xia.
\newblock Boostadapter: Improving vision-language test-time adaptation via
  regional bootstrapping.
\newblock \emph{Advances in Neural Information Processing Systems},
  37:\penalty0 67795--67825, 2024.

\bibitem[Zhang et~al.(2026)Zhang, Cheng, Aviles-Rivero, He, and
  Zhang]{zhang2026tata}
Yi~Zhang, Chun-Wun Cheng, Angelica~I. Aviles-Rivero, Zhihai He, and Liang-Jie
  Zhang.
\newblock Training-free test-time adaptation with brownian distance covariance
  in vision-language models.
\newblock In \emph{ICASSP 2026 - 2026 IEEE International Conference on
  Acoustics, Speech and Signal Processing (ICASSP)}, pp.\  12782--12786, 2026.
\newblock \doi{10.1109/ICASSP55912.2026.11462803}.

\bibitem[Zhou et~al.(2022{\natexlab{a}})Zhou, Yang, Loy, and
  Liu]{zhou2022cocoop}
Kaiyang Zhou, Jingkang Yang, Chen~Change Loy, and Ziwei Liu.
\newblock Conditional prompt learning for vision-language models.
\newblock In \emph{Proceedings of the IEEE/CVF conference on computer vision
  and pattern recognition}, pp.\  16816--16825, 2022{\natexlab{a}}.

\bibitem[Zhou et~al.(2022{\natexlab{b}})Zhou, Yang, Loy, and Liu]{zhou2022coop}
Kaiyang Zhou, Jingkang Yang, Chen~Change Loy, and Ziwei Liu.
\newblock Learning to prompt for vision-language models.
\newblock \emph{Int. J. Comput. Vision}, 130\penalty0 (9):\penalty0
  2337–2348, 2022{\natexlab{b}}.
\newblock ISSN 0920-5691.
\newblock \doi{10.1007/s11263-022-01653-1}.
\newblock URL \url{https://doi.org/10.1007/s11263-022-01653-1}.

\end{thebibliography}
\bibliographystyle{iclr2027_conference}

\appendix
\clearpage
\newpage

\begin{center}
{\Large\bf Appendix}\\[3mm]
{\large\bf Table of Contents}
\end{center}
\vspace{1em}
\startcontents[sections]
\printcontents[sections]{l}{1}{\setcounter{tocdepth}{3}}

\vspace{2em}
\noindent\fbox{\parbox{\dimexpr\linewidth-2\fboxsep-2\fboxrule\relax}{\small
\textbf{Table conventions.} The following apply to every table in the paper and appendix;
where a caption gives a table-specific meaning, the caption takes precedence.\\[0.5em]
\begin{tabular}{@{}p{0.2\linewidth}p{0.74\linewidth}@{}}
\colorbox{bestgreen}{Green} & Best entry in its column or block. \\
\textcolor{gaingreen}{Green}\,/\,\textcolor{red}{Red} & Favourable\,/\,unfavourable value, e.g.\ a gain versus a loss, or a strong versus a weak correlation. \\
$XX.XX\ci{Y.YY}$ & Mean with standard deviation across seeds. \\
$\uparrow$\,/\,$\downarrow$ & Higher\,/\,lower is better. \\
$\dagger$ & Our runs of the method. \\
\na & Not applicable or not reported. \\
\end{tabular}}}

\clearpage
\newpage

\raggedbottom
\setcounter{table}{0}
\setcounter{figure}{0}
\renewcommand{\thetable}{A\arabic{table}}
\renewcommand{\thefigure}{A\arabic{figure}}
\renewcommand{\theHtable}{appx.\arabic{table}}
\renewcommand{\theHfigure}{appx.\arabic{figure}}


\section{More Details on Related Work}
\label{sec:apx-related}

\subsection{Cache based Adaptation}
Test time adaptation (TTA) traditionally adapts a model to the target stream by updating its parameters, for example through entropy minimization~\citep{wang2021tent} or source free hypothesis transfer~\citep{liang2020shot}. Cache based methods offer a different paradigm: they keep the pretrained model frozen and perform adaptation through an external memory constructed from target examples. Tip Adapter~\citep{zhang2022tipadapter} establishes a representative formulation of this approach. It stores image features and their labels in a key value cache, retrieves examples similar to the current query, and combines the retrieved evidence with the zero shot CLIP logits. TDA~\citep{karmanov2024tda} extends this formulation to the online, unlabeled setting. Instead of relying on labeled support examples, it populates the cache with entropy prioritized pseudo labels from the incoming test stream and maintains separate positive and negative memories for each class. Our method adopts TDA's cache construction and update rule, allowing us to focus specifically on how the choice of retrieval space affects adaptation.
\\
Several subsequent methods improve other components of the cache adaptation pipeline. These include regional bootstrapping~\citep{zhang2024boostadapter}, online per class Gaussian modeling~\citep{han2024dota}, clique based attribute prompting~\citep{zhang2025scap}, SVD based clustering~\citep{hu2024bafta}, and Brownian distance covariance~\citep{zhang2026tata}. COSMIC~\citep{huang2025cosmic}, in particular, combines CLIP and DINOv2 caches through dual semantic graphs and clique guided hyper class queries. While these approaches modify cache construction, representation, or aggregation jointly, our goal is more controlled: we isolate retrieval geometry by varying only the encoder that defines similarity, while keeping cache construction and admission fixed.

\subsection{Self supervised Representations as Retrieval Spaces}
The representation used for retrieval directly determines which cached examples are considered neighbors. Self supervised vision models therefore provide a useful family of alternative retrieval spaces. DINO~\citep{caron2021dino} and DINOv2~\citep{oquab2023dinov2} learn representations through cross view consistency and exhibit strong transfer to tasks that depend directly on image to image similarity, including nearest neighbor classification and retrieval. Their embeddings can consequently induce neighborhood structures that differ substantially from those produced by CLIP's vision encoder.
\\
MAE~\citep{he2022mae}, in contrast, learns representations through masked image reconstruction and generally exhibits weaker $k$ NN structure. This difference gives us a useful testbed for separating representation quality from the mechanics of cache adaptation. Rather than treating the retrieval encoder as an interchangeable implementation detail, we study how encoder induced similarity structures affect cache gain, the number of useful augmented views, and the effective memory capacity.

\subsection{Geometry of Nearest neighbor Retrieval}
The effectiveness of a cache depends not only on what it stores, but also on the geometry that determines which entries are retrieved. Neighborhood purity and hubness~\citep{radovanovic2010hubness} are established measures for characterizing $k$ NN behavior in high dimensional representation spaces. High purity indicates that nearby samples tend to share semantic labels. Hubness captures an orthogonal phenomenon in which a small subset of points appears in many neighbor lists, while anti hubs are rarely or never retrieved.
\\
Prior adaptation work touches on these geometric properties without explicitly treating retrieval space as an experimental variable. BaFTA~\citep{hu2024bafta}, for example, uses $k$ NN accuracy to assess the quality of a CLIP projection. Cache based TTA methods also show that increasing memory capacity introduces a tradeoff between greater sample diversity and increased pseudo label contamination~\citep{han2024dota,karmanov2024tda}. LAME~\citep{boudiaf2022lame} further demonstrates that neighborhood structure itself can support adaptation by constructing a $k$ NN affinity graph over test features without updating model parameters. However, these works largely operate within a fixed feature space. We instead compare retrieval spaces explicitly, characterize them through neighborhood purity and anti hubness, and test whether these geometric properties predict both cache gain and useful memory capacity.

\subsection{Prompt based and Training free CLIP Adaptation}
Test time augmentation provides another mechanism for extracting useful information from a single target image. Augmentation has long been used to improve robustness under distribution shift~\citep{hendrycks2020augmix}, and recent TTA methods use multiple augmented views as an explicit adaptation signal. MEMO~\citep{zhang2022memo}, for instance, adapts predictions by minimizing the marginal entropy across augmented views. In the vision language setting, TPT~\citep{shu2022tpt} and DiffTPT~\citep{feng2023difftpt} use view ensembles for test time prompt optimization, while ZERO~\citep{zero2024} explores training free adaptation from multiple views.
\\
Across these methods, the number of augmented views is typically treated primarily as a protocol or computational choice. In a cache based system, however, additional views can also influence the reliability of pseudo labels and the evidence entering the memory. We therefore study view budget as part of the retrieval problem and examine whether the benefit of additional views changes systematically with the effectiveness of the underlying retrieval space.

\section{Retrieval-Space Comparison: Additional Controls}
\label{sec:s_swap}

\begin{table}[H]
\centering
\scriptsize
\setlength{\tabcolsep}{2pt}
\renewcommand{\arraystretch}{1.05}
\caption{Tier~(i), controlled same-trajectory comparison on ImageNet-A with a CLIP ViT-B/16
predictor, $M{=}1$, and a $50.84\%$ no-cache baseline; only encoder-induced query--key
similarities change. Purity and anti-hubness are measured at $K{=}10$ before adaptation, and fixed-$w$ entries
are separate diagnostic runs. $G^\star$ is the evaluation-stream oracle envelope, not a
deployable result. Oracle
subscripts denote standard deviations over three seeds; transfer uses $K{=}8$ and weights
selected on ImageNet-V2, averaged over three seeds. \Cref{tab:s_spaces16_seeds} reports a second, independent three-seed sweep of the same
protocol, differing only in that the stream permutation is drawn from an explicit
generator rather than the global RNG; the two means agree within seed dispersion
($+20.06\pm0.34$ here versus $+19.74\pm0.40$ there for DINOv2-L; $+9.88\pm0.38$ versus
$+9.97\pm0.29$ for DINOv2-B).
}
\label{tab:swap}
\adjustbox{max width=\linewidth}{
\begin{tabular}{lccccrrr}
\toprule
Cache space & $\purity@10$\,$\uparrow$ & $\antihub$\,$\downarrow$ & $\Kstar$ & $w$
& accuracy ($G_{10}$) & $G^\star$\,$\uparrow$ & $G^{\mathrm{tr}}$\,$\uparrow$ \\
\midrule
\blk{8}{Control spaces} \\
MAE ViT-B/16 & 3.2 & 3.4 & 1 & 0.1 & 32.93 \small{($-17.91$)} & $+0.09$\ci{0.07} & $+0.00$ \\
CLIP ViT-B/16 & 29.7 & 6.3 & 2 & 1 & 46.84 \small{($-4.00$)} & $+0.62$\ci{0.10} & $+0.00$ \\
\midrule
\blk{8}{Self-supervised spaces} \\
DINOv2-S & 26.1 & \best{0.7} & \best{16} & 3 & 50.40 \small{($-0.44$)} & $+1.33$\ci{0.09} & $+0.45$ \\
DINOv2-B & 48.9 & 1.2 & 8 & 15 & 60.01 \small{($+9.17$)} & $+9.88$\ci{0.38} & $+9.73$ \\
\ours DINOv2-L & \best{67.3} & 1.7 & 8 & 20 & 68.99 \small{($+18.15$)} & \best{$+20.06$\ci{0.34}} & \best{$+19.16$} \\
\midrule
\emph{Span, weakest to strongest} & & & & & & \gain{$19.97$} & \gain{$19.16$} \\
\bottomrule
\end{tabular}
}
\end{table}

\subsection{Single versus Parallel Caches}

Every row of Tab.~\ref{tab:swap} uses the named representation as the \emph{only} cache.
A parallel-cache
pipeline instead gives the CLIP arm $51.59$ at $K{=}2$; because this is a different system,
the controlled comparisons use the single-cache result.

\subsection{Why DINOv2 Is Not a Zero-Shot Predictor}

The converse experiment DINOv2 for zero-shot prediction is not directly defined because
DINOv2 has no language supervision or text-aligned head. An unlearned projection into CLIP
text space would evaluate the map rather than the encoder; a fixed random projection obtains
$0.44\%$, approximately chance.

A better-defined test discards CLIP logits at query time and scores only the cache.
\Cref{tab:s_cachealone} shows strong dependence on encoder scale.

\begin{table}[H]
\centering\footnotesize
\setlength{\tabcolsep}{4pt}
\caption{Cache-only prediction on ImageNet-A at $M{=}1$, with CLIP logits discarded at query
time. CLIP still supplies stored pseudo-labels, so this is not an independent classifier.}
\label{tab:s_cachealone}
\begin{tabular}{lcc}
\toprule
Cache scored alone & Top-1\,$\uparrow$ & vs.\ no-cache prompt ensemble \\
\midrule
No-cache, prompt ensemble & 50.84 & \na \\
\midrule
CLIP cache      & 27.00 & \loss{$-23.84$} \\
DINOv2-B cache  & 49.40 & \loss{$-1.44$} \\
\ours DINOv2-L cache & \best{64.20} & \gain{$+13.36$} \\
\bottomrule
\end{tabular}
\end{table}

CLIP still supplies the pseudo-labels, while DINOv2 supplies the image neighborhoods; the
result supports their complementary roles rather than a universally stronger encoder.

\section{Diminishing Marginal Cache Gain: Additional Results}
\label{sec:s_substitution}

Let $A_{M,r}$ be accuracy at view budget $M$ with retrieval indicator $r\in\{0,1\}$. For
$M_1<M_2$, our exact interaction contrast is
\begin{equation}
\begin{aligned}
\operatorname{Sub}(M_1,M_2)
&=[A_{M_2,1}-A_{M_2,0}]-[A_{M_1,1}-A_{M_1,0}]\\
&=[A_{M_2,1}-A_{M_1,1}]-[A_{M_2,0}-A_{M_1,0}].
\end{aligned}
\label{eq:substitution}
\end{equation}
Values below zero mean that retrieval contributes less after increasing the view budget. The
contrast is an algebraic difference in differences that summarizes the measured interaction.

\noindent\textbf{Measurement.}
For fixed stream, retrieval space, capacity, and weight, let
$\Delta_M=\mathrm{Acc}_{\mathrm{cache}}(M)-\mathrm{Acc}_0(M)$. Then
$\operatorname{Sub}(M_1,M_2)=\Delta_{M_2}-\Delta_{M_1}$: values below, at, and above zero
represent decreasing, unchanged, and increasing marginal cache gain, respectively.

\begin{table}[H]
\centering\footnotesize
\setlength{\tabcolsep}{4pt}
\caption{ImageNet-A accuracy across view budgets and cache spaces (CLIP ViT-B/16). $\Delta$
is the fixed-weight gain $G_w$ of the DINOv2-B cache over the same no-cache ensemble; the shaded row uses one view.
The bottom-right entry is $\operatorname{Sub}(1,64)$ from \cref{eq:substitution}. This table
shows the complete seed-0 budget sweep; \cref{tab:s_substitution_seeds} reports three-seed
replication at the anchor budgets. Costs are reported in \cref{tab:pareto}.}
\label{tab:grid}
\begin{tabular}{rcccr}
\toprule
$M$ & no cache\,$\uparrow$ & +CLIP cache\,$\uparrow$ & +DINOv2 cache\,$\uparrow$ & $G_w{=}\Delta$\,$\uparrow$ \\
\midrule
\ours 1  & 50.84 & 49.59 & 59.73 & \best{\gain{$+8.89$}} \\
2  & 53.76 & 52.85 & 61.52 & \gain{$+7.76$} \\
4  & 55.77 & 55.23 & 62.93 & \gain{$+7.16$} \\
8  & 57.20 & 56.20 & 63.96 & \gain{$+6.76$} \\
16 & \best{59.03} & 57.85 & 64.40 & \gain{$+5.37$} \\
32 & 58.21 & 57.69 & 64.72 & \gain{$+6.51$} \\
64 & 58.77 & \best{58.21} & \best{65.23} & \gain{$+6.45$} \\
\midrule
\emph{$\Delta$, $M{=}1 \rightarrow 64$} & \gain{$+7.93$} & \gain{$+8.62$} & \gain{$+5.50$} & \loss{$-2.44$} \\
\bottomrule
\end{tabular}
\end{table}

\begin{table}[H]
\centering\footnotesize
\setlength{\tabcolsep}{5pt}
\caption{Three-seed replication of the ImageNet-A augmentation--retrieval measurements.
Cache gains are $G_w(M)$ and $\operatorname{Sub}(1,M)=G_w(M)-G_w(1)$. Repair counts compare
the $64$-view ensemble and the single-view DINOv2-B cache with the single-view predictor.}
\label{tab:s_substitution_seeds}
\begin{tabular}{lrrrr}
\toprule
Quantity & seed 0 & seed 1 & seed 2 & mean $\pm$ sd \\
\midrule
$G_w(1)$ & $+8.89$ & $+9.65$ & $+9.35$ & $+9.30\pm0.38$ \\
$G_w(4)$ & $+7.16$ & $+7.60$ & $+7.06$ & $+7.27\pm0.29$ \\
$G_w(64)$ & $+6.45$ & $+6.56$ & $+6.66$ & $+6.56\pm0.10$ \\
$\operatorname{Sub}(1,4)$ & $-1.73$ & $-2.05$ & $-2.29$ & $-2.02\pm0.28$ \\
$\operatorname{Sub}(1,64)$ & $-2.44$ & $-3.09$ & $-2.69$ & $-2.74\pm0.33$ \\
Augmentation repairs & 881 & 880 & 873 & $878\pm4$ \\
Retrieval repairs & 1,057 & 1,097 & 1,069 & $1,074\pm21$ \\
Shared repairs & 425 & 440 & 412 & $426\pm14$ \\
Shared / augmentation & $48.2\%$ & $50.0\%$ & $47.2\%$ & $48.5\pm1.4\%$ \\
\bottomrule
\end{tabular}
\end{table}

\noindent\textbf{Repair-set accounting.}
Relative to the one-view no-cache predictor $f_0$, define
\[
E_0=\{t:f_0(x_t)\ne y_t\},\quad
R_u=\{t\in E_0:f_u(x_t)=y_t\},\quad
B_u=\{t\notin E_0:f_u(x_t)\ne y_t\}.
\]
Then
$\mathrm{Acc}(f_u)-\mathrm{Acc}(f_0)=(|R_u|-|B_u|)/T$.
Across three seeds, retrieval repairs $1{,}074\pm21$ baseline errors versus $878\pm4$ for
augmentation, while their intersection contains $426\pm14$ images
($48.5\pm1.4\%$ of augmentation repairs). Together with the gain contrasts, these counts
provide directional evidence that the two sources exploit partially shared errors.

\begin{table}[H]
\centering\footnotesize
\setlength{\tabcolsep}{5pt}
\caption{Joint repair and regression accounting on ImageNet-A over three seeds. The reference
is the one-view cache-free predictor; augmentation uses $64$ views, and retrieval uses a
one-view DINOv2-B cache. Values are mean $\pm$ standard deviation.}
\label{tab:s_repair_regression}
\begin{tabular}{lcccc}
\toprule
Set & Augmentation & Retrieval & CLIP cache & Aug. $\cap$ retrieval \\
\midrule
Repairs $|R_u|$ & $878\pm4$ & $1{,}074\pm21$ & $378\pm30$ & $426\pm14$ \\
Regressions $|B_u|$ & $289\pm3$ & $377\pm12$ & $429\pm16$ & $80\pm2$ \\
\bottomrule
\end{tabular}
\end{table}

The mean repair partition contains $452$ augmentation-only, $648$ retrieval-only, and $426$
shared images. By definition, $R_u\subseteq E_0$ and $B_u\subseteq E_0^c$, so
$|R_{\mathrm{aug}}\cap B_{\mathrm{ret}}|=|R_{\mathrm{ret}}\cap B_{\mathrm{aug}}|=0$.
The identity $(|R_u|-|B_u|)/T$ recovers the measured accuracy change for each source.

\subsection{Results Across OOD Splits}

\Cref{tab:s_ood_sub} repeats the comparison on all four OOD splits; cache gain decreases from
one to four views on each.

\begin{table}[H]
\centering\scriptsize
\setlength{\tabcolsep}{3pt}
\caption{Single-run DINOv2-B cache gain at one and four views across OOD splits (CLIP
ViT-B/16). Cells show no cache\,$\rightarrow$\,+DINOv2; $\Delta_M=G_w(M)$ is fixed-weight
gain, and \emph{shrink} is $\Delta_4-\Delta_1$, which is negative on every split.}
\label{tab:s_ood_sub}
\adjustbox{max width=\linewidth}{%
\begin{tabular}{lccrrr}
\toprule
Split & $M{=}1$ & $M{=}4$ & $G_w(1)$\,$\uparrow$ & $G_w(4)$\,$\uparrow$ & shrink \\
\midrule
IN-A      & 50.84\,$\rightarrow$\,59.73 & 55.77\,$\rightarrow$\,62.93 & \best{\gain{$+8.89$}} & \best{\gain{$+7.16$}} & \loss{$-1.73$} \\
IN-Sketch & 49.27\,$\rightarrow$\,55.49 & 50.63\,$\rightarrow$\,56.23 & \gain{$+6.22$} & \gain{$+5.60$} & \loss{$-0.62$} \\
IN-R      & 77.37\,$\rightarrow$\,81.34 & 79.07\,$\rightarrow$\,82.12 & \gain{$+3.97$} & \gain{$+3.05$} & \loss{$-0.92$} \\
IN-V2     & 63.66\,$\rightarrow$\,65.25 & 64.49\,$\rightarrow$\,65.64 & \gain{$+1.59$} & \gain{$+1.15$} & \loss{$-0.44$} \\
\bottomrule
\end{tabular}%
}
\end{table}

Gain follows shift severity more closely than the CLIP-to-DINOv2 purity gap. ImageNet-V2 has
the largest gap but smallest gain, plausibly because its predictor already reaches $63.66\%$.

\subsection{Joint View-Budget Sweeps}

\Cref{fig:s_substitution} sweeps the view budget for two cache spaces. Both gains decline
overall, giving the same directional pattern, with the stronger space starting higher and
losing more as views increase.

\begin{figure}[H]
\centering
\includegraphics[width=\linewidth]{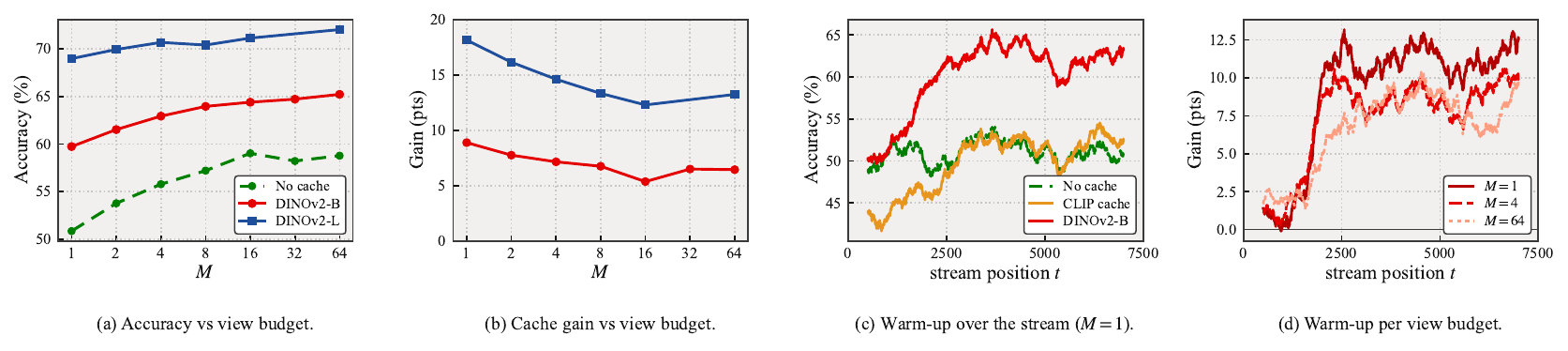}
\caption{View budget and cache warm-up on ImageNet-A, at the fixed weights of
Tab.~\ref{tab:swap}. (a) One-view DINOv2-L reaches $68.99$, above $64$-view DINOv2-B
($65.25$) and no cache ($58.77$). (b) Three-seed mean cache gain changes
$9.30\!\rightarrow\!6.56$ for
DINOv2-B and $18.15\!\rightarrow\!13.24$ for DINOv2-L. (c) At $M{=}1$ the DINOv2-B arm
climbs $13.9$ points from the first to the last decile of the stream while the no-cache
control moves $+0.9$; the climb is the cache being built. (d) The same climb shrinks as
views are added, consistent with augmentation--retrieval overlap inside the stream.}
\label{fig:s_substitution}
\label{fig:s_warmup}
\end{figure}

\subsection{Cache Warm-Up over the Stream}

Because aggregate accuracy mixes retrieval quality with cache maturity, we also examine gain
over the stream (\cref{fig:s_warmup}, panels c and d).

This warm-up also clarifies the ImageNet-val caveat in
Sec.~\ref{sec:benchmark}. With $50$ images per class and
$K{=}16$, the cache is populated for most of the stream; all cache methods benefit from this
density, and high-purity retrieval uses it more effectively.

\subsection{Interaction with a ResNet-50 Predictor}

With a CLIP ResNet-50 predictor on ImageNet-A, $64$ views improve accuracy from $24.12$ to
$28.97$ ($+4.85$), whereas a single-view DINOv2-B cache reaches $34.99$ ($+10.87$); the full
configuration reaches $39.07$. The sweep uses one run; the replicated single-view benchmark
result is $35.52 \pm 0.48$.

\paragraph{Scope of the $\pm0.4$ band.} The dispersion quoted throughout is measured at
$M{=}64$, where the cache-only configuration on ImageNet-A spans $0.04$ points across three
seeds. It does not apply uniformly at $M{=}1$. On ImageNet-A, which has $7{,}500$ images, the
cache-only configuration spans $0.76$ points across seeds, and the configurations that add
pseudo-label refinement span $1.08$--$1.61$. Refinement is therefore itself a source of
run-to-run variance, consistent with its negative mean effect, and comparisons at $M{=}1$ on
this split should be read against the wider band. This does not affect the conclusion of
Sec.~\ref{sec:protocol}, because the cache-only configuration is the most accurate in
every seed.

\subsection{Oracle Coverage Analysis}

An oracle retaining every prediction that either source gets right would reach $71.01\%$,
versus $65.23\%$ for their actual combination. For the CLIP cache, the corresponding values
are $64.99\%$ and $58.21\%$. Because these gaps mix source interactions with regressions on
previously correct images, the main paper uses the directly observed $48\%$ repair-set overlap.

\section{Geometry and Capacity: Additional Results}
\label{sec:s_capacity}

\Cref{fig:s_geometry} visualizes the measurements from
Sec.~\ref{sec:geometry}; \cref{fig:s_capacity} gives
the associated capacity-response curves. The t-SNE panels use the same $800$ images from
eight classes, with perplexity $30$, PCA initialization, and seed $0$. Their quoted purity is
within those classes and is not comparable to the full $200$-class measurement.

\begin{figure}[H]
\centering
\includegraphics[width=\linewidth]{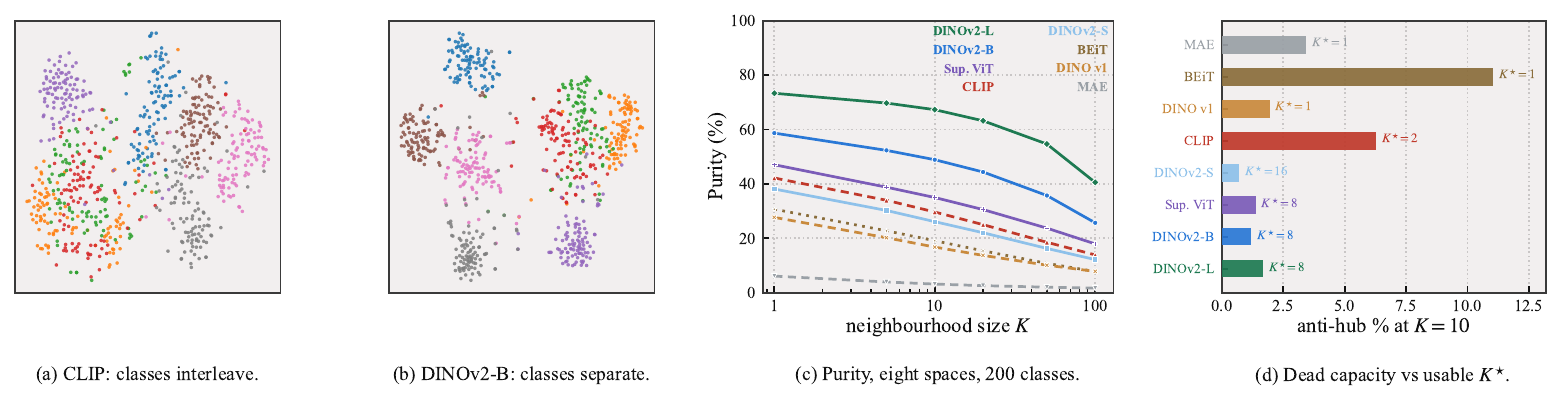}
\caption{Retrieval geometry on ImageNet-A before adaptation. The t-SNE projections
illustrate the cleaner DINOv2-B neighborhoods; the unprojected full-feature result is the
purity curve over all eight retrieval spaces, including BEiT, DINO~v1, and the supervised
ViT. CLIP has $6.3\%$ anti-hubs and
$\Kstar{=}2$, versus $0.7$--$1.7\%$ and $\Kstar{=}8$--$16$ for DINOv2; BEiT is
the extreme case, $11.1\%$ anti-hubs and $\Kstar{=}1$, while DINO~v1 ($2.0\%$) and the
supervised ViT ($1.4\%$) sit between CLIP and the DINOv2 family.}
\label{fig:s_geometry}
\end{figure}

Tab.~\ref{tab:swap} reports each optimum $\Kstar$; \cref{fig:s_capacity} shows the full
curves.

\begin{figure}[H]
\centering
\includegraphics[width=\linewidth]{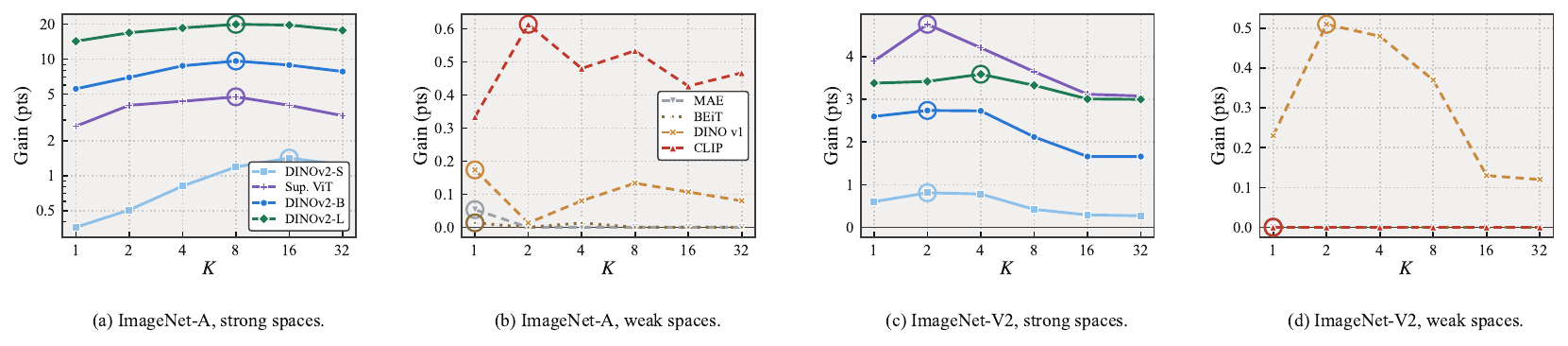}
\caption{Full capacity response: all eight retrieval spaces on both ImageNet streams.
Rings mark each space's optimum $\Kstar$, the summary statistic reported for each curve in
Tab.~\ref{tab:swap}. Curves are weight-swept envelopes, so a space that cannot be scaled usefully
sits at $+0.00$ rather than going negative; this isolates capacity behaviour from the scaling
failure that dominates a common-weight view. Each stream is split into its four
strong and four weak spaces so no curve hides at the bottom of another's scale; the strong
ImageNet-A panel uses a log ordinate because DINOv2-L spans twenty points. The optimum
shifts with the stream, not
only with the space: every DINOv2 space peaks lower on ImageNet-V2 than on ImageNet-A.}
\label{fig:s_capacity}
\end{figure}

Beyond the optimum, more memory can reduce accuracy. At $K{=}64$, the CLIP cache's best
weight is $0$, so the oracle ignores it.

\subsection{Estimators and conditional interpretation}
On a shared calibration set $S$ of $n$ images, let $N_\kappa^s(i)$ be the $\kappa$ nearest
neighbors of $i$ in space $s$, excluding $i$, and let
$d_j^s=|\{i:j\in N_\kappa^s(i)\}|$. We compute
\begin{equation}
P_s(\kappa)=\frac{1}{n\kappa}\sum_{i\in S}\sum_{j\in N_\kappa^s(i)}
\mathbf1[y_j=y_i],\qquad
A_s(\kappa)=\frac{1}{n}\sum_{j\in S}\mathbf1[d_j^s=0].
\label{eq:geom-estimators}
\end{equation}
The probe radius $\kappa$ is distinct from cache capacity $K$. Purity requires labels;
anti-hubness does not. The graph has exactly $n\kappa$ edges and
$n\kappa(1-P_s)$ cross-class edges, while $n(1-A_s)$ points receive any edge. Thus purity
measures contamination and anti-hubness measures how broadly retrieval mass is distributed.

For query $t$, let $\mu_t$ and $\delta_t^{s,K}$ be the base and cache margins from
\cref{eq:margins}. If an influential band contains $m_t$ entries, has true-label purity
$p_t^\circ$, pseudo-label error $\eta_t$, minimum within-band weight $\rho_t$, and tail mass
bounded by $N_t\varepsilon_t$, then
\begin{equation}
\delta_t^{s,K}\ge
\alpha\!\left\{m_t\big[(1+\rho_t)[p_t^\circ-\eta_t]_+-1\big]
-N_t\varepsilon_t\right\}.
\label{eq:margin-bound}
\end{equation}
\begin{corollary}[Sufficient local-purity threshold]\label{cor:threshold}
A sufficient condition for a positive cache margin is
\begin{equation}
p_t^\circ-\eta_t>
\frac{1}{1+\rho_t}+\frac{N_t\varepsilon_t}{m_t(1+\rho_t)}.
\label{eq:purity-threshold}
\end{equation}
\end{corollary}
These are conditional, query-level statements, not performance guarantees.

\begin{assumption}[Calibration representativeness]\label{asm:identify}
For each space $s$, the mean true-label purity $\bar p_s$ of influential online
neighborhoods satisfies $|\bar p_s-P_s(\kappa)|\le\xi_s$, where $\xi_s$ absorbs sampling,
online-prefix, per-class truncation, priority-selection, and probe-versus-kernel discrepancies.
\end{assumption}
This is the empirical link tested by rank correlation, held-out transfer, and the EuroSAT
failure case; it is not implied by the kernel.

Capacity is non-monotone. If increasing capacity exposes weights $a_c^{K+1}$ for class $c$,
and $c_K^\star$ is the strongest wrong class at capacity $K$, then
\begin{equation}
\delta_t^{s,K+1}-\delta_t^{s,K}
\le\alpha\big(a_y^{K+1}-a_{c_K^\star}^{K+1}\big).
\label{eq:cap-increment}
\end{equation}
The margin can fall when the newly exposed confuser is more similar than the new true-class
key. Anti-hubness therefore diagnoses potentially wasted capacity; it does not directly
estimate $\Kstar$ or select a retrieval space.

\section{Encoder-Induced Retrieval Similarities at Query Level}
\label{sec:s_drawn}

The aggregate measurements above are reflected in individual margins, retrieved keys, and
kernel weights. Statistics use the full embedding; projections are visual only, and matched
panels use the same retained indices.

\begin{figure}[H]
\centering
\includegraphics[width=\linewidth]{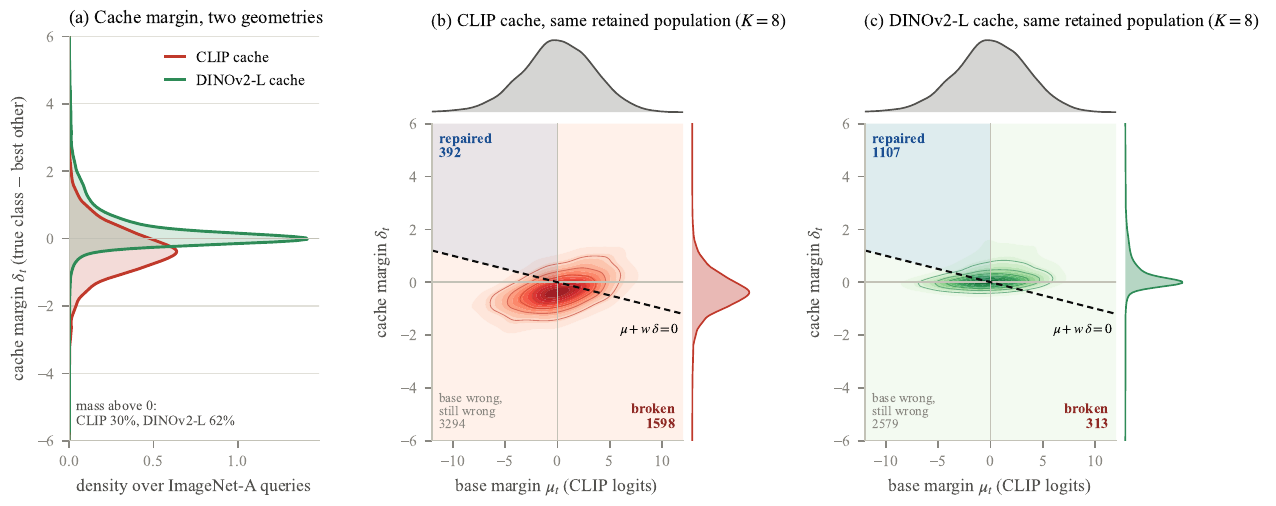}
\caption{Base and cache margins on ImageNet-A for matched CLIP and DINOv2-L caches at
$K{=}8$. Fusion at $w{=}10$ is correct above $\mu+w\delta=0$; DINOv2-L places many more
base errors in the correctable region.}
\label{fig:s_margin}
\end{figure}

\begin{figure}[H]
\centering
\includegraphics[width=\linewidth]{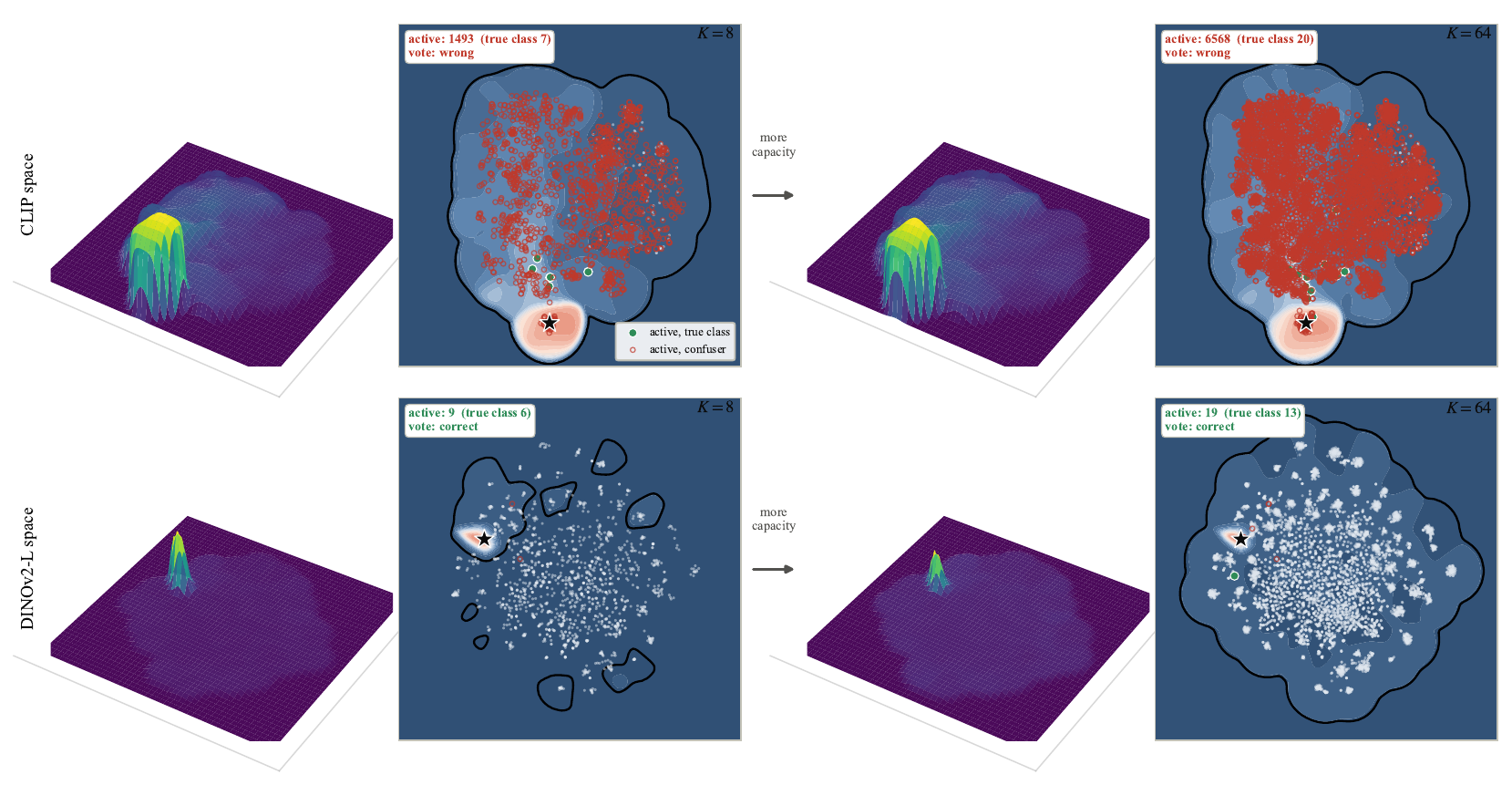}
\caption{Kernel weights for one ImageNet-A query that CLIP retrieval misses and DINOv2-L
retrieval corrects. CLIP activates $1{,}493$ of $1{,}586$ keys and is dominated by confusers;
DINOv2-L activates nine keys, six with the true label.}
\label{fig:s_kernel}
\end{figure}

\begin{figure}[H]
\centering
\includegraphics[width=\linewidth]{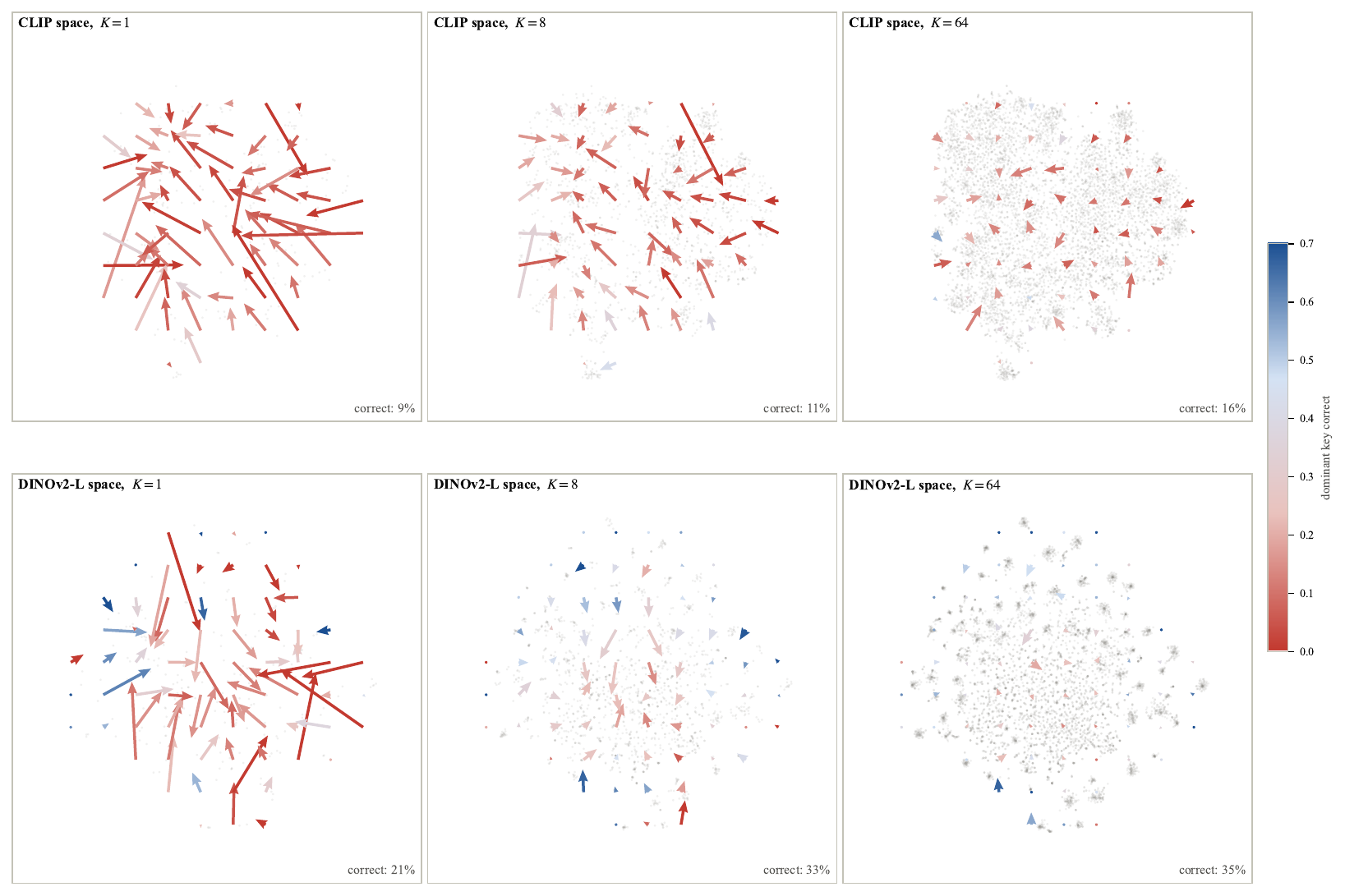}
\caption{Dominant keys for ImageNet-A base errors. The strongest key has the true label for
$11\%$ of queries in CLIP space and $33\%$ in DINOv2-L space; increasing $K$ makes retrieval
local in both spaces but does not repair CLIP's label disagreement.}
\label{fig:s_flow}
\end{figure}

\begin{figure}[H]
\centering
\includegraphics[width=\linewidth]{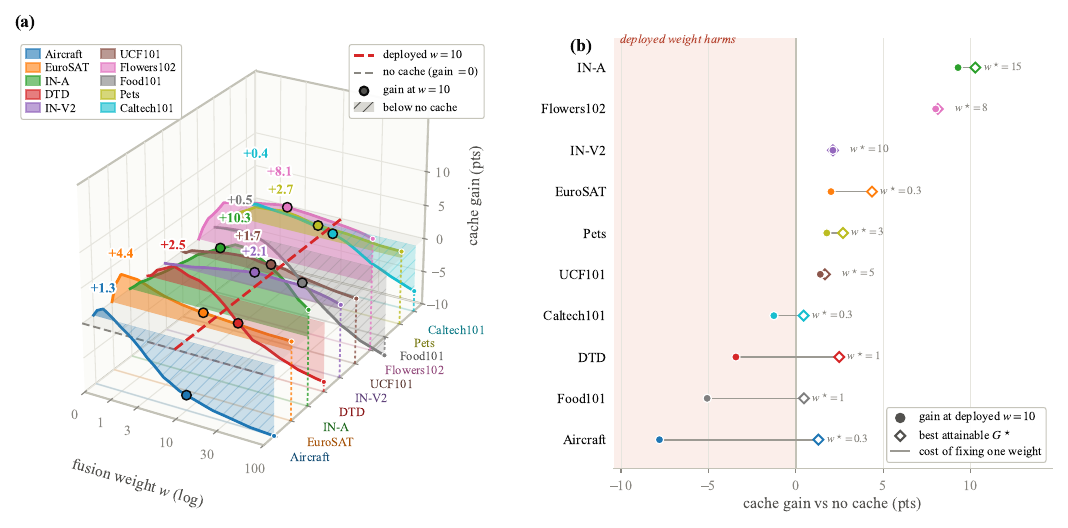}
\caption{Cache gain across fusion weights on ten streams. The selected $w{=}10$ lies near the
best region on the ImageNet shifts but is too large for Aircraft and Food101, motivating the
held-out transfer analysis below.}
\label{fig:hazard}
\end{figure}

Together, these views show the same mechanism at four scales: label-consistent encoder-induced
similarities create positive cache margins by concentrating weight on suitable neighbors, while
diffuse similarities spread weight across confusers and can amplify base errors.

\section{Limits of a Composite Geometry Diagnostic}
\label{sec:s_discarded}

This negative result bounds the rule in Sec.~\ref{sec:capacity}. On ImageNet-A, purity
reverses one pair:
DINOv2-S has lower purity than CLIP ($26.1$ vs.\ $29.7$) but a higher ceiling ($+1.33$ vs.\
$+0.62$). Its lower anti-hub fraction ($0.7\%$ vs.\ $6.3\%$) motivates a
hubness-discounted purity score, which perfectly orders the five spaces on ImageNet-A
($\rho{=}1.00$). On ImageNet-V2, however, it falls to $0.87$, below purity alone at $0.98$.
We therefore discard the composite as overfit to a single inversion.

The fitted family
\begin{equation}
\mathcal C_\lambda(s)=P_s(\kappa)-\lambda A_s(\kappa),\qquad \lambda\ge0,
\label{eq:composite}
\end{equation}
can change order only when a pair of affine scores crosses. Five spaces therefore permit at
most $\binom{5}{2}+1=11$ rankings. Selecting $\lambda$ to repair the one ImageNet-A
inversion and evaluating on the same points gives an in-sample $\rho{=}1.00$; without
refitting on ImageNet-V2 it falls to $0.87$, below purity alone at $0.98$. We discard the
composite rather than report this overfit improvement.

\begin{definition}[Resolution-aware separation]\label{def:separation}
If $s_{(1)}$ has the largest observed purity $\widehat P_s$ and $e_s$ are simultaneous
measurement-error radii, its lead is resolved when
\begin{equation}
\widehat P_{s_{(1)}}-\widehat P_s>e_{s_{(1)}}+e_s
\quad\text{for every }s\ne s_{(1)}.
\label{eq:sep}
\end{equation}
\end{definition}
We use $B{=}300$ calibration-subsample bootstrap replicates at each
$n\in\{250,500,1000,2000\}$, resampling purity against the fixed measured gains.
\Cref{tab:s_bootstrap_selection} reports central $95\%$ intervals for $\rho$ and the frequency
with which the selected encoder attains zero regret. ImageNet-A selection is exact at every
tested size; ImageNet-V2 reaches probability $0.99$ at $n{=}1000$ and $1.00$ at $n{=}2000$.

\begin{table}[H]
\centering\footnotesize
\setlength{\tabcolsep}{5pt}
\caption{Calibration-subsample uncertainty over sixteen retrieval spaces using $B{=}300$
replicates. Intervals are central $95\%$ intervals; $\Pr(\mathrm{Reg}{=}0)$ is the fraction of
replicates selecting an oracle-gain leader.}
\label{tab:s_bootstrap_selection}
\begin{tabular}{lccc}
\toprule
Stream & $n$ & $\rho$ [95\% interval] & $\Pr(\mathrm{Reg}{=}0)$ \\
\midrule
ImageNet-A & 250  & 0.953 [0.925, 0.974] & 1.00 \\
           & 500  & 0.959 [0.944, 0.971] & 1.00 \\
           & 1000 & 0.959 [0.950, 0.968] & 1.00 \\
           & 2000 & 0.957 [0.956, 0.962] & 1.00 \\
\midrule
ImageNet-V2 & 250  & 0.749 [0.420, 0.932] & 0.60 \\
            & 500  & 0.871 [0.753, 0.951] & 0.88 \\
            & 1000 & 0.917 [0.864, 0.952] & 0.99 \\
            & 2000 & 0.937 [0.906, 0.958] & 1.00 \\
\bottomrule
\end{tabular}
\end{table}

\paragraph{Evaluation-blind joint selection.}
We select $s$ using pseudo-purity on the unlabeled target batch, select $(K,w)$ on a disjoint
labeled calibration stream, and transfer those hyperparameters unchanged to the evaluation
stream. Evaluation-stream labels are used only to compute final accuracy.
\Cref{tab:s_joint_selection} compares this protocol with the oracle over the evaluated
configuration set for each stream.

\begin{table}[H]
\centering\footnotesize
\setlength{\tabcolsep}{4pt}
\caption{Evaluation-blind joint selection of $(s,K,w)$. ImageNet-A receives $(K,w)$ calibrated
on ImageNet-V2, and ImageNet-V2 receives $(K,w)$ calibrated on ImageNet-A. The oracle column is
the envelope over the evaluated configuration set (58 configurations for ImageNet-A and 56
for ImageNet-V2).}
\label{tab:s_joint_selection}
\adjustbox{max width=\linewidth}{%
\begin{tabular}{lclcccc}
\toprule
Evaluation stream & Pseudo-purity choice $s$ & Transferred $(K,w)$ & Accuracy & Gain & Oracle & Regret \\
\midrule
ImageNet-A  & DINOv2-L & $(4,20)$ from V2 & 68.39 & $+17.55$ & 70.45 & 2.07 \\
ImageNet-V2 & DeiT-III & $(8,15)$ from A & 68.62 & $+4.96$  & 68.72 & 0.10 \\
\bottomrule
\end{tabular}%
}
\end{table}

\subsection{EuroSAT Failure Case}
\label{sec:s_eurosat}

EuroSAT is the only stream with an unreliable space ranking. Candidate spaces have similar
purity, so small run-to-run changes alter their order; on the other five streams, the leading
spaces remain separated across seeds. This subsection concerns the five spaces of
Tab.~\ref{tab:swap}; over all eight, the best space on EuroSAT is DINO v1
(Sec.~\ref{sec:s_newspaces}), which does not change the point being made here.

In \cref{fig:s_seedspread}, DINOv2-S and DINOv2-L have disjoint three-seed ranges on every
stream except EuroSAT. There, the ranges overlap ($+4.91$--$+8.36$ versus
$+2.89$--$+11.40$): DINOv2-S leads for seeds $0$ and $2$, and DINOv2-L for seed $1$.
DINOv2-B, previously omitted for lacking a third seed, is now measured at all
three; it tracks DINOv2-S closely on most streams but spans $-3.43$ to $+1.54$ on DTD,
where its rank against DINOv2-S flips with the seed while DINOv2-L stays above it in every
seed. The ranking claim this section defends concerns the DINOv2-S/DINOv2-L pair and is
unchanged.

Converting these ranges to intervals sharpens the point. EuroSAT carries the two widest seed
intervals in the paper: $[47.63,70.06]$ for DINOv2-L at $M{=}4$, a span of $22.4$ points, and
a $38.6$-point span for DINOv2-B at $M{=}64$. No ranking claim can be supported on a column
that wide, which is why EuroSAT is presented as a failure case rather than as a measurement.
The cache \emph{gain} on EuroSAT is a separate quantity and is well resolved
($+7.32$ $[+6.21,+8.43]$ at $M{=}4$); it is the ordering among candidate spaces, not the
benefit of caching, that EuroSAT cannot settle.

\begin{figure}[H]
\centering
\includegraphics[width=\linewidth]{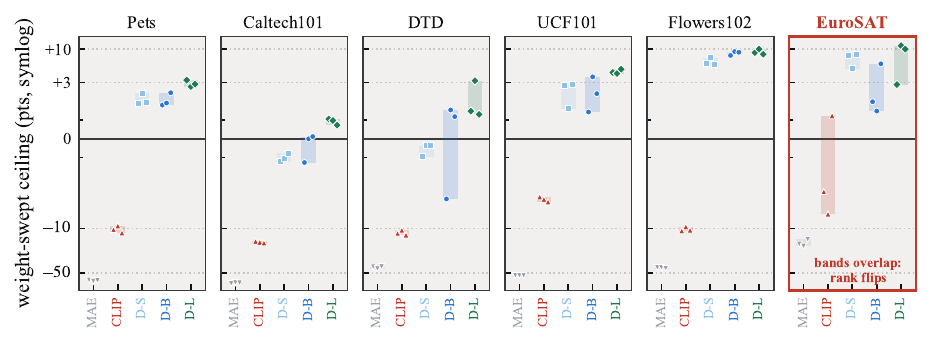}
\caption{Seed spread of cross-domain space rankings at $M{=}4$, for the five spaces of
Tab.~\ref{tab:swap}; markers are seeds and bands
show their range. All five spaces now have three seeds, the missing DINOv2-B
runs having since completed. The DINOv2-S/DINOv2-L ordering is seed-independent on five
streams and flips on EuroSAT, where their bands overlap; the new DINOv2-B measurements add
one caveat, a wide band on DTD ($-3.43$ to $+1.54$) that overlaps both neighbours'
bands.}
\label{fig:s_seedspread}
\end{figure}

The separation condition distinguishes stable from unstable rankings. EuroSAT's purity range
is $14.4$ points, versus $47.3$--$85.6$ elsewhere; four spaces lie within $4.4$ points, and
even MAE reaches $74.6\%$ purity.

We did not estimate the error radii $e_s$, so we cannot formally certify separation on
any stream. EuroSAT's narrow purity range and overlapping gain ranges show that its leading
candidates are difficult to distinguish, but do not tell us whether the cause is sampling
variation or a failure of the diagnostic. EuroSAT is also a counterexample to a simple scale account:
DINOv2-L has the highest purity but the third-best ceiling, $5.0$ points below DINOv2-S,
despite ranking first on every other stream.

\subsection{Limits of Cross-Stream Prediction}

Cache gain varies by an order of magnitude across streams, but DINOv2-B purity at $K{=}5$
does not predict it ($\rho{=}+0.03$). Purity saturates at $0.87$--$1.00$ on five streams,
versus $0.52$--$0.56$ on ImageNet-A and ImageNet-V2. Anti-hubness and skewness correlate more
strongly ($\rho{=}+0.71$ and $+0.66$), but $n{=}6$ supports no conclusion. Thus,
Sec.~\ref{sec:capacity} ranks
candidate \emph{spaces} within a fixed stream; it does not predict which streams benefit.
Predictor headroom appears relevant, but these experiments do not establish its causal role.

\section{Held-Out Fusion-Weight Transfer}
\label{sec:s_transfer}

Eq.~\ref{eq:envelope} selects the fusion weight on the evaluation stream and is therefore an oracle. We
quantify its advantage by selecting the weight on one stream and applying it unchanged to
another, re-scoring stored logits over the same $\mathcal{W}$ without new inference.

Because stored labels and priorities are independent of $w$, one online pass produces base
logits $b_t$ and cache logits $L_t$ that can be re-scored at every tested weight.
\begin{proposition}[Offline re-scoring]\label{prop:traj}
For fixed stream, views, retrieval space, capacity, and random draw, the memory trajectory is
identical for all $w$. Hence $b_t+wL_t$ can be evaluated for every $w\in\mathcal W$ without
rerunning either encoder or changing the cache.
\end{proposition}
\begin{proof}
The update stores the base pseudo-label and CLIP-derived priority, neither of which depends on
the fused prediction. Induction over the stream therefore gives the same retained indices and
keys for every $w$.
\end{proof}

For a common no-cache accuracy $A_0$, distinguish fixed-weight gain
$G_w=\mathrm{Acc}(w)-A_0$, the evaluation-stream oracle
$G^\star=\max_{w\in\mathcal W}\mathrm{Acc}(w)-A_0$, and held-out gain
$G^{\mathrm{tr}}=\mathrm{Acc}(\widehat w_{\mathrm{sel}})-A_0$. Oracle gain is non-negative
only because $0\in\mathcal W$; $+0.00$ means no tested weight helps. Held-out selection uses no
evaluation labels. Its cost $G^\star-G^{\mathrm{tr}}$ is non-negative when both entries use
the same stream, capacity, grid, and seeds; unmatched oracle comparisons do not have that
interpretation.

For \cref{tab:s_transfer}, we select $w$ on ImageNet-V2 at a common $K{=}8$ and apply it to
ImageNet-A, averaging three stored-logit seeds. This is more conservative than using each
space's $\Kstar$. MAE and CLIP are harmful on V2 at $K{=}8$, so transfer selects $w{=}0$.

\begin{table}[H]
\centering\footnotesize
\setlength{\tabcolsep}{4pt}
\caption{Oracle and held-out fusion weights on ImageNet-A at $M{=}1$, averaged over two
seeds. Transfer selects $w$ on ImageNet-V2 and preserves both the five-space ordering and its
$0.90$ rank correlation with $\purity@10$. Columns report $G^\star$ and $G^{\mathrm{tr}}$,
respectively.}
\label{tab:s_transfer}
\begin{tabular}{lrrr}
\toprule
Cache space & $G^\star$ oracle\,$\uparrow$ & $G^{\mathrm{tr}}$ held out\,$\uparrow$ & cost \\
\midrule
MAE      & \gain{$+0.09$}  & $+0.00$         & \loss{$0.09$} \\
CLIP     & \gain{$+0.56$}  & $+0.00$         & \loss{$0.56$} \\
DINOv2-S & \gain{$+1.33$}  & \gain{$+0.45$}  & \loss{$0.88$} \\
DINOv2-B & \gain{$+9.88$}  & \gain{$+9.73$}  & \loss{$0.16$} \\
DINOv2-L & \best{\gain{$+20.06$}} & \best{\gain{$+19.16$}} & \loss{$0.90$} \\
Span     & \gain{$19.97$} & \gain{$19.16$} & \loss{$0.81$} \\

\emph{Spearman with $\purity@10$} & $0.90$ & $0.90$ & $0.00$ \\
\bottomrule
\end{tabular}
\end{table}

For cross-domain transfer, each space selects $w$ on the mean curve of five streams and
applies it to the sixth. Across $30$ cells, transfer costs $0.37$ points on average. Mean
purity--gain correlation falls from $0.87$ to $0.73$ but remains positive on every stream.
The largest loss is $-2.44$ for CLIP on EuroSAT, the only cell where a CLIP cache contributes
over one point and its weight does not transfer.

Tab.~\ref{tab:swap} therefore reports upper bounds, but ImageNet-A's ordering does not
depend on tuning
$w$ on the evaluation stream: held-out selection loses under one point of a roughly
$20$-point retrieval-space span.

\section{Component Ablations Across OOD Splits}
\label{sec:s_components}

\Cref{tab:s_components_ood} gives the cumulative ablation from
Sec.~\ref{sec:protocol} by OOD split.

\begin{table}[H]
\centering\footnotesize
\setlength{\tabcolsep}{4pt}
\caption{Incremental accuracy change from each component at $M{=}1$. Refinement consistently
hurts, the negative memory is neutral, and adaptive fusion helps only on ImageNet-Sketch.}
\label{tab:s_components_ood}
\begin{tabular}{lccccr}
\toprule
Component & A & V2 & R & Sketch & Mean \\
\midrule
Cross-modal refinement      & \loss{$-3.27$} & \loss{$-0.58$} & \loss{$-0.10$} & \loss{$-3.81$} & \loss{$-1.94$} \\
Entropy-band neg.\ memory   & \gain{$+0.01$} & $+0.00$ & $+0.00$ & $+0.00$ & $+0.00$ \\
Uncertainty-adaptive fusion & \loss{$-0.32$} & \loss{$-0.22$} & \loss{$-1.22$} & \gain{$+1.41$} & \loss{$-0.09$} \\
\bottomrule
\end{tabular}
\end{table}

\paragraph{Entropy and admission-priority controls.}
\Cref{tab:s_protocol_controls} reports paired three-seed comparisons on ImageNet-A. Replacing
the inherited entropy score with entropy of the view posterior preserves accuracy to within
$0.04$ points. Retrieval-aware priority, computed from fused pre-update evidence, gives closely
matched accuracy for both auxiliary encoders. These controls support the shared admission
trajectory used throughout the encoder intervention; the admission rows retain CLIP-derived
stored pseudo-labels.

\begin{table}[H]
\centering\footnotesize
\setlength{\tabcolsep}{4pt}
\caption{Paired protocol controls on ImageNet-A at $M{=}1$ and $w{=}10$, averaged over three
seeds. The entropy control uses $K{=}8$; the admission-priority controls use $K{=}16$.
Subscripts are standard deviations across seeds, and $\Delta$ is paired variant minus reference.}
\label{tab:s_protocol_controls}
\begin{tabular}{llccc}
\toprule
Control & Retrieval encoder & Reference & Variant & paired $\Delta$ \\
\midrule
Entropy: inherited $\rightarrow$ standard & DINOv2-B & 60.14\ci{0.38} & 60.10\ci{0.41} & $-0.04$ \\
Priority: CLIP $\rightarrow$ retrieval-aware & DINOv2-B & 60.14\ci{0.38} & 60.24\ci{0.14} & $+0.11$\ci{0.26} \\
Priority: CLIP $\rightarrow$ retrieval-aware & DINOv2-L & 69.93\ci{0.23} & 70.01\ci{0.28} & $+0.08$\ci{0.22} \\
\bottomrule
\end{tabular}
\end{table}

\Cref{tab:s_admission_policies} broadens the population intervention to confidence-only CLIP
admission, random per-pseudo-class reservoir sampling, and chronological FIFO caching. Across
all five policies, the DINOv2-L--DINOv2-B accuracy gap stays near $9.6$ points. Random and FIFO
retain at least $8.73$ points of gain with DINOv2-B and $18.33$ with DINOv2-L, establishing the
encoder ordering across retained populations.

\begin{table}[H]
\centering\footnotesize
\setlength{\tabcolsep}{5pt}
\caption{Admission-policy comparison on ImageNet-A at $M{=}1$, $K{=}16$, and $w{=}10$.
All policies use CLIP pseudo-labels; values are mean $\pm$ standard deviation over three seeds.
The cache-free reference is $50.84\%$.}
\label{tab:s_admission_policies}
\begin{tabular}{lcc}
\toprule
Admission policy & DINOv2-B & DINOv2-L \\
\midrule
Entropy priority & $60.14\pm0.38$ & $69.93\pm0.23$ \\
Confidence-only CLIP $p_1$ & $60.28\pm0.35$ & $70.14\pm0.17$ \\
Random reservoir & $59.60\pm0.28$ & $69.27\pm0.16$ \\
Chronological FIFO & $59.57\pm0.15$ & $69.17\pm0.37$ \\
Retrieval-aware & $60.24\pm0.14$ & $70.01\pm0.28$ \\
\bottomrule
\end{tabular}
\end{table}

Refinement is most harmful where CLIP is weakest: $-3.27$ on ImageNet-A and $-3.81$ on
ImageNet-Sketch, versus $-0.10$ on ImageNet-R. Refinement and negative memory both use the
weak CLIP neighborhoods identified in Tab.~\ref{tab:swap}; adaptive fusion does not, so
retrieval space
alone does not determine its effect.

The optional mechanisms have three distinct pathways. With CLIP-memory evidence $r_t^C$,
refinement replaces only the stored retrieval label by
$\hat y_t^\star=\arg\max_c[b_t+r_t^C]_c$. The entropy-band negative memory instead emits
\[
\widetilde b_t=b_t+wL^{\mathcal M_{t-1}}(q_t)-\lambda_-R_t,
\]
and adaptive fusion replaces $w$ by $w_t=w\bar h_t^\gamma$, where $\bar h_t$ is normalized
view entropy. The latter two change the emitted logit but not the positive-cache trajectory.
The reduced configuration disables all three output paths while retaining the $K_C{=}16$ CLIP
memory for admission priority; the strict one-memory control sets $K_C{=}0$.

\section{Full Rank-Correlation Results}
\label{sec:s_figures}

\Cref{tab:s_correlations} gives the rank correlations quoted in
Sec.~\ref{sec:capacity}.

Spearman $\rho$ is the Pearson correlation of the two midrank vectors; tied gains receive
midranks. We use it descriptively because the encoders form a designed candidate set rather
than a random sample, and report no significance tests. With untied rankings,
$\rho=1-6\sum_i d_i^2/[n(n^2-1)]$: attainable values are spaced by $0.1$ at $n{=}5$ and
$1/42\approx0.024$ at $n{=}8$. We therefore interpret the ordering, held-out replication,
and selection regret rather than small numerical differences. Midranks matter wherever gains
tie at zero: on ImageNet-V2 five of the sixteen spaces do, and breaking those ties by index
order instead of averaging them would raise $\rho(-\antihub,G^\star)$ from $0.756$ to $0.771$.
Every correlation reported in this paper uses the midrank definition stated above.

\begin{table}[H]
\centering\footnotesize
\setlength{\tabcolsep}{4pt}
\caption{Spearman $\rho$ between geometry and outcome across five spaces. ImageNet-V2 is
held out with no refitting. \colorbox{bestgreen}{Green} marks the stronger association per
outcome. Over eight spaces, $\purity@10$ correlates with gain at $0.93$ on ImageNet-A and
$0.98$ on ImageNet-V2 (Sec.~\ref{sec:s_newspaces}). Over sixteen spaces, the correlations
are $0.959$ and $0.943$ (\cref{tab:spaces16}); retaining positive-gain spaces gives $0.971$
and $0.973$, respectively.}
\label{tab:s_correlations}
\adjustbox{max width=\linewidth}{%
\begin{tabular}{llcc}
\toprule
Dataset & Predictor & vs.\ ceiling\,$\uparrow$ & vs.\ $\Kstar$\,$\uparrow$ \\
\midrule
\blk{4}{ImageNet-A (predictors selected here)} \\
& $\purity@10$   & \best{$0.90$ $[0.900,0.900]$} & $0.36$ $[0.359,0.359]$ \\
& $-\antihub$    & $0.50$ $[0.500,0.718]$ & \best{$0.87$ $[0.718,0.872]$} \\
\midrule
\blk{4}{ImageNet-V2 (held out, nothing refitted)} \\
& $\purity@10$   & \best{$0.98$ $[0.975,0.975]$} & $0.95$ $[0.949,0.949]$ \\
& $-\antihub$    & $0.95$ $[0.564,0.872]$ & \best{$0.97$ $[0.632,0.949]$} \\
\bottomrule
\end{tabular}%
}
\end{table}

Bracketed values are central $95\%$ calibration-subsample intervals at the largest tested
size, $n{=}2000$, from $B{=}300$ replicates that resample calibration images and recompute the
geometry statistic while holding the measured ceilings and $\Kstar$ fixed. The encoder set is
held fixed throughout: these five spaces are a designed candidate set rather than a draw from
a population of encoders, so an interval taken over encoders would have no sampling model
behind it, which is why the correlations themselves remain descriptive. What the intervals do
establish is that both arms of the dissociation are resolved on ImageNet-A: purity's lead on
the ceiling and anti-hubness's lead on $\Kstar$ are each separated, and the weak entries are
weak precisely rather than uncertainly. Note that $\purity$ is close to size-invariant while
$\antihub$ counts zero-in-degree points in an $n$-point graph and is therefore not, so the
anti-hubness intervals describe the subsample regime and need not bracket the full-split point
estimate printed beside them; the two ImageNet-V2 anti-hubness rows are the clearest instance.

\begin{table}[H]
\centering\footnotesize
\setlength{\tabcolsep}{6pt}
\caption{Three-seed replication of the sixteen-space ImageNet-A oracle-envelope ranking.
Each seed uses a coupled stream across all encoders. DINOv2-L leads every seed, and offline
purity selects that leader with zero regret. Streams use an explicitly seeded permutation; \cref{tab:swap} reports an independent
three-seed sweep of the same protocol drawn from the global RNG.
}
\label{tab:s_spaces16_seeds}
\begin{tabular}{lccc}
\toprule
Seed & $\rho(\purity@10,G^\star)$ & DINOv2-L $G^\star$ & selection regret \\
\midrule
0 & 0.959 & $+19.36$ & 0.00 \\
1 & 0.961 & $+19.71$ & 0.00 \\
2 & 0.967 & $+20.16$ & 0.00 \\
\midrule
Mean & 0.962 & $+19.74\pm0.40$ & 0.00 \\
\bottomrule
\end{tabular}
\end{table}

\begin{table}[H]
\centering\footnotesize
\setlength{\tabcolsep}{4pt}
\caption{Correlation checks for the sixteen-space intervention. Tie-excluded rows retain spaces
with positive oracle gain. The ResNet-50 row repeats the complete intervention with a second
frozen predictor and pseudo-label source; the near-matched row uses the six ImageNet-A spaces
whose gains lie between $4.4$ and $10.0$ points.}
\label{tab:s_rho_checks}
\adjustbox{max width=\linewidth}{%
\begin{tabular}{llcclc}
\toprule
Predictor & Stream/candidate set & $n$ & $\rho(\purity@10,G^\star)$ & Selected encoder & Regret \\
\midrule
CLIP ViT-B/16 & ImageNet-A, all & 16 & 0.959 & DINOv2-L & 0.00 \\
CLIP ViT-B/16 & ImageNet-A, positive gain & 15 & 0.971 & DINOv2-L & 0.00 \\
CLIP ViT-B/16 & ImageNet-V2, all & 16 & 0.943 & DeiT-III & 0.00 \\
CLIP ViT-B/16 & ImageNet-V2, positive gain & 11 & 0.973 & DeiT-III & 0.00 \\
CLIP ViT-B/16 & ImageNet-A, near-matched & 6 & 0.943 & DINOv2-B & 0.00 \\
CLIP ResNet-50 & ImageNet-A, all & 16 & 0.968 & DINOv2-L & 0.00 \\
\bottomrule
\end{tabular}%
}
\end{table}

With the ResNet-50 predictor, the cache-free accuracy is $24.1\%$ and the encoder-envelope
gain spans $0.00$--$19.04$ points. The same training-regime ordering appears under both
predictors, complementing the coupled-seed replication in \cref{tab:s_spaces16_seeds}.

\section{Additional Benchmark Baselines}
\label{sec:s_benchmark}

\Cref{tab:s_extra} collects the parameter-updating OOD references in compact form and reports
the larger auxiliary scale.

\begin{table}[H]
\centering\footnotesize
\setlength{\tabcolsep}{5pt}
\renewcommand{\arraystretch}{1.08}
\caption{Top-1 accuracy (\%) on the cross-domain benchmark with CLIP ViT-B/16 under the
protocol of~\citep{karmanov2024tda}. \colorbox{bestgreen}{Green} marks the best comparable
training-free result per dataset; shaded rows are \method{}. Every \method{} row is
training-free and uses a separate frozen DINOv2-B retrieval encoder. The lower block removes the
cache from our pipeline, so $\Delta$ isolates its contribution. Subscripts on \method{} rows
are standard deviations over three seeds; averages and deltas use unrounded values.
Aircraft and Food101 are evaluated separately in \cref{tab:cd_extra}. The average here
uses the six streams shown; \cref{tab:cd_extra} also reports the average over all eight.}
\label{tab:crossdomain}
\adjustbox{max width=\linewidth}{%
\begin{tabular}{lccccccc}
\toprule
Method & Caltech101\,$\uparrow$ & DTD\,$\uparrow$ & EuroSAT\,$\uparrow$ & Flowers102\,$\uparrow$ & Pets\,$\uparrow$ & UCF101\,$\uparrow$ & Average\,$\uparrow$ \\
\midrule
CLIP-ViT-B/16              & 93.55 & 45.04 & 50.42 & 66.99 & 86.92 & 65.16 & 68.01 \\
\midrule
\sub{8}{Methods with parameter adaptation} \\
CoOp~\citep{zhou2022coop}       & 93.70 & 41.92 & 46.39 & 68.71 & 89.14 & 66.55 & 67.74 \\
CoCoOp~\citep{zhou2022cocoop}   & 93.79 & 45.45 & 39.23 & 70.85 & 90.46 & 68.44 & 68.04 \\
TPT~\citep{shu2022tpt}          & 94.16 & 47.75 & 42.44 & 68.98 & 87.79 & 68.04 & 68.19 \\
DiffTPT~\citep{feng2023difftpt} & 92.49 & 47.00 & 43.13 & 70.10 & 88.22 & 62.67 & 67.27 \\
\midrule
\sub{8}{Training-free methods: CLIP visual encoder} \\
TDA~\citep{karmanov2024tda} & 94.24 & 47.40 & 58.00 & 71.42 & 88.63 & 70.66 & 71.73 \\
BoostAdapter~\citep{zhang2024boostadapter} & \best{94.77} & 45.69 & \best{61.22} & 71.66 & 89.51 & 71.93 & 72.46 \\
TaTa~\citep{zhang2026tata} & 93.82 & 49.57 & 59.67 & 71.28 & 89.91 & \best{73.54} & 72.97 \\
ETTA~\citep{etta2025} & 94.74 & 49.64 & 59.86 & 73.86 & 90.08 & 71.95 & 73.36 \\
\midrule
\sub{8}{Training-free methods: CLIP + frozen DINOv2-B} \\
\ours \textbf{\method{}(Ours)} (DINOv2-B), $M{=}4$ & 94.73\ci{0.18} & 53.76\ci{0.54} & 55.60\ci{2.93} & \best{81.27\ci{0.12}} & 92.03\ci{0.32} & 72.36\ci{1.01} & 74.96\ci{0.53} \\
\ours \textbf{\method{}(Ours)} (DINOv2-B), $M{=}64$ & 94.69\ci{0.08} & \best{54.37\ci{0.87}} & 58.67\ci{1.52} & 80.67\ci{0.33} & \best{92.39\ci{0.19}} & 71.94\ci{0.33} & \best{75.45\ci{0.14}} \\
\midrule
\sub{8}{Ensemble-only controls: same pipeline without the cache} \\
Ensemble only, $M{=}4$ & 94.65 & 52.36 & 50.37 & 72.80 & 89.89 & 69.73 & 71.63 \\
Ensemble only, $M{=}64$ & 95.17 & 53.66 & 47.47 & 72.15 & 89.70 & 70.39 & 71.42 \\
\emph{$\Delta$ contributed by the cache, $M{=}64$} & \loss{$-0.48$} & \gain{$+0.71$} & \gain{$+11.20$} & \gain{$+8.52$} & \gain{$+2.69$} & \gain{$+1.55$} & \gain{$+4.03$} \\
\bottomrule
\end{tabular}%
}
\end{table}

The cross-domain table likewise includes CoOp, CoCoOp, TPT, and DiffTPT as parameter-adapting
literature references and separates them from the training-free comparison. On this benchmark,
\method{} exceeds the best published results by
$4.7$, $6.8$, and $2.3$ points on DTD, Flowers102, and Pets, and trails on
Caltech101, EuroSAT, and UCF101. Reducing $M$ from $64$ to $4$ costs
$0.50$ average points, driven primarily by EuroSAT, which alone loses
$3.07$ points while three of the other five datasets improve;
the ensemble-only baseline instead changes from $71.42$ to $71.63$.
TPT, DiffTPT, and ZERO+Ens average $68.19$, $67.27$, and $67.77$, respectively.

\paragraph{The transferred weight hurts on two additional streams.}
We extend the six-dataset benchmark with Aircraft and Food101, using three seeds for each
configuration. Cars is no longer distributed, and we did not evaluate SUN397. At $w{=}10$,
the cache reduces accuracy on both added datasets. Including them lowers the mean cache
contribution at $M{=}4$ from $+3.32$ to $+0.84$ points. The cache itself retains a small
positive gain when the weight is selected per stream: $+1.39$ on Aircraft and $+0.50$ on
Food101.

\begin{table}[H]
\centering\footnotesize
\setlength{\tabcolsep}{5pt}
\renewcommand{\arraystretch}{1.08}
\caption{Results on Aircraft and Food101 at $M{=}4$ and $w{=}10$, averaged over three seeds.
Subscripts are standard deviations. Comparable published baselines are unavailable under this
protocol, so we compare with the same pipeline without the cache. The lower rows show how the
two datasets change the cross-domain average; $\Delta=G_{10}$. Aircraft has too many incorrect pseudo-labels
for reliable retrieval, while Food101 leaves few base errors to correct (\cref{cor:threshold}). The two arms are separate runs, so $\Delta$ here is an unpaired difference of means and
differs slightly from the paired per-image losses in App.~\ref{sec:s_ci}.
Table \ref{tab:spaces16} uses \(M=1\), so its values should not be compared directly with the Aircraft and Food101 results from this run.
}
\label{tab:cd_extra}
\begin{tabular}{lccr}
\toprule
Stream & \method{} (DINOv2-B)\,$\uparrow$ & Ensemble only, no cache\,$\uparrow$ & $G_{10}$ \\
\midrule
\blk{4}{Streams added to the suite} \\
Aircraft & $19.25$\ci{0.93} & $27.71$\ci{0.62} & \loss{$-8.46$} \\
Food101 & $78.98$\ci{0.14} & $83.77$\ci{0.05} & \loss{$-4.79$} \\
\midrule
\blk{4}{Effect on the cross-domain average, $M{=}4$} \\
\emph{Six streams (\cref{tab:crossdomain})} & $74.96$ & $71.63$ & \gain{$+3.32$} \\
\emph{Eight streams} & $68.50$ & $67.66$ & \gain{$+0.84$} \\
\bottomrule
\end{tabular}
\end{table}

\begin{table}[H]
\centering\scriptsize
\setlength{\tabcolsep}{3pt}
\caption{Compact OOD reference for parameter-updating methods and the larger auxiliary scale.
The method groups keep training protocol and auxiliary scale explicit; $\dagger$ marks our runs.}
\label{tab:s_extra}
\adjustbox{max width=\linewidth}{%
\begin{tabular}{lcccccc}
\toprule
Method & IN & IN-A & IN-V2 & IN-R & IN-Sk & OOD Avg \\
\midrule
\blk{7}{ResNet-50: methods that update parameters} \\
CoOp~\citep{zhou2022coop}       & 63.33 & 23.06 & 55.40 & 56.60 & 34.67 & 42.43 \\
CoCoOp~\citep{zhou2022cocoop}   & 62.81 & 23.32 & 55.72 & 57.74 & 34.48 & 42.82 \\
TPT~\citep{shu2022tpt}          & 60.74 & 26.67 & 54.70 & 59.11 & 35.09 & 43.89 \\
DiffTPT~\citep{feng2023difftpt} & 60.80 & 31.06 & 55.80 & 58.80 & 37.10 & 45.69 \\
\midrule
\blk{7}{ViT-B/16: methods that update parameters} \\
TPT~\citep{shu2022tpt}          & 68.98 & 54.77 & 63.45 & 77.06 & 47.94 & 60.81 \\
DiffTPT~\citep{feng2023difftpt} & 70.30 & 55.68 & 65.10 & 75.00 & 46.80 & 60.65 \\
\midrule
\blk{7}{ViT-B/16: larger auxiliary scale} \\
COSMIC~\citep{huang2025cosmic} (DINOv2-L-reg) & 78.19 & 73.32 & 69.62 & 85.60 & 62.79 & 72.83 \\
\ours \textbf{\method{}} (DINOv2-L), $M{=}64$$^\dagger$ & 77.03 & 72.04 & 66.95 & 84.62 & 59.83 & 70.86 \\
\bottomrule
\end{tabular}%
}
\end{table}

At DINOv2-L scale, \method{} reaches a $70.86$ OOD average versus COSMIC's $72.83$. The
configurations differ in unisolated factors, so we do not attribute the gap. Within
\method{}, changing only DINOv2-B to DINOv2-L adds $2.95$ points
($67.91 \rightarrow 70.86$).

\section{Cross-Domain Retrieval-Space Rankings}
\label{sec:s_scope}

\Cref{tab:s_scope} gives the per-stream ceilings underlying Sec.~\ref{sec:capacity}.

\begin{table}[H]
\centering\footnotesize
\setlength{\tabcolsep}{4pt}
\caption{Selecting among sixteen retrieval spaces on three cross-domain streams at $M{=}4$.
\emph{Range} is the purity spread, and \emph{ties} counts spaces with zero gain. Purity selects
the best space on Flowers102 and DTD but not on EuroSAT, where its values are tightly grouped.
All gains used for ranking and regret are oracle envelopes $G^\star$.}
\label{tab:cd16}
\adjustbox{max width=\linewidth}{%
\begin{tabular}{lccrllr}
\toprule
Stream & range & ties & $\rho(\text{purity},G^\star)$\,$\uparrow$ & best space & purity selects & regret\,$\downarrow$ \\
\midrule
\blk{7}{Purity is separated: the rule applies} \\
Flowers102 & 85.6 & 2 & $+0.922$ & DINOv2-L & DINOv2-L & \best{$0.00$} \\
DTD & 50.5 & 4 & $+0.736$ & DINOv3-B & DINOv3-B & \best{$0.00$} \\
\midrule
\blk{7}{Precondition violated: purity has no spread to rank with} \\
EuroSAT & 16.7 & 1 & $+0.400$ & DINOv3-B & EVA-02 CLIP & \loss{$1.98$} \\
\bottomrule
\end{tabular}%
}
\end{table}

\begin{table}[H]
\centering\footnotesize
\setlength{\tabcolsep}{3pt}
\caption{Oracle-envelope cache gains $G^\star$ for five spaces on six cross-domain streams at $M{=}4$.
$\rho$ correlates $\purity@10$ with $G^\star$; \emph{range} is the purity spread.
\colorbox{bestgreen}{Green} marks the best of these five spaces. The three spaces
added in Sec.~\ref{sec:s_newspaces} are not listed here; they enter the eight-space selection
analysis of Tab.~\ref{tab:s_selection}, where DINO v1 takes the lead on EuroSAT.}
\label{tab:s_scope}
\adjustbox{max width=\linewidth}{%
\begin{tabular}{lrrrrrcr}
\toprule
Stream & MAE & CLIP & D-S & D-B & D-L & $\rho$\,$\uparrow$ & range \\
\midrule
\blk{8}{Purity is discriminative: the rule applies} \\
Pets       & $+0.03$ & $+0.11$ & $+1.94$ & $+2.45$ & \best{$+3.60$} & $1.00$ & 82.2 \\
Caltech101 & $+0.00$ & $+0.00$ & $+0.28$ & $+0.89$ & \best{$+1.46$} & $1.00$ & 63.1 \\
DTD        & $+0.00$ & $+0.00$ & $+2.19$ & $+3.25$ & \best{$+4.14$} & $0.90$ & 47.3 \\
UCF101     & $+0.05$ & $+1.24$ & $+2.80$ & $+2.62$ & \best{$+4.55$} & $0.90$ & 50.0 \\
Flowers102 & $+0.08$ & $+0.00$ & $+6.33$ & $+8.61$ & \best{$+9.26$} & $0.80$ & 85.6 \\
\emph{Mean} & & & & & & \gain{$0.92$} & 65.6 \\
\midrule
\blk{8}{Precondition violated: purity has no spread to rank with} \\
EuroSAT    & $+0.00$ & $+0.07$ & \best{$+8.63$} & $+7.54$ & $+3.60$ & \loss{$0.10$} & \textbf{14.4} \\
\bottomrule
\end{tabular}%
}
\end{table}

EuroSAT is the only stream with a narrow purity range; the next-narrowest is over three times
wider. Separation therefore tests ranking stability, not correctness. EuroSAT also remains a
counterexample: DINOv2-L has the highest purity but only the third-best gain among these
five, and the sixth-best of all eight.

The two reported EuroSAT correlations use different candidate sets: $\rho{=}0.400$ in
\cref{tab:cd16} ranks sixteen spaces, whereas $\rho{=}0.10$ in \cref{tab:s_scope} ranks the
original five.

\Cref{tab:s_scope8} extends the sweep to the three spaces of
Sec.~\ref{sec:s_newspaces}. Unlike on the ImageNet streams, adding candidates lowers the
cross-domain rank correlation: on the matched single-seed basis of this table, the mean
$\rho$ over the five well-separated streams falls from $0.91$ to $0.82$, though it remains at
least $0.78$ on each.
The added spaces introduce many ties in cache gain: BEiT~\citep{bao2022beit} contributes exactly $+0.00$ on all
five well-separated streams, and DINO v1 does so on two of them. No predictor can order tied
outcomes; as a diagnostic, restricting each stream to the spaces that deliver a
non-negligible gain returns the mean to $0.89$. That restriction conditions on the outcome
and is therefore not a correction we adopt, only evidence that the loss is concentrated in
unrankable ties.

Despite the lower rank correlation, the rule's selection performance does not degrade.
Across the same eight streams, selection from eight spaces has lower regret than selection
from the original five (\cref{tab:s_selection}). Ties among spaces that contribute nothing
reduce rank correlation without increasing practical regret.

\begin{table}[H]
\centering\footnotesize
\setlength{\tabcolsep}{3.5pt}
\caption{Oracle-envelope cache gains $G^\star$ for the three added spaces on the six cross-domain
streams at $M{=}4$, with $\rho$ recomputed over all eight. To keep the eight comparable, every
entry here and every $\rho$ in this table is a single seed-$0$ run;
Tab.~\ref{tab:s_scope} averages two seeds for DINOv2-B, which is why its $\rho$ column differs
by up to $0.03$. \colorbox{bestgreen}{Green} marks the best space over all eight.}
\label{tab:s_scope8}
\begin{tabular}{lrrrc}
\toprule
Stream & Sup.\ ViT & DINO v1 & BEiT & $\rho$ over $8$\,$\uparrow$ \\
\midrule
\blk{5}{Purity is discriminative} \\
Pets       & $+2.59$ & $+1.94$ & $+0.00$ & $0.84$ \\
Caltech101 & $+0.00$ & $+0.00$ & $+0.00$ & $0.79$ \\
DTD        & $+0.06$ & $+0.00$ & $+0.00$ & $0.91$ \\
UCF101     & $+2.83$ & $+2.67$ & $+0.00$ & $0.79$ \\
Flowers102 & $+5.64$ & $+2.15$ & $+0.00$ & $0.78$ \\
\emph{Mean} & & & & \gain{$0.82$} \\
\midrule
\blk{5}{Precondition violated} \\
EuroSAT    & $+5.31$ & \best{$+9.72$} & $+6.35$ & \loss{$0.26$} \\
\bottomrule
\end{tabular}
\end{table}

EuroSAT is again the exception, and the added spaces sharpen rather than soften it.
DINO v1 gives $+9.72$ there, the largest gain achieved by any space on that stream and more
than it achieves elsewhere; on the other five streams, its gain ranges from $+0.00$ to
$+2.67$. BEiT shows an even sharper contrast: $+6.35$ on EuroSAT and $+0.00$ elsewhere. A
stream on which otherwise weak spaces become competitive may reward structure other than
class-level neighborhood purity, consistent with EuroSAT's
purity range of $14.4$ points and with its being the one stream where our rule does not
apply.

\section{Qualitative Examples}
\label{sec:s_figures2}

Four ImageNet-A queries that the base predictor mislabels and the fused rule
repairs, chosen by the data rather than by hand: base prediction wrong, cache margin
negative in CLIP space, positive in DINOv2-L space, fused prediction correct. Each row of
\cref{fig:s_qualitative} shows the query and its ten nearest retained keys in the two
spaces; the memory is the same $K{=}8$ population in both
(\cref{prop:space-intervention}), and neighbours are ranked by cosine in the full
embedding.

\begin{figure}[H]
\centering
\includegraphics[width=\linewidth]{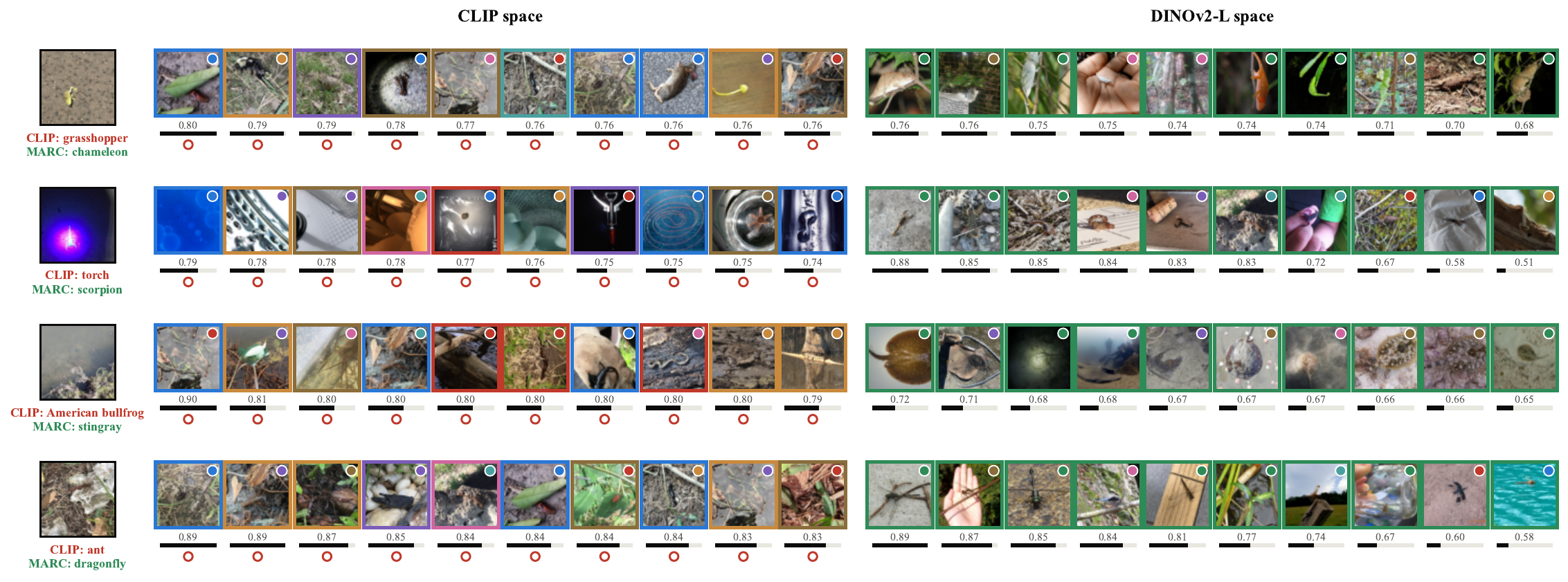}
\caption{Queries repaired after the retrieval-encoder swap. Left block: the ten nearest
retained keys in CLIP space; right block: the same memory re-embedded by DINOv2-L. Border
colour is the neighbour's true label (green matches the query), the corner dot its stored
pseudo-label, the number the cosine similarity, the bar the kernel weight on a scale
shared across both spaces, and a red ring marks a neighbour of a different true class. In
CLIP space the nearest keys mix classes at nearly uniform similarity, so the vote goes to
a confuser; re-embedded by DINOv2-L, the same memory returns neighbours that are almost
all of the query's class.}
\label{fig:s_qualitative}
\end{figure}

\section{Cache Capacity at a Large View Budget}
\label{sec:s_cd_capacity}

The capacity results in Sec.~\ref{sec:capacity} are measured at $M{=}1$, leaving
open whether the same optimum holds once many views are averaged. We therefore repeat
the six cross-domain streams at $M{=}64$ under both $K{=}8$ and $K{=}16$, with three seeds in
each arm, so that capacity is the only factor that varies.

\begin{table}[H]
\centering\footnotesize
\setlength{\tabcolsep}{4pt}
\caption{Cross-domain accuracy at $M{=}64$ under two cache capacities, three seeds per
arm. \emph{Ratio} is the change divided by the pooled dispersion of the two arms. Means and
deltas are computed from unrounded values.}
\label{tab:s_cd_capacity}
\begin{tabular}{lccrr}
\toprule
Stream & $K{=}8$ & $K{=}16$ & $\Delta$ & ratio \\
\midrule
Caltech101 & 94.69\ci{0.08} & 93.51\ci{0.14} & \loss{$-1.18$} & $10.3\times$ \\
DTD        & 54.37\ci{0.87} & 50.18\ci{1.24} & \loss{$-4.20$} & $3.9\times$ \\
Flowers102 & 80.67\ci{0.33} & 79.71\ci{0.39} & \loss{$-0.96$} & $2.6\times$ \\
Pets       & 92.39\ci{0.19} & 91.66\ci{0.05} & \loss{$-0.73$} & $5.2\times$ \\
UCF101     & 71.94\ci{0.33} & 70.65\ci{0.52} & \loss{$-1.29$} & $3.0\times$ \\
\midrule
EuroSAT    & 58.67\ci{1.52} & 58.67\ci{0.42} & $+0.00$ & $0.0\times$ \\
\midrule
\emph{Mean} & 75.45 & 74.06 & \loss{$-1.39$} & \\
\bottomrule
\end{tabular}
\end{table}

Increasing capacity from $K{=}8$ to $K{=}16$ remains harmful at a large view budget,
costing $1.39$
points on average. Five streams lose accuracy and every loss exceeds twice the pooled
dispersion of its two arms, reaching ten times on Caltech101. EuroSAT is the exception for the
third time: the two capacities agree to two decimals, and its $K{=}8$ dispersion of $1.52$
points is the widest in the table, so the stream that defeats the geometry ranking is also the
one whose capacity cannot be resolved. On the other five streams, the $M{=}1$ preference for
$K{=}8$ over $K{=}16$ transfers to $M{=}64$. This was not obvious in advance: averaging many
views could have compensated for a diluted cache, but it does not.

\section{Additional ViT-B/16 Pretraining Regimes}
\label{sec:s_newspaces}

The five spaces compared in Sec.~\ref{sec:capacity} jointly vary the pretraining
objective, data source, dataset scale, and encoder scale. DINOv2 is self-supervised on curated
LVD-142M, CLIP is language-supervised on WIT-400M, and MAE uses masked image modeling on
ImageNet-1k~\citep{deng2009imagenet}; three of the five are scale variants of DINOv2. The
observed association between
encoder-induced retrieval similarities and gain may therefore reflect the pretraining data and recipe as well as
the objective category.

We add three retrieval spaces that hold the architecture fixed at ViT-B/16 with
$768$-dimensional features while broadening the pretraining regimes represented. They are
measured under the same protocol as the other spaces: frozen encoder, CLS token,
$\ell_2$-normalized keys, and the
same cache and fusion configuration. Each model uses its own input normalization rather than
a shared one.

\begin{table}[H]
\centering\scriptsize
\setlength{\tabcolsep}{3pt}
\renewcommand{\arraystretch}{1.15}
\caption{Additional ViT-B/16 retrieval spaces spanning different pretraining regimes.}
\label{tab:s_newspaces}
\begin{tabular}{@{}>{\raggedright\arraybackslash}p{1.85cm}
                  >{\raggedright\arraybackslash}p{2.35cm}
                  >{\raggedright\arraybackslash}p{3.35cm}@{}}
\toprule
Space & Pretraining & Contrast provided \\
\midrule
Supervised ViT-B/16 & labels, IN-21k\,$\rightarrow$\,IN-1k & supervised objective with fixed architecture \\
DINO v1 ViT-B/16    & self-distillation, IN-1k             & related objective family, ${\sim}100\times$ less data than DINOv2 \\
BEiT-B/16           & masked image modeling, IN-22k        & alternative masked-image pretraining recipe \\
\bottomrule
\end{tabular}
\end{table}

DINO v1 provides the closest comparison to DINOv2 among the added spaces. It uses a
related self-distillation approach but is trained on ImageNet-1k rather than LVD-142M. A
result near DINOv2-B would support an explanation based on the broad objective family,
whereas a result near MAE would point to differences in data scale or training recipe. This
comparison is informative but cannot isolate those factors completely.

\begin{table}[H]
\centering\footnotesize
\setlength{\tabcolsep}{3.5pt}
\caption{Comparison of the three added spaces with the five spaces already reported, on ImageNet-A at
$M{=}1$. Purity and anti-hubness are measured at $K{=}10$ before adaptation, and the ceiling is the
oracle envelope $G^\star$ at each space's own $\Kstar$. Rows are ordered by ceiling.
$\Kstar$ and the ceiling are computed from seed $0$; the three-seed figures in the text
use the common $K{=}8$. Because those figures average different stream orders, DINO v1's
three-seed mean at $K{=}8$ can exceed its seed-$0$ ceiling.}
\label{tab:s_newspaces_results}
\begin{tabular}{lccrr}
\toprule
Cache space & $\purity@10$\,$\uparrow$ & $\antihub$\,$\downarrow$ & $\Kstar$ & $G^\star$ ceiling\,$\uparrow$ \\
\midrule
BEiT-B/16                      & 19.1 & 11.1 & 1 & $+0.01$ \\
MAE ViT-B/16                  &  3.2 & 3.4 &  1 & $+0.09$ \\
DINO v1                        & 16.8 & 1.9 & 1 & $+0.17$ \\
CLIP ViT-B/16                 & 29.7 & 6.3 &  2 & $+0.62$ \\
DINOv2-S                      & 26.1 & \best{0.7} & 16 & $+1.33$ \\
Supervised ViT                 & 35.0 & 1.4 & 8 & $+4.75$ \\
DINOv2-B                      & 48.9 & 1.2 &  8 & $+9.88$ \\
\ours DINOv2-L                & \best{67.3} & 1.7 & 8 & \best{$+20.06$} \\
\bottomrule
\end{tabular}
\end{table}

DINO v1 reaches $+0.17$ points at its best seed-$0$ capacity and $+0.25 \pm 0.20$
over three seeds at the matched $K{=}8$, placing it with MAE and BEiT rather than with
DINOv2. Thus, membership in the same broad self-distillation family is not sufficient to
produce a useful retrieval space. The supervised ViT reaches $+4.75$ at its seed-$0$
optimum and $+4.82 \pm 0.20$ over three seeds at $K{=}8$, roughly three times the DINOv2-S
gain despite using no self-supervision. The same ordering appears on ImageNet-V2, where the
supervised ViT reaches $+4.76$ and DINO v1 reaches $+0.51$. Together, these results implicate
pretraining data scale and recipe, rather than the objective category alone, but do not
separate their individual effects.

The geometry rule remains effective after the three structurally different spaces
are added.
Across eight spaces, purity ranks the gain at $\rho{=}0.93$ on ImageNet-A and $\rho{=}0.98$ on
ImageNet-V2, against $0.90$ and $0.98$ over the original five. Selection regret falls from
$0.98$ to $0.35$ points, and its worst case from $5.02$ to $1.65$
(\cref{tab:s_selection}). Rather than degrading as the candidate set broadens, the
association strengthens on ImageNet-A, supporting the interpretation that purity tracks
retrieval quality rather than merely DINOv2 model identity.
With sixteen spaces, the correlation remains $\rho{=}0.959$, and purity selects the best
space on both ImageNet streams (Sec.~\ref{sec:s_spaces16}).

\begin{table}[H]
\centering\footnotesize
\setlength{\tabcolsep}{4pt}
\caption{Selecting a retrieval space by a single measurement, over eight streams and
eight candidate spaces. \emph{Hits} counts streams where the rule picks the best space;
\emph{regret} is $G^\star$ lost relative to oracle selection. Among the tabulated rules, only
purity requires labels.}
\label{tab:s_selection}
\begin{tabular}{llccc}
\toprule
Rule & Labels & hits & mean regret\,$\downarrow$ & worst\,$\downarrow$ \\
\midrule
$\purity@5$        & yes & \best{$7/8$} & \best{$0.14$} & \best{$1.14$} \\
$\purity@10$       & yes & $6/8$ & $0.35$ & $1.65$ \\
\midrule
lowest skewness    & no  & $4/8$ & $4.14$ & $19.89$ \\
lowest $\antihub$  & no  & $0/8$ & $6.45$ & $18.71$ \\
\bottomrule
\end{tabular}
\end{table}

Among the original label-free geometry summaries, none selects reliably. Anti-hubness
prefers MAE on almost every stream because a space whose neighborhoods carry no class
information can still leave few unretrieved points. Purity at $K{=}5$ is slightly better than
at $K{=}10$, although the paper reports $K{=}10$ consistently.
For a measurement $\phi$, let $\sigma_\phi(\mathcal D)$ be its selected space. Selection
regret is
\begin{equation}
\operatorname{Reg}_\phi(\mathcal D)=
\max_s G^\star(s,\mathcal D)-G^\star(\sigma_\phi(\mathcal D),\mathcal D),
\label{eq:selection-regret}
\end{equation}
and a hit has zero regret. Replacing ground-truth agreement with agreement against CLIP
predictions gives label-free \emph{pseudo-purity}. It reproduces the labeled ordering closely
($\rho{=}0.98$ on ImageNet-A, $1.00$ on ImageNet-V2, and $0.93$ averaged over eight streams),
selects the best space on $5/8$ streams, and incurs $1.48$ points of mean regret. This is worse
than labeled $\purity@10$ ($0.35$) but substantially better than anti-hubness ($6.45$); its
worst case is EuroSAT, where even the labeled rule is unresolved.

\section{The Full Sixteen-Space Candidate Set}
\label{sec:s_spaces16}

\Cref{tab:spaces16} compares sixteen retrieval spaces. Eight were introduced earlier: CLIP
ViT-B/16, MAE, three DINOv2 variants, a supervised ViT, DINO v1, and BEiT. The remaining eight
test whether the result depends on a particular training objective, architecture, or scale.

OpenCLIP~\citep{cherti2023openclip}, SigLIP~\citep{zhai2023siglip}, and EVA-02
CLIP~\citep{sun2023evaclip} test three additional forms of cross-modal training, while
CLIP ViT-L/14 tests the effect of scale. ConvNeXt-B~\citep{liu2022convnext} and
ResNet-50~\citep{he2016resnet} test convolutional
architectures. DINOv3-B~\citep{simeoni2025dinov3} adds a newer self-supervised model, and
DeiT-III~\citep{touvron2022deit3} adds a second label-supervised transformer.

Every space is used under the protocol of \cref{sec:protocol}: frozen encoder, a single
pooled or CLS embedding, $\ell_2$-normalised keys, each model's own input normalisation, and
the same cache, kernel and fusion configuration. Only the retrieval encoder changes, so the
controlled estimand is the resulting replacement of the query--key similarity matrix.

\begin{table}[H]
\centering\scriptsize
\setlength{\tabcolsep}{4pt}
\renewcommand{\arraystretch}{1.12}
\caption{The eight retrieval spaces added to reach the sixteen of \cref{tab:spaces16}, with
the purpose of each addition in the controlled comparison.}
\label{tab:s_spaces16_models}
\adjustbox{max width=\linewidth}{%
\begin{tabular}{@{}>{\raggedright\arraybackslash}p{3.6cm}
                  >{\raggedright\arraybackslash}p{9.4cm}@{}}
\toprule
Retrieval space & Purpose in the comparison \\
\midrule
OpenCLIP ViT-B/16 & same objective and architecture as CLIP, trained on LAION-2B~\citep{schuhmann2022laion5b} rather than WIT-400M \\
SigLIP ViT-B/16 & tests a cross-modal model trained with a different loss \\
EVA-02 CLIP & adds an independent cross-modal training recipe \\
CLIP ViT-L/14 & tests whether greater model scale improves cross-modal retrieval \\
DINOv3-B & adds a later generation of self-supervised training \\
DeiT-III ViT-B/16 & adds a second label-supervised vision transformer \\
ConvNeXt-B & tests whether the result extends beyond transformer architectures \\
ResNet-50 (sup.) & repeats that test with a smaller convolutional model \\
\bottomrule
\end{tabular}%
}
\end{table}

CLIP ViT-L/14 is the only cross-modal space with a substantial improvement, reaching $+5.52$
on ImageNet-A; scale therefore explains part of the gap. The supervised ResNet-50 also changes
rank across streams, with gains of $+0.04$ on ImageNet-A and $+1.54$ on ImageNet-V2. This is
another example of the stream dependence studied in \cref{tab:cd16}.

\section{Hyperparameter Sensitivity}
\label{sec:s_sensitivity}

The following one-at-a-time sweeps vary $(K,w,\beta,\alpha)$ around the deployed setting
$(8,10,5,2)$. They are placed after the broader space comparisons because they test local
robustness of the fixed scoring rule rather than the paper's central retrieval-space claim.

\begin{figure}[H]
\centering
\includegraphics[width=\linewidth]{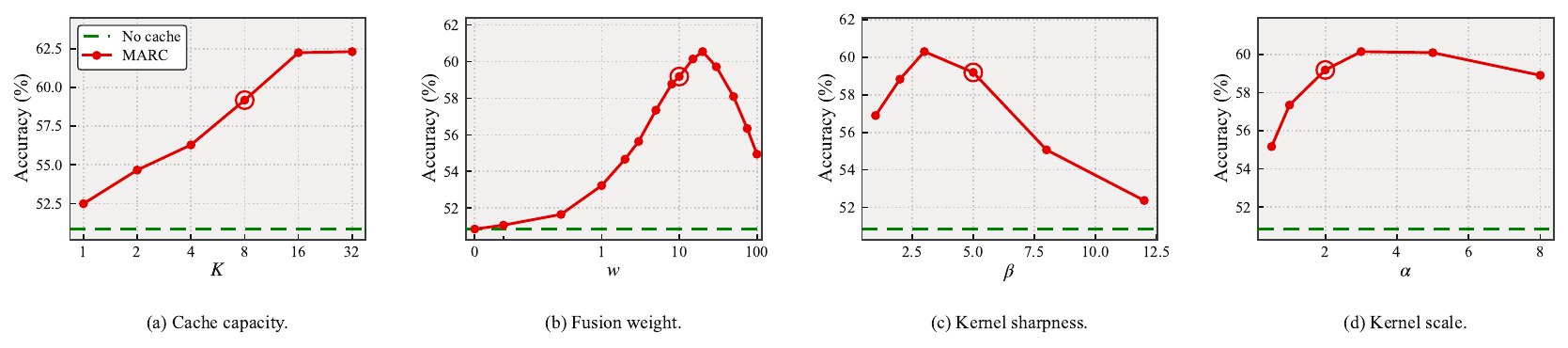}
\caption{ImageNet-A sensitivity with DINOv2-B retrieval. The deployed setting lies in the
broad high-accuracy region; capacity and fusion weight dominate the response.}
\label{fig:s_sens_ab}
\end{figure}

\begin{figure}[H]
\centering
\includegraphics[width=\linewidth]{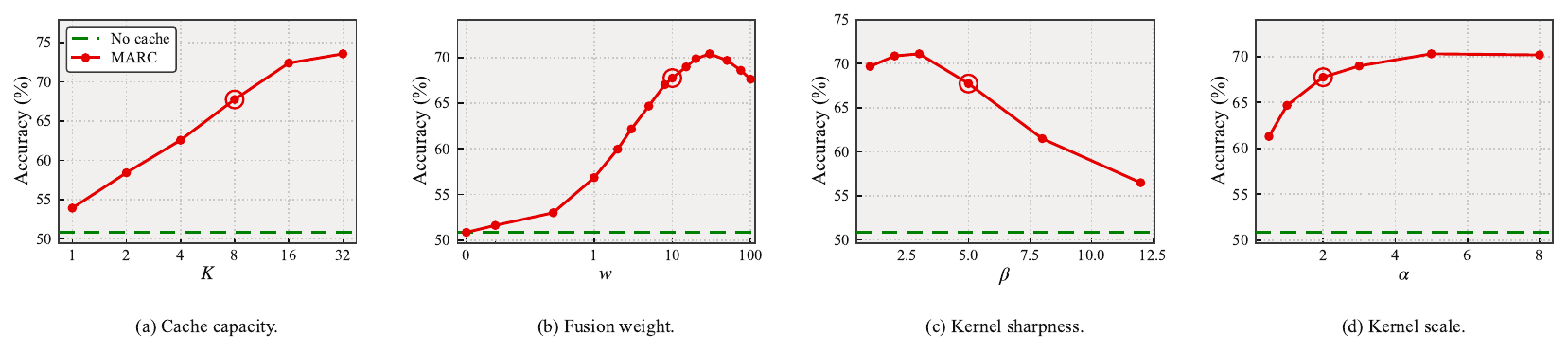}
\caption{ImageNet-A sensitivity with DINOv2-L retrieval. Its similarities support a
wider range of cache weights while retaining a large gain.}
\label{fig:s_sens_al}
\end{figure}

\begin{figure}[H]
\centering
\includegraphics[width=\linewidth]{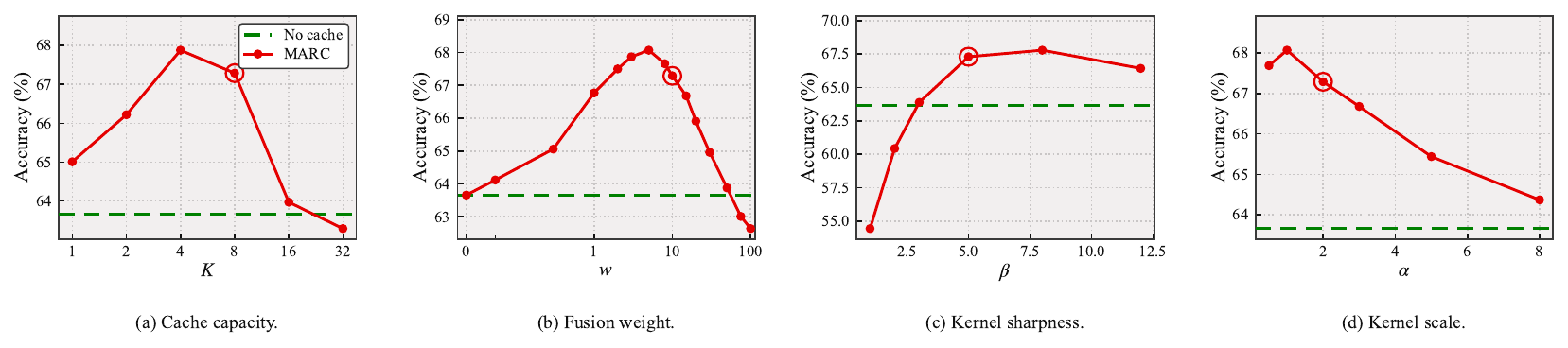}
\caption{ImageNet-V2 sensitivity with DINOv2-B retrieval. Smaller predictor headroom yields a
flatter and lower-gain response than ImageNet-A.}
\label{fig:s_sens_vb}
\end{figure}

\begin{figure}[H]
\centering
\includegraphics[width=\linewidth]{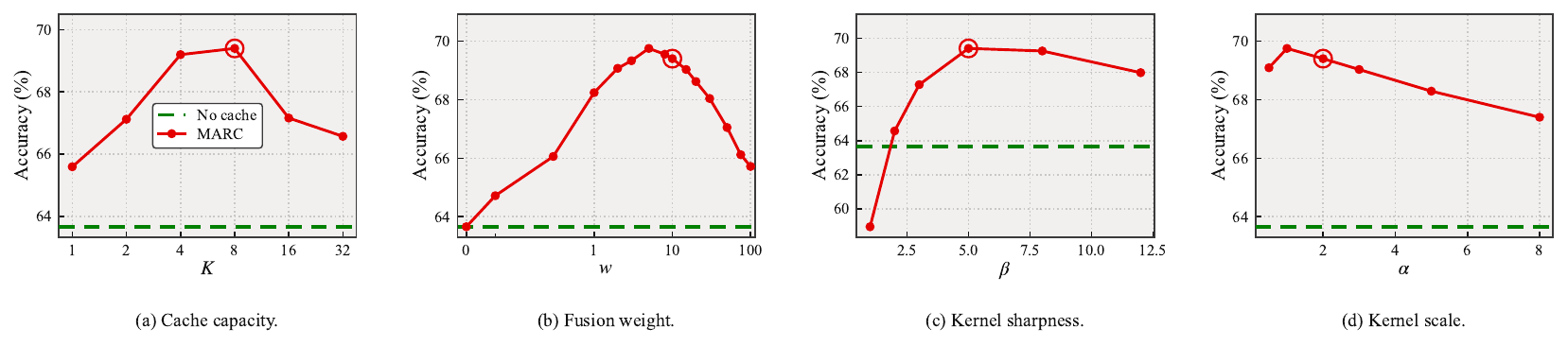}
\caption{ImageNet-V2 sensitivity with DINOv2-L retrieval. The useful region remains broad,
although the attainable gain is smaller than on ImageNet-A.}
\label{fig:s_sens_vl}
\end{figure}

\begin{figure}[H]
\centering
\includegraphics[width=\linewidth]{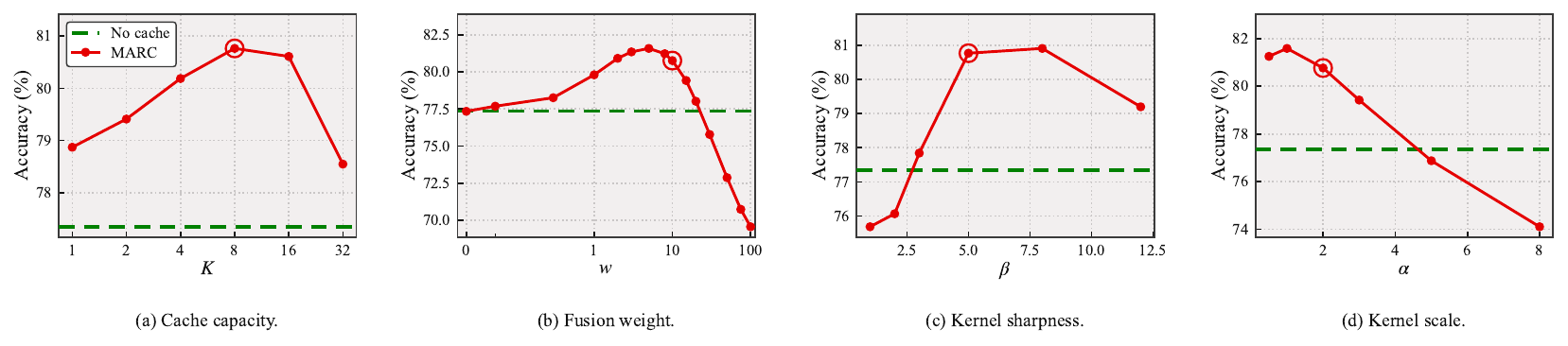}
\caption{ImageNet-R sensitivity with DINOv2-B retrieval under the same one-factor sweeps.}
\label{fig:s_sens_rb}
\end{figure}

\begin{figure}[H]
\centering
\includegraphics[width=\linewidth]{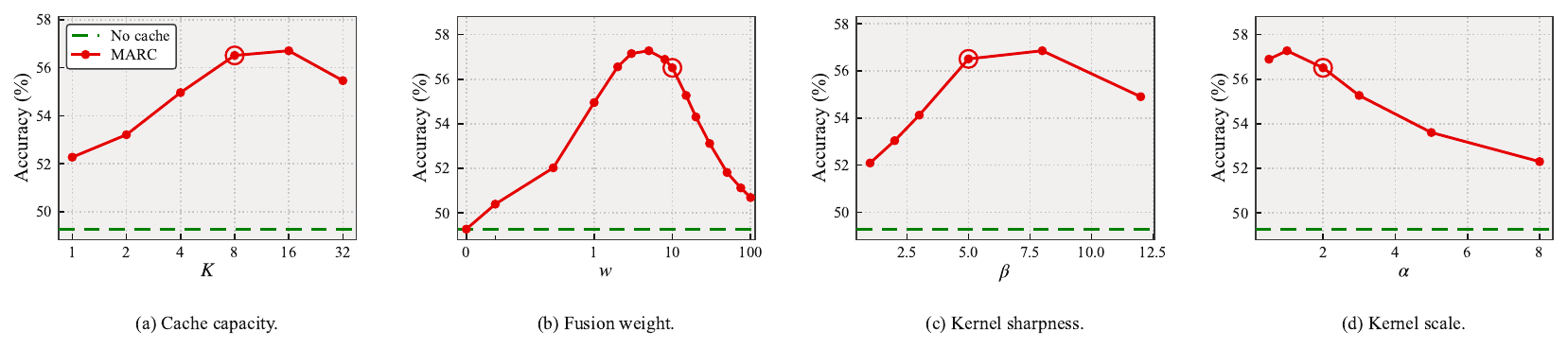}
\caption{ImageNet-Sketch sensitivity with DINOv2-B retrieval under the same one-factor
sweeps.}
\label{fig:s_sens_kb}
\end{figure}

\section{Full Experimental Protocol and Isolation Proof}
\label{sec:s_protocol}

\paragraph{Data and models.}
We evaluate ImageNet-A ($7{,}500$ images, $200$ classes), ImageNet-V2 ($10{,}000$)  ~\citep{recht2019imagenetv2},
ImageNet-R ($30{,}000$, $200$ classes), ImageNet-Sketch ($50{,}889$), ImageNet-val
($50{,}000$), and the cross-domain suite of~\citet{karmanov2024tda}. The predictor is frozen
CLIP ViT-B/16 or ResNet-50 with fixed prompts. Retrieval encoders are frozen, use one pooled
or CLS embedding with native normalization, and are listed in
\cref{sec:s_newspaces,sec:s_spaces16}.

\paragraph{Geometry calibration.}
The reported purity, pseudo-purity, and hubness measurements use the complete target test split
as the calibration batch: $7{,}500$ images for ImageNet-A, $10{,}000$ for ImageNet-V2,
$30{,}000$ for ImageNet-R, and $50{,}889$ for ImageNet-Sketch, with the analogous complete
split for each cross-domain dataset. These batches are the evaluation splits rather than held-out
sets. Ground-truth labels enter only the offline purity analysis; pseudo-purity uses frozen CLIP
top-1 predictions on the same unlabeled images.

\paragraph{Views and online update.}
From $M$ views we retain $r_M=\max(1,\lfloor0.1M\rfloor)$ lowest-entropy views and average
their raw CLIP logits. The unaugmented view supplies each query and candidate key. Both
memories are queried before insertion and store the ensemble top-1 pseudo-label. The CLIP
memory has per-class capacity $K_C{=}16$, zero output weight, and supplies retrieval-admission
priority. Within each pseudo-class, the retrieval cache retains the $K$ highest-priority
arrivals; a new item replaces the current minimum only when its priority is larger. All keys
are $\ell_2$-normalized. Setting $K_C{=}0$ gives the strict one-memory control.

Query-before-insert prevents a current-sample self-vote. Strict priority replacement makes
each pseudo-class contain the top-$K$ arrivals in every stream prefix; under distinct
priorities, capacity-$K$ entries are nested in those at $K{+}1$. Because stored pseudo-labels
and priorities do not use the fusion weight, the same trajectory supports the complete
offline sweep of Proposition~\ref{prop:traj}.

\begin{proof}[Proof of \cref{prop:space-intervention}]
Couple two retrieval-space runs with the same stream and augmentation draws. Their CLIP
quantities $(b_t,\hat y_t,e_t,k_t^C)$ coincide. By induction, their CLIP-memory states,
pre-update evidence, and retrieval priorities also coincide. Since admission compares only
pseudo-class and priority, both runs retain the same stream indices at every time; only the
vectors attached to those indices differ.

For a retained index $i$, write $u_{ti}^s=g_s(x_t)^\top g_s(x_i)$. Substitution into
\cref{eq:kernel} gives the logit identity in \cref{eq:space-perturbation}. Moreover,
$\phi_\beta'(u)=\beta\phi_\beta(u)\le\beta$ on $[-1,1]$, so the mean-value theorem gives
\[
|L_{t,c}^{s}-L_{t,c}^{s'}|
\le\alpha\beta\sum_{i\in I_{t-1}^K(c)}|u_{ti}^s-u_{ti}^{s'}|
=D_{t,c}^{s,s'}.
\]
Because $b_t$ is shared, subtracting the largest possible perturbations from the winning-class
gap yields the sufficient stability condition in \cref{eq:space-stability}.
\end{proof}

\paragraph{Search and reporting.}
Capacity uses $K\in\{1,2,3,4,8,16,32,64\}$ and fusion uses
$\mathcal W=\{0,0.1,0.3,1,2,3,5,8,10,15,20,30,50,75,100\}$. Including $w{=}0$ makes the
oracle envelope a non-negative upper bound. Retrieval-space comparisons, benchmark rows,
component ablations, and capacity sweeps use three seeds unless a caption states otherwise.
Seeds change stream order and, for $M>1$, augmentation draws; matched configurations share
seeds. Subscripts report sample standard deviations with the $n-1$ divisor, and averages use
unrounded values. A standard deviation is not a confidence interval: at three seeds the
$95\%$ $t$-interval is $2.484\times$ the printed subscript. Interval estimates, and the exact
tests behind every comparative claim, are reported separately in App.~\ref{sec:s_ci}.

\paragraph{Timing.}
We time a $3{,}000$-image prefix on one thermally settled RTX 3060 after $100$ warm-up
samples. TDA, COSMIC, and \method{} are measured in the same session, repeated three times in
\cref{sec:s_timing}. We interpret within-session ratios rather than absolute consumer-GPU
latency.

\section{Confidence Intervals for the Reported Gains}
\label{sec:s_ci}

We quantify uncertainty using four complementary procedures, each associated with a distinct
source of variation. Absolute accuracies are accompanied by \emph{sampling intervals},
computed as Wilson intervals over the evaluation split. Differences between methods are
summarized by \emph{paired intervals} over the same evaluation images, together with exact
two-sided McNemar tests on the discordant predictions; these intervals support the comparative
claims in this section. For the oracle gain $G^\star$, we report an \emph{envelope interval}
obtained by bootstrapping evaluation examples and re-optimizing over $\mathcal W$ within each
replicate, thereby incorporating uncertainty due to weight selection. Finally, variation
across stream orders is summarized by a \emph{seed interval}: a $t$-interval with half-width
$t_{.975,n-1}s/\sqrt{n}$, where $s$ is the sample standard deviation across orders. For three
orders, this half-width is $2.484s$. Sampling and seed intervals address different estimands. The former characterizes uncertainty
over evaluation examples, whereas the latter characterizes sensitivity to stream order.
Similarly, overlap between two marginal accuracy intervals does not determine whether their
difference is statistically distinguishable: because the methods are evaluated on identical
images, inference about their difference should use the paired results. Calibration-subsample
intervals for Spearman's $\rho$ and selection regret are reported separately in
\cref{tab:s_bootstrap_selection}.

\paragraph{Sixteen-space intervention.}
\Cref{tab:s_ci_spaces16} reports paired intervals and exact tests for all retrieval spaces in
\cref{tab:spaces16} on ImageNet-A. The intervention holds both the stream and retained sample
population fixed (\cref{prop:space-intervention}); consequently, every row is compared with
the same no-cache predictor on the same $7{,}500$ images.

\begin{table}[H]
\centering\footnotesize
\setlength{\tabcolsep}{4pt}
\caption{Paired uncertainty estimates for the sixteen-space intervention on ImageNet-A at
$K{=}8$ and $M{=}1$, relative to the $50.84\%$ no-cache baseline. \emph{Gain at $w{=}10$}
denotes the paired accuracy difference, reported with its $95\%$ confidence interval and the
exact two-sided McNemar $p$-value over discordant predictions. $G^\star$ denotes the oracle
envelope; its bootstrap interval re-optimizes the maximum over $\mathcal W$ within each
replicate. These columns estimate different quantities and therefore need not coincide.}
\label{tab:s_ci_spaces16}
\adjustbox{max width=\linewidth}{%
\begin{tabular}{lrrrr}
\toprule
Cache space & gain at $w{=}10$ & disc. & McNemar $p$ & $G^\star$ \\
\midrule
DINOv2-L & \gain{$+18.19$} $[+17.21,+19.17]$ & 1658 & $2\mathrm{e}{-}285$ & $+19.36$ $[+18.36,+20.51]$ \\
DINOv2-B & \gain{$+9.31$} $[+8.43,+10.18]$ & 1186 & $4\mathrm{e}{-}97$ & $+10.27$ $[+9.32,+11.31]$ \\
DeiT-III & \gain{$+7.73$} $[+6.80,+8.67]$ & 1322 & $5\mathrm{e}{-}59$ & $+8.01$ $[+7.00,+9.11]$ \\
DINOv3-B & \gain{$+5.72$} $[+4.96,+6.48]$ & 879 & $3\mathrm{e}{-}49$ & $+6.49$ $[+5.67,+7.49]$ \\
Supervised ViT & \gain{$+4.13$} $[+3.38,+4.89]$ & 844 & $5\mathrm{e}{-}27$ & $+4.43$ $[+3.61,+5.24]$ \\
CLIP ViT-L/14 & \gain{$+3.69$} $[+2.63,+4.75]$ & 1657 & $1\mathrm{e}{-}11$ & $+5.52$ $[+4.65,+6.39]$ \\
ConvNeXt-B & \gain{$+2.88$} $[+2.20,+3.56]$ & 684 & $1\mathrm{e}{-}16$ & $+2.85$ $[+2.27,+3.72]$ \\
EVA-02 CLIP & \gain{$+1.88$} $[+0.84,+2.92]$ & 1577 & $4\mathrm{e}{-}04$ & $+4.15$ $[+3.44,+4.93]$ \\
DINOv2-S & $+0.53$ $[-0.24,+1.31]$ & 882 & $0.189$ & $+1.24$ $[+0.85,+1.75]$ \\
DINO v1 & \loss{$-2.27$} $[-3.01,-1.52]$ & 810 & $3\mathrm{e}{-}09$ & $+0.05$ $[+0.00,+0.48]$ \\
OpenCLIP & \loss{$-6.16$} $[-7.17,-5.15]$ & 1536 & $2\mathrm{e}{-}32$ & $+0.76$ $[+0.39,+1.21]$ \\
ResNet-50 (sup) & \loss{$-7.05$} $[-7.98,-6.13]$ & 1289 & $2\mathrm{e}{-}50$ & $+0.04$ $[+0.00,+0.32]$ \\
BEiT & \loss{$-8.60$} $[-9.37,-7.83]$ & 933 & $2\mathrm{e}{-}108$ & $+0.00$ $[+0.00,+0.03]$ \\
SigLIP & \loss{$-9.40$} $[-10.55,-8.25]$ & 2003 & $6\mathrm{e}{-}57$ & $+0.99$ $[+0.47,+1.56]$ \\
CLIP ViT-B/16 & \loss{$-10.28$} $[-11.38,-9.18]$ & 1861 & $2\mathrm{e}{-}73$ & $+0.44$ $[+0.12,+0.87]$ \\
MAE & \loss{$-43.68$} $[-44.88,-42.48]$ & 3556 & $<10^{-300}$ & $+0.03$ $[+0.00,+0.33]$ \\
\bottomrule
\end{tabular}%
}
\end{table}

The fixed-weight and oracle-envelope results expose different aspects of performance. At
$w{=}10$, low-purity retrieval spaces can be strongly detrimental: CLIP ViT-B/16 decreases
accuracy by $10.28$ points, while MAE decreases it by $43.68$ points. By construction, the
oracle envelope cannot be negative because $w{=}0$ is included in $\mathcal W$. Thus, an
envelope gain of $+0.00$ indicates that none of the evaluated weights improves over the
no-cache predictor; it does not imply that every nonzero weight is safe. Reporting both
quantities is therefore necessary to characterize the deployment risk illustrated in
\cref{fig:hazard}.

The ceiling is an attainable upper envelope, not the fixed deployment outcome. At the reported fixed w = 10, seven of sixteen ImageNet A retrieval spaces reduce accuracy, including CLIP and MAE, as shown in Table \cref{fig:hazard}.

\paragraph{Results with no detectable effect.}
Across the $378$ tabulated configurations, $50$ paired $95\%$ confidence intervals include
zero. We consistently interpret these cases as showing no detectable effect, rather than as
evidence of a gain or loss. They occur primarily for weaker retrieval spaces and smaller
cross-domain evaluation sets. Within the sixteen-space ImageNet-A comparison, DINOv2-S is the
only such case: $+0.53$ points $[-0.24,+1.31]$, with $p{=}0.19$, although its oracle envelope
is positive at $+1.24$ $[+0.85,+1.75]$. In \cref{tab:crossdomain}, the effects for Caltech101
($-0.04$ $[-0.95,+0.87]$, $p{=}1.00$) and DTD ($+0.83$ $[-1.18,+2.83]$, $p{=}0.45$) are not
distinguishable from zero at either $M{=}4$ or $M{=}64$, whereas those for EuroSAT,
Flowers102, Pets, and UCF101 are. The additional streams in \cref{tab:cd_extra} instead show
statistically resolved losses: $-7.92$ points $[-9.18,-6.66]$ on Aircraft and $-4.74$ points
$[-5.08,-4.40]$ on Food101, paired per image and averaged over three seeds, consistent with
the transfer risk identified in that analysis.

\paragraph{Multiplicity adjustment.}
We assess multiplicity across the $378$ reported tests. Benjamini--Hochberg control at
$q{=}0.05$ rejects $328$ null hypotheses the same set identified at the uncorrected
$p<0.05$ threshold corresponding to a nominal false-discovery allowance of $16.4$ among
these rejections. At $q{=}0.01$, $314$ hypotheses remain rejected. The more conservative
Bonferroni correction at family-wise level $0.05$, corresponding to a per-test threshold of
$1.32\times10^{-4}$, retains $296$ rejections. The primary comparisons are significant under
each correction, often by several orders of magnitude. Moreover, these comparisons are
specified by the intervention design rather than selected post hoc from the full set of
configurations.

\paragraph{Matched-order comparison with COSMIC.}
The original $M{=}8$ COSMIC evaluation did not retain the per-example predictions required for
paired inference. We therefore evaluated both systems on three recorded stream orders. Within
each order, the methods receive the same permutation and label sequence over all $7{,}500$
images. The resulting matched comparison is reported in \cref{tab:s_ci_cosmic}.

\begin{table}[H]
\centering\footnotesize
\setlength{\tabcolsep}{5pt}
\caption{Matched comparison of \method{} and COSMIC on ImageNet-A at $M{=}8$ and DINOv2-B
scale. Each row uses an identical recorded stream order for both methods. For the pooled
estimate, we first average each image's paired difference across the three orders and then
construct the interval, so that each image contributes once. Treating all $22{,}500$
predictions as independent would underestimate uncertainty because the same $7{,}500$ images
appear in every order.}
\label{tab:s_ci_cosmic}
\begin{tabular}{lrrrrr}
\toprule
Order & \method{} & COSMIC & paired advantage & disc. & McNemar $p$ \\
\midrule
0 & 63.56 & 61.80 & \gain{$+1.76$} $[+0.99,+2.53]$ & 878 & $9.5\mathrm{e}{-}06$ \\
1 & 63.89 & 63.52 & $+0.37$ $[-0.41,+1.16]$ & 906 & $0.37$ \\
2 & 64.04 & 62.81 & \gain{$+1.23$} $[+0.43,+2.03]$ & 942 & $3.0\mathrm{e}{-}03$ \\
\midrule
\emph{Pooled, clustered by image} & 63.83 & 62.71 & \gain{$\mathbf{+1.12}$} $[+0.66,+1.58]$ & & \\
\bottomrule
\end{tabular}
\end{table}

Three aspects of this comparison merit clarification. First, the recorded orders used here
differ from those in \cref{tab:ood,tab:pareto}. Because test-time adaptation is order
dependent, the corresponding accuracies differ slightly from the values
$63.95/64.53/64.03$ and $62.89/62.60/62.76$ reported in those tables. Second, the seed interval
for the mean difference across these three matched orders is $+1.12$ $[-0.62,+2.86]$ and
includes zero. With $n{=}3$, the critical value is $t_{.975,2}=4.303$, so uncertainty over
stream order remains substantial even though the image-paired interval excludes zero. These
intervals address distinct sources of variation and should be interpreted jointly. Third,
COSMIC's reported accuracy is the maximum over a $14\times14$ grid of fusion weights selected
on the evaluation stream, whereas \method{} uses the prespecified value $w{=}10$. The
comparison therefore favors COSMIC with respect to hyperparameter selection and is not fully
matched in that dimension.

Across these matched orders, the standard deviation is $0.86$ percentage points for COSMIC
and $0.25$ for \method{}. The smaller $\pm0.15$ dispersion reported for COSMIC in
\cref{tab:ood} characterizes its three evaluation orders in that experiment; it should not be
interpreted as a general estimate of sensitivity to stream order.

\paragraph{Rank correlations.}
We quantify uncertainty in the rank correlations from \cref{tab:spaces16} using the bootstrap
procedure of \cref{tab:s_bootstrap_selection}. Each of $B{=}300$ replicates samples calibration
images with replacement and recomputes the geometry statistic, while holding the encoder set
and measured gains fixed. \Cref{tab:s_rho_ci} reports the results at the largest calibration
size considered. On both streams, the intervals for purity and anti-hubness are well separated,
supporting the conclusion that purity more reliably ranks attainable gain.

\begin{table}[H]
\centering\footnotesize
\setlength{\tabcolsep}{5pt}
\caption{Calibration-subsample intervals for the rank correlations in \cref{tab:spaces16},
computed over sixteen retrieval spaces using $B{=}300$ bootstrap replicates at $n{=}2000$.
Full-split point estimates use midranks throughout.}
\label{tab:s_rho_ci}
\begin{tabular}{llcc}
\toprule
Stream & Measurement & full split & $n{=}2000$ [95\% interval] \\
\midrule
ImageNet-A & $\rho(\purity@10,G^\star)$ & $+0.959$ & \gain{$+0.957$ $[+0.956,+0.962]$} \\
           & $\rho(-\antihub,G^\star)$ & $+0.541$ & $+0.575$ $[+0.472,+0.682]$ \\
\midrule
ImageNet-V2 & $\rho(\purity@10,G^\star)$ & $+0.943$ & \gain{$+0.937$ $[+0.906,+0.958]$} \\
            & $\rho(-\antihub,G^\star)$ & $+0.756$ & $+0.809$ $[+0.743,+0.883]$ \\
\bottomrule
\end{tabular}
\end{table}

Two considerations govern the interpretation of these intervals. First, anti-hubness measures
the fraction of zero-indegree points in an $n$-point probe graph and is therefore intrinsically
sensitive to calibration size. Its subsample interval characterizes the $n{=}2000$ regime; it
is not intended to contain the full-split estimate. Second, the bootstrap resamples images,
not encoders. The sixteen retrieval spaces were deliberately selected to span distinct
training regimes and do not constitute a random sample from a defined encoder population.
Accordingly, the correlations are descriptive summaries of the observed ranking. The more
directly decision-relevant quantity is the selection regret in
\cref{tab:s_bootstrap_selection}, which is zero on both streams at every calibration size
considered.

\paragraph{Latency.}
Timing ratios are consistent across the three sessions in \cref{tab:s_timing}. Computing each
ratio within a session and then constructing an interval across sessions gives a latency ratio
of $2.58\times$ $[2.52,2.64]$ for COSMIC relative to \method{} at $M{=}8$, and
$5.06\times$ $[4.95,5.18]$ for the $64$-view ensemble relative to single-view \method{}-L.

\paragraph{Results without uncertainty estimates.}
We omit confidence intervals when the available evidence does not support a meaningful
estimate. This applies to three groups of results. First, the $216$ literature values in
\cref{tab:ood,tab:crossdomain,tab:s_extra} are reproduced from prior work under the respective
evaluation protocols; without access to the underlying per-example predictions, we treat them
as benchmark context rather than subjecting them to retrospective uncertainty analysis.
Second, \cref{tab:s_newspaces,tab:s_spaces16_models} report descriptive attributes rather than
performance estimates. Third, several configurations are represented by a single run or two
stream orders, including the ImageNet-V2 column of \cref{tab:spaces16} and the transfer column
of \cref{tab:swap}. For these configurations, we report paired sampling intervals over
evaluation images when paired predictions are available, but do not report seed intervals.
At $n{=}2$, the critical value $t_{.975,1}=12.706$ yields a half-width of approximately nine
sample standard deviations, providing little information about variability across stream
orders. We therefore do not use two-order dispersion as inferential evidence.

\section{Timing Reproducibility}
\label{sec:s_timing}

\begin{table}[H]
\centering\footnotesize
\setlength{\tabcolsep}{4pt}
\caption{Tier~(ii), same-budget system comparison: ImageNet-A accuracy versus cost on one RTX
3060, with all rows timed in one session. \colorbox{bestgreen}{Green} marks the
compute--accuracy Pareto frontier. The CLIP reference uses standard prompts; controlled gains
use the $50.84\%$ prompt-ensemble baseline. The \method{} and COSMIC $M{=}8$ accuracies each
average three seeds; the latency ratio is repeated across sessions in \cref{tab:s_timing}.
The remaining \method{} rows are the seed-$0$ runs of the timing session and therefore differ
slightly from the three-seed means quoted in \cref{tab:ood} and in the text.}
\label{tab:pareto}
\begin{tabular}{lrrr}
\toprule
Configuration & views & s/img\,$\downarrow$ & Top-1\,$\uparrow$ \\
\midrule
\blk{4}{Baselines, same GPU and same session} \\
CLIP zero-shot (standard prompts) &  1 & \best{0.021} & \best{47.87} \\
Ensemble only, no cache           & 64 & 0.310 & 58.77 \\
TDA~\citep{karmanov2024tda}       &  1 & 0.190 & 60.11 \\
COSMIC~\citep{huang2025cosmic} (DINOv2-B) & 8 & 0.181 & 62.75 \\
\midrule
\blk{4}{\method{}: view budget swept at fixed retrieval space} \\
\ours \textbf{\method{}} (DINOv2-B) &  1 & \best{0.045} & \best{59.73} \\
\ours \textbf{\method{}} (DINOv2-B) &  4 & \best{0.053} & \best{62.93} \\
\method{} (DINOv2-B)              &  8 & 0.070 & 64.17\ci{0.31} \\
\method{} (DINOv2-B)              & 64 & 0.338 & 65.23 \\
\midrule
\blk{4}{\method{}: retrieval space swept at fixed view budget} \\
\ours \textbf{\method{}} (DINOv2-L) &  1 & \best{0.061} & \best{68.99} \\
\ours \textbf{\method{}} (DINOv2-L) &  4 & \best{0.069} & \best{70.69} \\
\midrule
\emph{\method{}-L $M{=}1$ vs.\ $64$-view ensemble} & & \gain{$5.1\times$} & \gain{$+10.22$} \\
\bottomrule
\end{tabular}
\end{table}

\Cref{tab:pareto} measures all rows within one session for direct comparison. Two further
sessions on an otherwise idle GPU replicate the timing ratios.

\begin{table}[H]
\centering\footnotesize
\setlength{\tabcolsep}{4pt}
\caption{Steady-state s/img across three independent timing sessions. Session~1
is the one reported in Tab.~\ref{tab:pareto}.}
\label{tab:s_timing}
\begin{tabular}{lrrrr}
\toprule
Configuration & S1 & S2 & S3 & spread \\
\midrule
CLIP zero-shot          & 0.0207 & 0.0279 & 0.0207 & 0.0072 \\
Ensemble only, $M{=}64$ & 0.3101 & 0.3272 & 0.3003 & 0.0269 \\
TDA~\citep{karmanov2024tda} & 0.1903 & 0.1918 & 0.1893 & 0.0025 \\
COSMIC~\citep{huang2025cosmic}, $M{=}8$ & 0.1812 & 0.1733 & 0.1787 & 0.0079 \\
\midrule
\ours \method{} (DINOv2-B), $M{=}1$  & 0.0450 & 0.0447 & 0.0445 & \best{0.0005} \\
\method{} (DINOv2-B), $M{=}8$        & 0.0701 & 0.0679 & 0.0688 & 0.0022 \\
\method{} (DINOv2-B), $M{=}64$       & 0.3378 & 0.3323 & 0.3355 & 0.0055 \\
\method{} (DINOv2-L), $M{=}1$        & 0.0614 & 0.0640 & 0.0598 & 0.0042 \\
\bottomrule
\end{tabular}
\end{table}

Every row varies by at most $0.008$\,s/img across sessions except the ensemble-only
row, whose spread is $0.027$\,s/img. The CLIP reference meets the absolute bound but
has the largest relative spread, at $31\%$ of its mean. At roughly $21$\,ms per image, it is
particularly sensitive to fixed per-batch and session-level overhead; no claim depends on
this row. The two ratios used in the paper are stable: the matched comparison against COSMIC
is $2.6\times$, $2.6\times$, and $2.6\times$ across the three sessions, and single-view
\method{}-L against the $64$-view ensemble is $5.1\times$, $5.1\times$, and $5.0\times$.
Constructing intervals from the session-level ratios gives $2.58\times$ $[2.52,2.64]$ and
$5.06\times$ $[4.95,5.18]$, respectively.

\paragraph{Matched-budget accuracy over seeds.} For the comparison in
Sec.~\ref{sec:capacity}, COSMIC reaches $62.89$, $62.60$, and $62.76$ on ImageNet-A at
$M{=}8$ over three seeds, yielding a mean of $62.75 \pm 0.15$. \method{} reaches $63.95$,
$64.53$, and $64.03$, yielding $64.17\pm0.31$. The paired system advantage is
$1.42\pm0.45$ points at the same view budget and auxiliary scale. These runs stored no
per-sample predictions for COSMIC, so the advantage they support is a seed-level one;
App.~\ref{sec:s_ci} reports a matched re-run of both systems over recorded stream orders,
which supplies the per-image paired interval and the exact test.

\section{Open Questions and Future Research Directions}
\label{sec:future}

Our results show retrieval space as an independent design axis, but also raise broader questions about how memory-based adaptation should interact with representation geometry, target shift, and inference compute. Here, we discuss several open questions arising from our findings that may guide future research in this area:

\noindent\textbf{$\checkmark$ Can retrieval quality be predicted without labels?}
Purity strongly tracks attainable cache gain when candidate spaces are sufficiently separated, while pseudo-purity shows that much of this signal can be recovered without ground-truth labels. The remaining challenge is to identify a fully online criterion that predicts not only neighborhood consistency but whether retrieved evidence will \emph{correct} rather than reinforce the frozen predictor. Such a criterion could combine neighborhood agreement, similarity margins, hubness, prediction uncertainty, and temporal stability. A particularly interesting direction is to estimate the expected utility of a retrieval space directly from the unlabeled stream, allowing systems to decide when a cache is likely to help before committing memory or compute.

\noindent\textbf{$\checkmark$ Should retrieval space itself adapt over time?}
Our experiments select one frozen retrieval encoder for an entire stream, yet the best space varies across datasets and shifts. Real deployments may be more heterogeneous: a stream can move between environments, acquisition conditions, or semantic regimes over time. This motivates dynamic retrieval systems that select among multiple frozen representations, combine complementary spaces, or locally reshape similarities as the stream evolves. Rather than adapting the predictor, future TTA systems could therefore adapt the \emph{geometry of memory access}, potentially preserving model stability while remaining responsive to non-stationary environments.

\noindent\textbf{$\checkmark$ What makes a representation a good cache space?}
DINOv2 performs strongly in our experiments, but encoder family or scale alone does not fully explain cache effectiveness. The contrast among CLIP, DINO, DINOv2, MAE, and other representations suggests that useful cache spaces must preserve class-relevant local neighborhoods while avoiding similarity structures that concentrate retrieval on a small set of samples. Purity and anti-hubness capture parts of this behavior, but neither fully characterizes it. Understanding which representation-learning objectives produce favorable memory geometry could connect cache-based adaptation to a broader theory of representation quality beyond linear probing and standard $k$-NN evaluation.

\noindent\textbf{$\checkmark$ How should memory capacity depend on geometry?}
Our results suggest that useful cache capacity is representation-dependent rather than a purely system-level hyperparameter. A space with clean, distributed neighborhoods may benefit from retaining more examples, whereas a hub-dominated or contaminated space may saturate earlier. This raises the possibility of geometry-aware memory management in which capacity, eviction, and admission are chosen from observed neighborhood statistics instead of fixed globally. Such mechanisms could allocate memory selectively across classes, regions of feature space, or phases of a changing stream.

\noindent\textbf{$\checkmark$ When should retrieval replace additional inference compute?}
The diminishing cache gain with increasing view count indicates that retrieval and augmentation provide partly overlapping evidence. This suggests a broader resource-allocation problem: for each input, should additional compute be spent on more views, deeper retrieval, a larger memory, another retrieval encoder, or no adaptation at all? An adaptive controller could allocate inference compute based on prediction uncertainty and retrieval confidence, using cheap single-view retrieval for easy corrections while reserving expensive augmentation for ambiguous cases. Such conditional computation may be especially useful in latency-, energy-, or throughput-constrained deployments.

\noindent\textbf{$\checkmark$ When should the system refuse to use its memory?}
Cross-domain results show that cache utility varies substantially across streams, including settings where its contribution is negligible. This is not only a limitation but an opportunity to make cache-based adaptation selective. Future systems could estimate whether the current neighborhood contains sufficiently reliable evidence and suppress cache fusion when it does not. Moving from `always retrieve'' to `retrieve when useful'' could improve robustness under weak neighborhood structure, abrupt shifts, or early stages of a stream when memory is sparse.

\noindent\textbf{$\checkmark$ Can multiple retrieval geometries be complementary?}
Our controlled study intentionally varies one retrieval encoder at a time to isolate its effect, but different representations may encode complementary notions of similarity. A semantic representation may retrieve class-consistent examples, while another space may better preserve texture, shape, acquisition conditions, or domain-specific structure. Future work could study mixtures of retrieval spaces without collapsing them into a single global similarity, including query-dependent routing or agreement-aware fusion. The central challenge is to obtain complementary evidence without reintroducing the complexity that our reduced design shows is often unnecessary.

\noindent\textbf{$\checkmark$ How does retrieval geometry interact with real non-stationarity?}
Our online protocol provides a controlled setting for studying memory trajectories, but real streams may contain recurring domains, gradual drift, abrupt transitions, class imbalance, or newly appearing categories. In such settings, the appropriate retrieval geometry and useful memory horizon may change together. Studying this interaction could lead to memories that detect distribution transitions, retain reusable evidence across recurring environments, and forget neighborhoods that no longer reflect the current stream.

\noindent\textbf{$\checkmark$ Can the principle extend beyond natural-image classification?}
The separation between prediction space and retrieval space is not inherently specific to CLIP or ImageNet. Remote sensing, scientific and medical imaging, embodied perception, and other specialized visual domains often combine pretrained models with streams whose local structure differs from the original training distribution. These settings provide natural tests of whether domain-specific or self-supervised retrieval representations can complement general-purpose predictors. More broadly, the same question arises wherever a frozen model consults an external memory, including multimodal and retrieval-augmented systems: the representation that is best for producing an output need not be the representation that is best for finding the evidence used to produce it.

Taken together, these directions suggest a shift from viewing the cache as a fixed auxiliary component toward treating memory access as an adaptive inference problem. The next generation of cache-based TTA may therefore ask not only \emph{what should be stored?}, but also \emph{in which geometry should it be retrieved, when should that geometry change, and when should retrieved evidence be trusted?}


\end{document}